\documentclass{article}

\usepackage{microtype}
\usepackage{graphicx}
\usepackage{subfigure}
\usepackage{booktabs} %

\usepackage{hyperref}

\usepackage[accepted]{icml2024}

\usepackage{amsmath}
\usepackage{amssymb}
\usepackage{mathtools}
\usepackage{amsthm}
\usepackage{dsfont}

\usepackage[capitalize]{cleveref}

\usepackage{enumitem}

\usepackage{xspace}

\newcommand{\eg}{e.g.\xspace}

\newcommand{\wrt}{{\it w.r.t.}\xspace}
\newcommand{\iid}{\xspace}
\newcommand{\Zcal}{\mathcal{Z}}
\newcommand{\ie}{i.e.\xspace}

\newcommand{\R}{\mathds{R}}
\newcommand{\el}{\widehat{L}_S}
\newcommand{\ef}{\widehat{F}_S}

\newcommand{\zcal}{\mathcal{Z}}

\newcommand{\Rd}{{\R^d}}
\newcommand{\levy}{L_t^\alpha}

\newcommand{\E}{\mathds{E}}
\newcommand{\Eof}[2][]{\mathds{E}_{#1} \left[ #2 \right]}

\newcommand{\Pof}[2][]{\mathds{P}_{#1} \left( #2 \right)}

\newcommand{\normof}[1]{\left\Vert #1 \right\Vert}
\newcommand{\klb}[2]{\text{\normalfont{{KL}}}\left(#1 || #2 \right)}
\newcommand{\renyi}[3][\alpha]{\text{\normalfont{{D}}}_{#1}\left(#2 || #3 \right)}

\newcommand{\op}[1]{I\left[ #1 \right] }

\newcommand{\bregman}[2]{D_\Phi(#1, #2)}
\newcommand{\uinftybar}{\bar{u}_\infty}
\newcommand{\entphi}[2][\uinftybar]{\text{\normalfont Ent}_{#1}^\Phi \left(#2\right)}
\newcommand{\set}[1]{\left\{ #1 \right\} }
\newcommand{\datadist}{\mu_z^{\otimes n}}
\newcommand{\Sd}{{\mathds{S}^{d-1}}}
\newcommand{\chisq}[2]{\chi^2\left(#1 \right|\left| #2 \right)}
\newcommand{\fraclap}{ \left( -\Delta \right)^{\frac{\alpha}{2}}}

\newcommand{\equald}{\overset{\mathbf{d}}{=}}
\newcommand{\landau}[1]{\mathcal{O} \left( #1 \right)}

\usepackage[textsize=tiny]{todonotes}

\crefname{ass}{Asmp.}{Asmps.}
\Crefname{ass}{Assumption}{Assumptions}

\usepackage{thm-restate}
\theoremstyle{plain}
\newtheorem{theorem}{Theorem}[section]

\newtheorem{lemma}[theorem]{Lemma}
\newtheorem{corollary}[theorem]{Corollary}
\theoremstyle{definition}
\newtheorem{definition}[theorem]{Definition}
\newtheorem{assumption}[theorem]{Assumption}
\theoremstyle{remark}
\newtheorem{remark}[theorem]{Remark}
\newtheorem{example}[theorem]{Example}

\icmltitlerunning{Generalization Bounds for Heavy-Tailed SDEs}

\begin{document}

\twocolumn[
\icmltitle{Generalization Bounds for Heavy-Tailed SDEs through the \\ Fractional Fokker-Planck Equation}

\icmlsetsymbol{equal}{*}

\begin{icmlauthorlist}
\icmlauthor{Benjamin Dupuis}{inria,ens,psl}
\icmlauthor{Umut \c{S}im\c{s}ekli}{inria,ens,psl,cnrs}
\end{icmlauthorlist}

\icmlaffiliation{inria}{Inria}
\icmlaffiliation{ens}{Ecole Normale Supérieure, Paris, France}
\icmlaffiliation{cnrs}{CNRS}
\icmlaffiliation{psl}{PSL Research University, Paris, France}

\icmlcorrespondingauthor{Benjamin Dupuis}{benjamin.dupuis@inria.fr}
\icmlcorrespondingauthor{Umut \c{S}im\c{s}ekli}{umut.simsekli@inria.fr}

\icmlkeywords{Generalisation bounds, heavy-tails, SDE}

\vskip 0.3in
]

\printAffiliationsAndNotice{}  %

\begin{abstract}
Understanding the generalization properties of heavy-tailed stochastic optimization algorithms has attracted increasing attention over the past years. 
While illuminating interesting aspects of stochastic optimizers by using heavy-tailed stochastic differential equations as proxies, prior works either provided \emph{expected} generalization bounds, or introduced non-computable information theoretic terms.
Addressing these drawbacks, in this work, we prove \emph{high-probability} generalization bounds for heavy-tailed SDEs which do not contain any nontrivial information theoretic terms. To achieve this goal, we develop new proof techniques based on estimating the entropy flows associated with the so-called \emph{fractional} Fokker-Planck equation (a partial differential equation that governs the evolution of the distribution of the corresponding heavy-tailed SDE). In addition to obtaining high-probability bounds, we show that our bounds have a better dependence on the dimension of parameters as compared to prior art.
Our results further identify a phase transition phenomenon, which suggests that heavy tails can be either beneficial or harmful depending on the problem structure. 
We support our theory with experiments conducted in a variety of settings.
\end{abstract}

\section{Introduction}
\label{sec:intro}

A supervised machine learning setup consists of a data space $\zcal$, a data distribution $\mu_z$, and a parameter space, which will be $\Rd$ in our study. Given a loss function $\ell : \Rd \times \zcal \longrightarrow \R_+$, the goal is to minimize the following population risk:
\begin{align}
    \label{eq:prm}
    \text{min}_{w \in \Rd} L(w), \quad L(w) := \Eof[z \sim \mu_z]{\ell(w,z)}.
\end{align}
As $\mu_z$ is typically unknown in practice, the population risk $L$ is replaced by the empirical risk, defined as follows:
\begin{align}
    \label{eq:er}
    \el(w) := \frac{1}{n} \sum_{i=1}^n \ell(w,z_i),
\end{align}
where $S = (z_1,\dots,z_n) \sim \datadist$ is a dataset and each $z_i$ is sampled independent and identically (i.i.d.) from $\mu_z$. Even though $\el$ can be computed in practice as opposed to $L$, 
in several practical scenarios, $\ell$ is further replaced with a `surrogate loss' function $f : \Rd \times \zcal \longrightarrow \R$. For instance, in a binary classification setting, $\ell$ is typically chosen as the non-differentiable $0$-$1$ loss, whereas $f$ can be chosen as a differentiable surrogate, such as the cross-entropy loss, which would be amenable to gradient-based optimization. 
We accordingly define the \emph{surrogate empirical risk}:\footnote{In our theoretical setup we introduce the surrogate loss to be able to cover more general settings. However, this is not a requirement, we can set $f = \ell$. }
\begin{align*}
    \ef(w) := \frac{1}{n} \sum_{i=1}^n f(w,z_i).
\end{align*} 
Given a dataset $S$ and a surrogate loss $f$, a stochastic optimization algorithm $\mathcal{A}$ aims at minimizing $\ef$ and can be seen as a function such that $\mathcal{A}(S,U) =: w_{S,U}$, where $U$ is a random variable encompassing all the randomness in the algorithm. One of the major challenges of statistical learning theory is then to upper-bound the so-called generalization error, i.e., $L(w_{S,U}) - \el(w_{S,U})$. Once such a bound can be obtained, it immediately provides an upper-bound on the true risk $L$, as $\el$ can be computed numerically. 

In our study, we analyze the generalization error induced by a specific class of \emph{heavy-tailed} optimization algorithms, described by the next stochastic differential equation (SDE): 
\begin{align}
    \label{eq:multifractal_dynamics}
    dW_t^S = -\nabla V_S(W_t^S) dt+ \sigma_1 d\levy + \sqrt{2}\sigma_2 dB_t,
\end{align}
where $\levy$ is a stable Lévy process, which will be formally introduced in Section~\ref{sec:background_notations}, $\alpha \in (1,2)$ is the tail-index, controlling the heaviness of the tails\footnote{$\levy$ does not admit a finite variance and as $\alpha$ gets smaller the process becomes heavier-tailed.} (examples of Lévy processes are provided in \cref{fig:levy-processes}), $B_t$ is a Brownian motion in $\Rd$, $\sigma_{1,2} \geq 0$ are fixed constants, and the \emph{potential} $V_S$ is the $\ell_2$-regularized loss that is defined as:
\begin{align}
    \label{eq:regularized_er}
    V_S(w) := \ef(w) + \frac{\eta}{2} \normof{w}^2.
\end{align}
The term $\frac{\eta}{2} \normof{w}^2$ corresponds to a $\ell_2$ weight decay that is commonly used in the theoretical analysis of Langevin dynamics, which corresponds to \eqref{eq:multifractal_dynamics} with $\sigma_1=0$ \citep{mou_generalization_2017,li_generalization_2020,farghly_time-independent_2021}.

\begin{figure}[t]
    \centering
    \includegraphics[width=0.9\columnwidth]{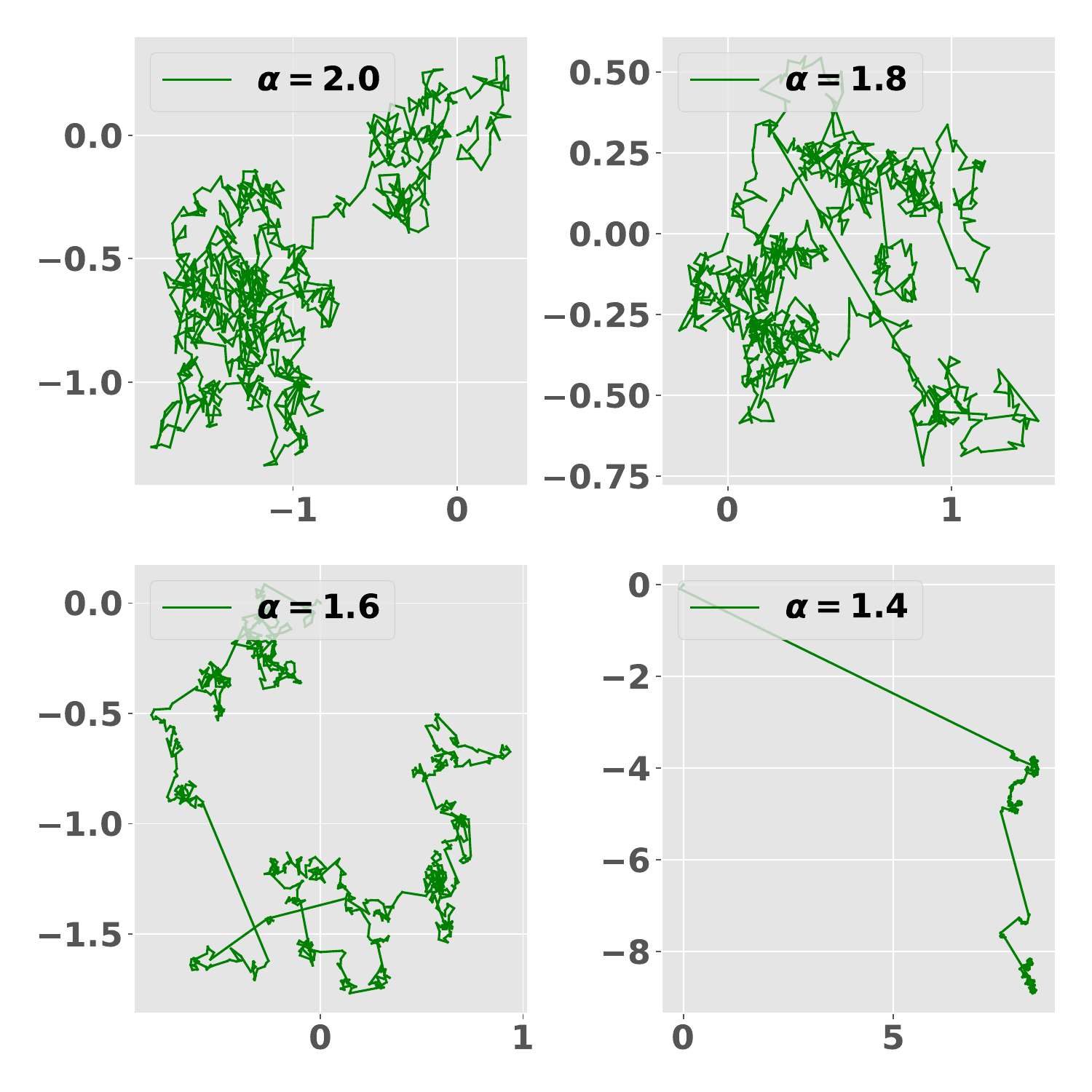}
    \vspace{-0.7cm}
    \caption{Simulation of $\levy$ for different values of $\alpha$.}
    \label{fig:levy-processes}
    \vskip -0.1in
\end{figure}

There has been an increasing interest in understanding the theoretical properties of heavy-tailed SDEs, such as \eqref{eq:multifractal_dynamics}, due to two main reasons.
\begin{enumerate}[noitemsep,topsep=0pt,leftmargin=.2in]
    \item Recently, several studies have provided empirical and theoretical evidence that stochastic gradient descent (SGD) can exhibit heavy tails when the step-size is chosen large, or the batch size small \citep{simsekli_tail-index_2019,gurbuzbalaban_heavy-tail_2021,hodgkinson_multiplicative_2020,pavasovic2023approximate}. This has motivated several studies (see e.g.,  \citep{nguyen_first_2019,simsekli_hausdorff_2021,simsekli_tail-index_2019,zhou2020towards}) to model the heavy tails, emerging in the large step-size/small batch-size regime, through heavy-tailed SDEs and analyze the resulting SDE as a proxy for SGD. 
    \item Injecting heavy-tailed noise to SGD in an explicit way has also been considered from several perspectives. It has been shown that heavy-tailed noise can help the algorithm avoid sharp minima \cite{csimcsekli2017fractional,simsekli_tail-index_2019,nguyen2019non,nguyen_first_2019}, attain better generalization properties \cite{lim_chaotic_2022,raj_algorithmic_2023-1} or to obtain sparse parameters in an overparametrized neural network setting \cite{wan2023implicit}. 
\end{enumerate}

Our main goal in this study is to develop \emph{high-probability} generalization bounds for the SDE given in \eqref{eq:multifractal_dynamics}. More precisely, we will choose the learning algorithm $\mathcal{A}$ as the solution to the SDE \eqref{eq:multifractal_dynamics}, i.e., $\mathcal{A}(S,U) = W_T^S$ for some fixed time horizon $T>0$, where in this case $U$ will encapsulate the randomness introduced by $L^\alpha_t$ and $B_t$.
We will then upper-bound the generalization gap $\Eof[U]{G_S(W_T^S)}$ over $S \sim \datadist$ under this specific choice, where
\begin{align}
    \label{eq:generalzation_gap}
    G_S(W_T^S) := L(W_T^S) - \el(W_T^S).
\end{align}

\textbf{Related work. }
In the case where the SDE \eqref{eq:multifractal_dynamics} is only driven by a Brownian motion, \ie $\sigma_1 = 0$, \cref{eq:multifractal_dynamics} reduces to the continuous Langevin dynamics, whose generalization properties have been widely studied \citep{mou_generalization_2017, li_generalization_2020, farghly_time-independent_2021, futami_time-independent_2023}, as well as its discrete-time counterpart \citep{raginsky_non-convex_2017,pensia_generalization_2018,negrea_information-theoretic_2020,haghifam_sharpened_2020,neu_information-theoretic_2021-1,farghly_time-independent_2021}. For instance, \citet{mou_generalization_2017} distinguished two different approaches: the first is based the concept of algorithmic stability \citep{bousquet_stability_2002, bousquet_sharper_2020}, while the second is based on PAC-Bayesian theory \citep{shawetaylor1997pac,mcAllester1998some,catoni_pac-bayesian_2007,germain_pac-bayesian_2009}. In our work, we extend this approach to handle the presence of heavy-tailed noise.

A first step toward generalization bounds for heavy-tailed dynamics was achieved by leveraging the fractal structures generated by such SDEs \citep{simsekli_hausdorff_2021,hodgkinson_generalization_2022,dupuis_mutual_2023}. These studies successfully brought to light new empirical links between the sample paths of these SDEs and topological data analysis \cite{birdal_intrinsic_2021,dupuis_generalization_2023,andreeva_metric_2023}. However, the uniform bounds developed in these studies contain intricate mutual information terms between the set of points of the trajectory and the dataset, which are not amenable for numerical computation to our knowledge \citep{dupuis_generalization_2023}.
Closest to our work are the results recently obtained by \citet{raj_algorithmic_2023,raj_algorithmic_2023-1}, in the case of pure heavy-tailed noise (\ie $\sigma_2 = 0$). \citet{raj_algorithmic_2023} used an algorithmic stability argument to derive expected generalization bounds. While their approach provided more explicit bounds that do not contain mutual information terms, 
it still has certain drawbacks: \textit{(i)} the proof technique cannot be directly used to derive high probability bounds and \textit{(ii)} their bound has a strong dependence in the dimension $d$, rendering it vacuous in overparameterized settings.

\textbf{Contributions. }
In our work, we aim to solve these issues by introducing new tools, taking inspiration from the PAC-Bayesian techniques already used in the case of Langevin dynamics. 
In particular, we will leverage recent results on fractional partial differential equations \citep{gentil_logarithmic_2008, tristani_fractional_2013}, and use them to extend the analysis technique presented in \citet{mou_generalization_2017} to our heavy-tailed setting. While the presence of the heavy tails makes our task significantly more technical, our results unify both light-tailed and heavy-tailed models around one proof technique. Our contributions are as follows:
\begin{itemize}[itemsep=1pt,topsep=0pt,leftmargin=.2in]
    \item We derive high-probability generalization bounds, first when $\sigma_2 > 0$, then in the case $\sigma_2=0$, which turns out to introduce the main technical challenge to our task. Informally, our result takes the following form, with high probability over $S\sim \datadist$,
    \begin{align*}
        \Eof[U]{G_S(W_T^S)} \lesssim \sqrt{\frac{K_{\alpha, d}}{n \sigma_1^\alpha} \int_0^T \E_U\normof{\nabla \ef(W_t)}^2dt}  ,
    \end{align*}
    where $U$ denotes the randomness coming from $\levy$ and $B_t$, and $K_{\alpha, d}$ is a constant depending on $\alpha$ and $d$. We further provide additional results where the resulting bound has a different form and is time-uniform (i.e., does not diverge with $T$) at the expense of introducing terms that cannot be computed in a straightforward way.
    \item By analyzing the constant $K_{\alpha, d}$, we study the impact of the tail-index $\alpha$ on our bounds. Our analysis reveals the existence of a phase transition: we identify two regimes, where in the first case heavy tails are malicious, i.e., the bound increases with the increasing heaviness of the tails, whereas, in the second regime, the heavy tails are beneficial, i.e., increasing the heaviness of the tails results in smaller bounds. Furthermore, we show that our bounds have an improved dependence on the dimension $d$ compared to \cite{raj_algorithmic_2023}.  
\end{itemize}
We support our theory with various experiments conducted on several models. As our experiments require discretizing the dynamics \eqref{eq:multifractal_dynamics}, we analyze the extension of our bounds to a discrete setting, as an additional contribution, see \cref{sec:discrete_case}. All the proofs are presented in the Appendix. The code for our numerical experiments is available at \url{https://github.com/benjiDupuis/heavy_tails_generalization}.

\section{Technical Background}
\label{sec:background_notations}

\subsection{Levy processes and Fokker-Planck equations}
\label{sec:levy_fokker_planck}

 A Lévy process $(L_t)_{t\geq 0}$ is a stochastic process which is stochastically continuous
and has stationary and independent increments, with $L_0 = 0$. We are interested in a specific class of such processes, called symmetric (strictly) $\alpha$-stable processes, which we denote $\levy$. These processes are defined through the characteristic function of their increments, i.e., $\Eof{e^{i \xi \cdot (\levy - L_s^\alpha)}} = e^{- |t - s|^\alpha \normof{\xi}^\alpha}$. 
When the tail-index, $\alpha$, is $2$, then $L_t^2$ corresponds to $\sqrt{2}B_t$, where $B_t$ is a standard Brownian motion in $\Rd$ . For $\alpha < 2$, the processes have heavy-tailed distributions and exhibit jumps (see \cref{fig:levy-processes}). We restrict our study to $\alpha > 1$, since when $\alpha \leq 1$, the expectation of $\levy$ is not defined, which may introduce technical complications and does not have a clear practical interest. We provide further details on Lévy processes in \cref{sec:background_levy_markov}, see also \citep{schilling_introduction_2016-1}.

As mentioned in the introduction, the learning algorithm treated in this study consists in the SDE \eqref{eq:multifractal_dynamics}, defined in the Itô sense \citep[Section $12$]{schilling_introduction_2016-1}, which generalizes both Langevin dynamics \citep{mou_generalization_2017, li_generalization_2020} and purely heavy-tailed dynamics \citep{raj_algorithmic_2023}.

Inspired by \citep{mou_generalization_2017, li_generalization_2020}, our proofs will  not be directly based on this SDE, but on an associated partial differential equation, called the \emph{fractional} Fokker-Planck equation (FPE), or the forward Kolmogorov equation \citep{umarov_beyond_2018}. This equation describes the evolution of the probability density function $u_t^S(w) := u^S(t,w)$ the random variable $W_t^S$, that is the solution of \cref{eq:multifractal_dynamics}. 
Following \citep{duan_introduction_2015, umarov_beyond_2018, schilling_symbol_2010}, \cref{eq:multifractal_dynamics} is associated with the following FPE:
\begin{align}
    \label{eq:empirical_lfp}
    \frac{\partial}{\partial_t} u_t^S = -\sigma_1^\alpha \fraclap u_t^S + \sigma_2^2 \Delta u_t^S  + \text{div} (u_t^S \nabla V_S),
\end{align}
where $\fraclap$ is the (negative) fractional Laplacian operator, which is formally defined in \ref{sec:background_levy_markov}, see also \citep{daoud_fractional_nodate, schertzer_fractional_2001} for introductions.

\subsection{PAC-Bayesian bounds}
\label{sec:pac_bayesian_boudns}

Based on the notations of \cref{eq:multifractal_dynamics}, the learning algorithm studied in this paper is a random map that takes the data $S \in \zcal^n$ as input and generates $W_T^S$ as the output. Due to the randomness introduced by $B_t$ and $\levy$, this procedure defines a \emph{randomized predictor}, \ie, given $S$, the output $W_T^S$ follows a certain probability distribution. 

Generalization properties of similar randomized predictors have been popularly studied through the PAC-Bayesian theory (see \citep{alquier_user-friendly_2021} for a formal introduction). Informally, in PAC-Bayesian analysis, a generalization bound is typically based on some notion of distance between a \emph{posterior} distribution over the predictors, typically denoted by $\rho_S$, a data-dependent probability distribution on $\Rd$, and a data-independent distribution over the predictors, typically denoted by $\pi$, called the \emph{prior}, see e.g., 
\citep{catoni_pac-bayesian_2007,mcallester_pac-bayesian_2003,maurer_note_2004,viallard_general_2021}. 

As an additional theoretical contribution, we begin by proving a generic PAC-Bayesian bound that will be suitable for our setting. This bound has a similar form to that of \citet{germain_pac-bayesian_2009}, but holds for subgaussian losses, and not only bounded losses, see \cref{sec:proofs-subgaussian-pb}.
\begin{theorem}
    \label{thm:kl_pb_bound}
    We assume that $\ell$ is $s^2$-subgaussian, in the sense of \cref{ass:subgaussian}. Then, we have, with probability at least $1 - \zeta$ over $S \sim \datadist$, that
    \begin{align*}
       \Eof[w \sim \rho_S]{G_S(w)} \leq 2s \sqrt{\frac{\klb{\rho_S}{\pi} + \log(3/\zeta)}{n}},
    \end{align*}
    where $\klb{\rho_S}{\pi}$  is the Kullback-Leibler (KL) divergence, whose definition is recalled in \cref{sec:it_terms_and_pb_bounds}.
\end{theorem}

 Our main theoretical contributions will be proving upper-bounds on  $\klb{\rho_S}{\pi}$, when $\rho_S$ is set to the distribution of $W_t^S$ and $\pi$ is chosen appropriately.
 Additionally, in \cref{sec:toward_time_uniform}, we will prove generalization bounds that are based on related but different generic PAC-Bayesian results, for which we provide a short introduction in \cref{sec:it_terms_and_pb_bounds}.

To end this section, we define the notion of $\Phi$-entropy, through which we link PAC-Bayesian bounds and the study of fractional FPEs \citep{gentil_logarithmic_2008, tristani_fractional_2013}.

\begin{restatable}[$\Phi$-entropies]{definition}{defPhiEntropy}
    \label{def:phi_entropies}
    Let $\mu$ be a non-negative measure on $\Rd$ and $\Phi: \mathds{R}_+ \longrightarrow \mathds{R}$ be a convex function. Then, for a $g: \Rd \longrightarrow \mathds{R}_+$, such that $g,\Phi(g) \in L^1(\mu)$, we define:
    \begin{align*}
        \entphi[\mu]{g} := \int \Phi(g) ~d\mu - \Phi \left( \int g ~d\mu \right).
    \end{align*}
\end{restatable}
\vskip -0.1in
Note in particular that, if $\Phi(x) = x\log(x)$ and $g$ is chosen to be $d\rho_S/d\pi$, the Radon-Nykodym derivative of $\rho_S$ with respect to $\pi$, then we have $\entphi[\pi]{g} = \klb{\rho_S}{\pi}$.

\section{Main Assumptions}
\label{sec:notations_assumptions}

As discussed in \cref{sec:levy_fokker_planck}, our analysis is based on \cref{eq:empirical_lfp}. 
 To avoid technical complications, we assume that it has a solution, $u_t^S = u^S(t,w)$, that is continuously differentiable in $t$ and twice continuously differentiable in $w$. We provide a discussion of these properties in \cref{sec:background_levy_markov}. 
 We also denote by $\rho_t^S$ the corresponding probability distribution on $\Rd$, so that $\rho_t^S$ is the law of $W_t^S$.

We first make two classical assumptions. The first is the subgaussian behavior of the objective $\ell$. Besides, we make a smoothness assumption on function $f$ ensuring the existence of strong solutions to \cref{eq:multifractal_dynamics} \citep{schilling_symbol_2010}. Those assumptions are made throughout the paper.

\begin{assumption}
    \label{ass:subgaussian}
    The loss $\ell(w,z)$ is $s^2$-subgaussian, \ie for all $w$ and all $\lambda \in \R$, $\Eof[z]{e^{\lambda (\ell(w,z) - \Eof[z']{\ell(w,z')})}} \leq e^{\frac{\lambda^2 s^2}{2}}$. Moreover, $\ell(w,z)$ is integrable with respect to $\rho_t^S \otimes \datadist$.
\end{assumption}
\begin{assumption}
    \label{ass:smooth}
    $f(w,z)$ is $M$-smooth, which means:
    \begin{align*}
        \normof{\nabla_w f(w,z) - \nabla_w f(w',z)} \leq M \normof{w - w'}.
    \end{align*}
\end{assumption}
\vskip -0.1in
As our proof technique is based on the use of PAC-Bayesian bounds, where we use $\rho_T^S$ as a posterior distribution, 
we are required to find a pertinent choice for the prior distribution $\pi$. We define it by considering the FPE of the Lévy driven Ornstein-Uhlenbeck process associated with the regularization term, $\frac{\eta}{2} \normof{w}^2$. 
More precisely, we consider $\uinftybar$ a solution to the following steady FPE.
\begin{align}
    \label{eq:steady_state}
    0 = -\sigma_1^\alpha \fraclap \uinftybar + \sigma_2^2 \Delta \uinftybar  + \eta \nabla \cdot (\uinftybar w).
\end{align}
It has been shown in \citep{tristani_fractional_2013, gentil_logarithmic_2008} that such a steady state is well-defined and regular enough. We hence denote by $\pi$ the probability distribution, on $\Rd$, with density $\uinftybar$. We characterize further properties of the prior $\pi$ in \cref{lemma:steady_state_regularized}. 

Throughout the paper, we will use the following notation:
\begin{align}
    \label{eq:vts_definition}
    v_t^S (x) := \frac{u_t^S(x)}{\uinftybar(x)} = \frac{d \rho_t^S}{d \pi} (x).
\end{align}
We will often omit the dependency of $v_t^S$ on $t$ and/or $S$, hence denoting $v_t$, or just $v$. 

Our theory will require a technical regularity condition on $v_t^S$, which we will now formalize. To achieve this goal, let $\Phi: \R_+ \to \mathds{R}$ be a twice differentiable convex function. Specific choices for $\Phi$ will be made in \cref{sec:main_results}.

\begin{assumption}
    \label{ass:phi_regularity}
    For all $t>0$ and $S$, the functions $v_t^S$ are positive, continuously differentiable and $\entphi{v_t^S}<\infty$. Moreover, we define:
    $$a(\theta,y;s,u) := \theta \cdot \nabla v(y + s\theta) \Phi''(v(y)) \theta \cdot \nabla v (y) \bar{u}_\infty (y + u\theta),$$
    and we assume:
    \begin{enumerate}[noitemsep,topsep=0pt,leftmargin=.2in]
        \item For each bounded interval $I$, there exists a non-negative integrable function $\chi_I$ s.t. $\forall t\in I,~|\partial_t \Phi(v_t) \uinftybar| \leq \chi_I$.
        \item Let us fixed $t$ and denote $v = v_t$. For any bounded open set $V$ of $\R_+^2$, the functions, defined for $(s,u) \in V$, 
     \begin{align*}
         (y,\theta) \mapsto a(\theta,y;s,u),~ (y,\theta) \mapsto \frac{\partial a}{\partial s}, ~ (y,\theta) \mapsto \frac{\partial a}{\partial u}, 
     \end{align*}
     with $(\theta, y) \in \Sd \times \Rd$, are uniformly dominated by an integrable function on $\Sd \times \Rd$.
     \item Finally, the function $\Rd \longrightarrow \R$ given by:
     \begin{align}
     \label{eq:ipp-condition}
     \left( v |\Phi' \circ v| + |\Phi \circ v|\right) \bar{u}_{\infty} \left( \Vert \nabla V \Vert + \Vert \nabla V_S \Vert \right),
     \end{align}
     vanishes at infinity (along each coordinate of $x \in \Rd$).
    \end{enumerate}
\end{assumption}

We will say that the functions $v_t^S$ are \emph{$\Phi$-regular}.
The first condition is essentially allowing to properly differentiate $\entphi{v_t^S}$, which, following \citet{gentil_logarithmic_2008}, is key to the proposed methods. 
The second requires local integrability of functionals naturally appearing in the computation of $\frac{d}{dt} \entphi{v_t^S}$.
The third condition makes valid the integration by parts performed in the proof of \cref{lemma:big_decomposition}, presented in \cref{sec:proofs-main_decomposition}. Note that we do not require any uniformity in $S$ in \cref{ass:phi_regularity}. 

Let us informally justify the third condition when $\Phi(x)$ is either $x\log(x)$ or $x^2$ (which is our case). It is known that the tail behavior of $\bar{u}_\infty$ is in $1/\Vert x\Vert^{d + \alpha}$ \citep{tristani_fractional_2013}. Moreover, $\Vert \nabla V(x) \Vert$ and $\Vert \nabla V_S(x) \Vert $ are of order at most $1 + \Vert x \Vert$ based on \cref{ass:smooth}. The condition boils down to $\Phi(v_t^S(x)) + v_t^S(x) \Phi'(v_t^S(x)) = \landau{\normof{x}^{d+\alpha -1}}$, which is reasonable given the definitions of $u_t^S$ and $v_t^S$.

Finally, we assume that an integral appearing repeatedly in our statements and proofs is finite.

\begin{assumption}
    \label{ass:phi_risk_integrability}
    For almost all $S$ and all $t\geq 0$, we have:
    \begin{align*}
        \int \uinftybar \Phi''(v) v^2 \normof{\nabla \ef}^2 dx  < +\infty.
    \end{align*}
\end{assumption}
Let us consider $\Phi(x) = x\log(x)$. In this case,  Assumption~\ref{ass:phi_risk_integrability} directly holds when the surrogate $f$ is Lipschitz in $w$. 
More generally, when $t \to \infty$, following arguments in \citep{tristani_fractional_2013}, it is reasonable to consider that the behavior of $u_t^S$ near $x \to \infty$ is $\mathcal{O} \left(\normof{x}^{-d-\alpha}\right)$. Therefore, the previous assumption can be informally thought as $\Vert\nabla \ef(x)\Vert \lesssim \normof{x}^{a/2}$ with $a<\alpha$ (e.g., if $\nabla f$ is H\"{o}lder continuous), weaker conditions may also be acceptable.

\section{Generalization Bounds via Multifractal Fokker-Planck Equations}
\label{sec:main_results}

In this section, we present our main theoretical contributions. The main tool is \cref{lemma:big_decomposition}, which offers a decomposition of $\frac{d}{dt} \entphi{v_t^S},$ that will be used throughout the proofs. 

After presenting \cref{lemma:big_decomposition}, we will start by dealing with the easier case, which is when $\sigma_2 > 0$ in \eqref{eq:multifractal_dynamics}. Then, in \cref{sec:pure_levy_case}, we show how an additional assumption can be leveraged to handle the case where $\sigma_2 = 0$, which presents the most interest for us. Finally, we will extend our analysis to obtain time-uniform bounds, in \cref{sec:toward_time_uniform}. 
For notational purposes, we define, with a slight abuse of notation:
\begin{align}
    G_S(T) := \Eof[U]{G_S(W_T^S)} = \Eof[w \sim \rho_T^S]{G_S(w)}.
\end{align}

\subsection{Warm-up: Noise with non-trivial Brownian part}
\label{sec:brownian case}

Thanks to our PAC-Bayesian approach, our task boils down to bounding the KL divergence between the posterior $\rho_T^S$ and the prior $\pi$. It is given by $\klb{\rho_T^S}{\pi} = \entphi{v^S_T}$,
where, in Sections \ref{sec:brownian case} and \ref{sec:pure_levy_case}, we fix the convex function $\Phi$ to be $\Phi(x) = \Phi_{\log}(x) := x \log(x)$, Assumptions~\ref{ass:phi_regularity} and \ref{ass:phi_risk_integrability} should be considered accordingly.

We bound the term $\klb{\rho_T^S}{\pi}$ by first computing the \emph{entropy flow}, \ie the time derivative of $\entphi{v^S_t}$. While such an approach has already been applied in the case of pure Brownian noise $(\sigma_1 = 0)$ \citep{mou_generalization_2017}, it is significantly more technical in our case, because of the presence of the fractional Laplacian in \cref{eq:empirical_lfp}. The following lemma is an expression of the entropy flow for a general $\Phi$, that we  obtain by adapting the technique presented in \citet{gentil_logarithmic_2008} (in the study of the convergence to equilibrium of FPEs) to our setting.

\begin{restatable}[Decomposition of the entropy flow]{lemma}{thmDecomposition}
    \label{lemma:big_decomposition}
    Given a convex and differentiable function $\Phi: (0, \infty) \longrightarrow \mathds{R}$, we make Assumptions \ref{ass:phi_regularity} and \ref{ass:phi_risk_integrability}, relatively to this function $\Phi$.
    \begin{align*}
    \frac{d}{dt} \entphi{v_t^S} = -\sigma_2^2&I_\Phi(v_t^S) 
    - \sigma_1^\alpha B_\Phi^\alpha (v_t^S) \\&- \int \Phi''(v_t^S) v_t^S \nabla v_t^S \cdot \nabla \ef \uinftybar dx,
    \end{align*}
    where $I_\Phi(v) := \int \Phi''(v) \normof{\nabla v}^2 \uinftybar dx$ is called the $\Phi$-information,
    and $B_\Phi^\alpha (v)$ is called the Bregman integral, which will be formally defined in \cref{sec:proofs-main_decomposition}. 
\end{restatable}
For $\Phi(x) = x\log(x)$, the term $I_\Phi$ reduces to the celebrated \emph{Fisher information} between $\rho_T^S$ and $\pi$, denoted $J(\rho_T^S|\pi)$. It is commonly used in the analysis of FPEs \citep[Section $1$]{chafai_logarithmic_2017} and is defined as:
\begin{align*}
    J(\rho_T^S | \pi) := \int \normof{\nabla \log \frac{d\rho_T^S}{d\pi}}^2 d\rho_T^S.
\end{align*}
\cref{lemma:big_decomposition} is a central tool for the derivation of our main results. 
As a preliminary result, we first present the simpler case where $\sigma_2 > 0$. This leads to the next corollary.

\begin{restatable}{corollary}{corKLbrownian}
    \label{cor:gen_brownian_time}
    We make \cref{ass:phi_regularity} and \ref{ass:phi_risk_integrability}.
    With probability at least $1 - \zeta$ over $\datadist$, we have:
    \begin{align*}
        G_S(T) \leq s \sqrt{\frac{1}{n\sigma_2^2} I(T,S) +  4\frac{\log(3/\zeta) + \Lambda}{n} },
    \end{align*}
    where\footnote{The presence of the $\limsup_{t\to 0}$ in the definition of $\Lambda$ accounts for potential discontinuity of the KL divergence at $t=0$.} $\Lambda := \limsup_{t \to 0}\klb{\rho_t^S}{\pi}$ and $I$ is defined by:
    \begin{align}
        \label{eq:I_definition}
        I(T,S) :=  \int_0^T \E_U{\normof{\nabla\ef(W_t^S)}^2 } dt.
    \end{align}
\end{restatable}
When $f$ is $L$-Lipschitz, we have in addition $I(T,S) \leq TL^2$, recovering known bounds in the case $\sigma_1=0$ \cite{mou_generalization_2017}. \cref{cor:gen_brownian_time} may seem to have no dependence on the tail-index $\alpha$, however, it is implicitly playing a role through the integral term involving $\nabla \ef$, as $W_t^S$ is generated by a heavy-tailed SDE. Nevertheless, this bound does not apply when $\sigma_2 = 0$, which we will now investigate.

\subsection{Purely heavy-tailed case}
\label{sec:pure_levy_case}

Now we assume $\sigma_2 = 0$, which makes our task much more challenging. Indeed, in the proof of \cref{cor:gen_brownian_time}, the $\Phi$-information, $I_\Phi$, is used to compensate for the contribution of the third term in \cref{lemma:big_decomposition}. As we cannot do this anymore since $I_\Phi$ does not appear with the choice of $\sigma_2=0$, we need to develop a finer understanding of the Bregman integral term, \ie $B_\Phi^\alpha (v)$, which is the contribution of the stable noise $\levy$ to the entropy flow.

Towards this goal, in \cref{sec:bregman_integral_bounds}, we prove that, under \cref{ass:phi_regularity}, there exists a function:
\begin{align*}
    J_{\Phi, v} : [0,+\infty) \longrightarrow [0,+\infty),
\end{align*}
such that $J_{\Phi, v}$ is non-negative, continuous, satisfies $J_{\Phi, v}(0) = I_\Phi(v)$ and we have the integral representation:
\begin{align}
    \label{eq:integral_representation_main}
    B_\Phi^\alpha (v) = C_{\alpha, d}\frac{\sigma_{d-1}}{2d} \int_0^\infty J_{\Phi, v} (r) \frac{dr}{r^{\alpha - 1}},
\end{align}
where the constant $C_{\alpha,d}$ is defined in \cref{eq:spherical_representation_final_formula}, and $\sigma_{d-1}$ is the area of the unit sphere, given by \cref{eq:sphere_area}. The identification of the function $J_{\Phi, v_t^S}$ turns out to be crucial, as it illustrates that the Bregman integral term can be used for approximating a $\Phi$-information term, and therefore re-use ideas from \cref{cor:gen_brownian_time}. Thus, $J_{\Phi, v_t^S}(r)$ can be seen as an approximation of $I_\Phi(v_t^S)$, \ie, in the case $\Phi = \Phi_{\log}$, of the Fisher information $J(\rho_T^S | \pi)$, at least for small values of $r$.

A takeaway of our analysis is that, for the approximation $J_{\Phi, v}(r) \approx I_\Phi(v)$ to be accurate, we need to introduce an additional condition regarding the behavior of the function $J_{\Phi, v}$ near the origin. This assumption is specified as follows:
\begin{assumption}
    \label{ass:jv_assumption}
    There exists an absolute constant $R > 0$ such that, for all $t> 0$ and $\datadist$-almost all $S \in \zcal^n$:
    \begin{align*}
        \forall r \in [0,R],~J_{\Phi,v^S_t}(r) \geq \frac{1}{2} J_{\Phi,v^S_t}(0).
    \end{align*}
\end{assumption}
Note that, by continuity, this condition trivially holds \emph{pointwise} for fixed $S \in \zcal^n$ and $t> 0$; however, we essentially require it to hold uniformly in both time $t$ and data $S$. On the other hand, if the dynamics \eqref{eq:multifractal_dynamics} is initialized at its stationary distribution (like an ideal `warm start' \citep{dalalyan_theoretical_2016}), then $v_t^S$ is independent of $t$, in which case the statement of \cref{ass:jv_assumption} can be obtained, in high probability over $S$, through Egoroff's theorem \citep[Thm. $2.2.1$]{bogachev}. 

The factor $R$ plays an important role in our analysis, it is needed that it is positive and preferably not too small. However, we are not able to formally estimate this quantity. The exact formula for $J_{\Phi_{\log},v}$, \cref{def:spherical_fisher_info}, shows that, if $v$ is a constant function, then $J_{\Phi_{\log},v}$ is a constant function, hence $R=+\infty$ (this corresponds to the trivial case where $\ef = 0$ and the dynamics is initialized at $\uinftybar$). Therefore, we argue that $R$ can be large to get non-vacuous bounds when the function $v$ is uniformly bounded away from $0$ and has bounded first and second-order derivatives. A more formal version of this argument is provided in \cref{sec:bregman_integral_bounds}.

This allows us to prove the following theorem, which is a high probability generalization bound in the case $\sigma_2 = 0$.
\begin{restatable}{theorem}{thmLevyCaseKL}
    \label{thm:bound_under_jv_assumption}
    We make Assumptions \ref{ass:phi_regularity}, \ref{ass:phi_risk_integrability} and \ref{ass:jv_assumption}.
    Then, with probability at least $1 - \zeta$ over $\datadist$, we have
    \begin{align*}
        G_S(T) \leq 2s \sqrt{\frac{K_{\alpha, d}}{n\sigma_1^\alpha} I(T,S) + \frac{\log(3/\zeta) + \Lambda}{n}}
    \end{align*}
    with $\Lambda$ and $I(T,S)$ as in \cref{cor:gen_brownian_time}, and:
    \begin{align}
        \label{eq:K_constant}
        K_{\alpha,d} = \frac{(2 - \alpha)\Gamma \left( 1 - \frac{\alpha}{2}\right) d \Gamma \left(\frac{d}{2}\right)}{\alpha 2^\alpha \Gamma \left( \frac{d +\alpha}{2}\right) R^{2 - \alpha} },
    \end{align}  
    where $\Gamma$ denotes the Euler's Gamma function, on which more information is provided in \cref{sec:gamma_function}.
\end{restatable}
Note that, in both Theorems~\ref{cor:gen_brownian_time} and \ref{thm:bound_under_jv_assumption}, if we set the initial distribution $\rho_0 = \pi$ that has the heavy-tailed density $\uinftybar$, we get $\Lambda = 0$ and, therefore, the bound becomes tighter. This might be an argument in favor of heavy-tailed initialization, which has been considered by several studies \citep{favaro_stable_2020,jung_alpha-stable_2021}, and has been argued to be beneficial \citep{gurbuzbalaban_fractional_2021}.
We will further highlight the quantitative properties of \cref{thm:bound_under_jv_assumption} in \cref{sec:qualitative_analysis}.

The proof of \cref{thm:bound_under_jv_assumption} would also apply when $\sigma_2>0$, however, compared to \cref{sec:brownian case}, it requires the additional \cref{ass:jv_assumption}. The bound of \cref{cor:gen_brownian_time} was obtained by using mainly the contribution of $B_t$ to the noise, while \cref{thm:bound_under_jv_assumption} corresponds to the contribution of $\levy$. It turns out that both approaches can be combined, under \cref{ass:jv_assumption}; it is presented in \cref{sec:noise_mixing}. 

\cref{thm:bound_under_jv_assumption} (valid for $1 < \alpha < 2$) should be compared with existing generalization bounds for continuous Langevin dynamics (CLD), \ie, $\alpha=2$, where integral terms that are similar to $I(T,S)$ appear \citep{mou_generalization_2017,li_generalization_2020,futami_time-independent_2023,dupuis2024setpacbayes}. Our bound features the new constant $K_{\alpha,d}/\sigma_1^\alpha$ and we show in \cref{sec:qualitative_analysis} that we recover similar constants to the CLD case in the limit $\alpha \to 2^-$.
Compared to \cite{mou_generalization_2017} (in the case of CLD), $I(T,S)$ does not contain any exponential time decay. This point is discussed in detail in \cref{sec:toward_time_uniform,sec:time-dependence_discussion}. Despite this fact, $I(T,S)$ can still be small because the norm of the gradients may become small.

\subsection{Towards time-uniform bounds}
\label{sec:toward_time_uniform}

\cref{cor:gen_brownian_time} and \cref{thm:bound_under_jv_assumption}, while being the first high probability bounds for heavy-tailed dynamics with explicit constants, may suffer from a time-dependence issue. The reasons why this is an outcome of our proofs are discussed in \cref{sec:time-dependence_discussion}. It appears from this discussion that \cref{thm:bound_under_jv_assumption} can be made time-uniform, under the existence of a specific class of functional inequalities. Unfortunately, we argue that such techniques do not always apply in our case, as it is detailed in \cref{sec:time-dependence_discussion}.

Nevertheless, in this section, we take a first step towards improving the time-dependence of the bounds, derived in our setting. However, this comes at the cost of weakening the interpretability of the bound and might make it hard to compute in practice. In order to present this result, we need to make another choice for the convex function $\Phi$: we consider $\Phi(x) = \Phi_2(x) := \frac{1}{2} x^2$, instead of $\Phi_{\log}$. This choice is justified by the fact that it significantly changes the structure of the Bregman integral term, \ie $B_\Phi^\alpha(v_t^S)$, in a way that is clearly presented in the proofs of \cref{sec:proofs-poincare-inequality}.

We only discuss the case $\sigma_2 = 0$, the case $\sigma_2>0$ can be found in \cref{sec:proofs-poincare-inequality}. Following the reasoning of \cref{sec:pure_levy_case}, we make \cref{ass:jv_assumption} with the convex function $\Phi_2$ instead of $\Phi_{\log}$. We will refer to it as \cref{ass:jv_assumption}$-\Phi_2$. This leads to our last theoretical result.

\begin{restatable}{theorem}{thmLevyCaseChi}
    \label{thm:chi_sq_levy}
    Let $\sigma_2 = 0$. We make Assumptions \ref{ass:jv_assumption}$-\Phi_2$, \ref{ass:phi_regularity} and \ref{ass:phi_risk_integrability}, with the choice $\Phi = \Phi_2$.
    Then, with probability at least $1 - \zeta$ over $S \sim \datadist$ and $w \sim \rho_T^S$, we have
    \begin{align*}
        G_S(w) \leq 2s \sqrt{\frac{4K_{\alpha, d}}{n\sigma_1^\alpha} \Tilde{I}(T,S) + \frac{2e^{-\frac{\alpha \eta T}{2}}\Lambda + \log \frac{24}{\zeta^3}}{n}},
    \end{align*}
    with $\Lambda = \entphi{\rho_0}$, and
    \begin{align*}
        \Tilde{I}(T,S) := \int_0^T e^{-\frac{\alpha \eta}{2} (T - t)} \Eof[\pi]{(v_t^S)^2 \normof{\nabla \ef}^2  } dt.
    \end{align*}
\end{restatable}
While the exponential decay term, \eg $e^{-\eta (T - t)}$ is a significant improvement over \cref{thm:bound_under_jv_assumption}, the integral term $\Tilde{I}(T,S)$ is less interpretable than the term $I(T,S)$, appearing in \cref{cor:gen_brownian_time} and \cref{thm:bound_under_jv_assumption}.
Indeed, $I(T,S)$ is simply related to the expected gradient of the empirical risk.

\section{Quantitative Analysis}
\label{sec:qualitative_analysis}

We focus our qualitative and experimental analysis on the results obtained in the case of pure heavy-tailed dynamics ($\sigma_2 = 0$), namely \cref{thm:bound_under_jv_assumption,thm:chi_sq_levy}, as they bring the most novelty compared to the literature. 
From now on, we assume that the constant $R$, coming from \cref{ass:jv_assumption}, can be taken independent of $\alpha$, $\sigma$ and $d$. This assumption has important consequences for our quantitative analysis.

\textbf{Asymptotic analysis.} We analyze the behavior of the constant $K_{\alpha,d}$, appearing in \cref{thm:bound_under_jv_assumption,thm:chi_sq_levy}. Let $\bar{K}_{\alpha, d} = R^{2 - \alpha} K_{\alpha, d} $, the following lemma provides an asymptotic formula of this constant, when the number of parameters $d$ goes to infinity. This is pertinent as modern machine learning models typically have a lot of parameters.

\begin{restatable}{lemma}{constantDimensionEquivalent}
    \label{lemma:d_limit}
    We have that, for all $\alpha \in (1,2)$:
    \begin{align}
        \label{eq:constant_equivalent}
        \bar{K}_{\alpha, d} \underset{d \to \infty}{\sim} P_\alpha d^{1 - \frac{\alpha}{2}}
     , \quad P_\alpha := \frac{(2 - \alpha) \Gamma \left( 1 - \frac{\alpha}{2} \right)}{\alpha 2^{\alpha / 2}},
    \end{align}
\end{restatable}
In \cref{eq:constant_equivalent}, we isolated a term $P_\alpha$ depending only on $\alpha$ and a dimension dependent term, $d^{1 - \frac{\alpha}{2}}$. A quick analysis (see \cref{sec:proofs-qualitative_analysis}) shows that the pre-factor $\alpha \longmapsto P_\alpha$ is decreasing in $(1,2)$ and satisfies $\frac{1}{2} \leq P_\alpha \leq \sqrt{\frac{\pi}{2}}$.

Despite being proven for $\alpha<2$, our bounds do not explode when $\alpha \to 2^-$, as we show in the following lemma.

\begin{restatable}{lemma}{constantAlphaLimit}
\label{lemma:alpha_limit}
For any $d \geq 1$, we have $K_{\alpha, d}  \underset{\alpha \to 2^-}{\longrightarrow} \frac{1}{2}$.
\end{restatable}

 \textbf{Phase transition.} By \cref{lemma:d_limit}, in the limit $d \gg 1$, \cref{thm:bound_under_jv_assumption} becomes:
\begin{align}
    \label{eq:bound_in_bounded_case}
    G_S(T) \leq 2s\sqrt{ \frac{P_\alpha d^{1 - \frac{\alpha}{2}}}{n\sigma_1^\alpha R^{2 - \alpha}} I(T,S)  + \frac{\Lambda + \log \frac{3}{\zeta}}{n} }.
\end{align}
We can rewrite the constant term, multiplying $I(T,S)$, as:
\begin{align*}
    \frac{P_\alpha d^{1 - \frac{\alpha}{2}}}{n\sigma_1^\alpha R^{2 - \alpha}} = \frac{P_\alpha d_0^{1 - \frac{\alpha}{2}}}{n\sigma_1^\alpha} = \frac{P_\alpha d_0}{n(\sigma_1 \sqrt{d_0})^\alpha},
\end{align*}
where we introduced a `reduced dimension' $d_0 := d/(R^2)$. 
As mentioned earlier, we have, for all $\alpha \in (1,2)$, that $\frac{1}{2} \leq P_\alpha \leq \sqrt{\frac{\pi}{2}}$. Therefore, it is clear that the main influence of the tail-index $\alpha$ on the generalization bounds is induced by the geometric term $(\sigma_1 \sqrt{d_0})^{-\alpha}$.
Based on this observation, our bounds suggest a phase transition between two regimes:
\begin{itemize}[noitemsep,topsep=0pt,leftmargin=.2in]
    \item Heavy regime: $(\sigma_1 \sqrt{d_0}) < 1$, the generalization error increases with the tail, \ie the performance should be better with heavier-tails. If we take into account the contribution of the factor $P_\alpha$, this condition becomes $(\sigma_1 \sqrt{d_0}) < 1/\sqrt{2\pi}$, see \cref{sec:low_noise_regime_precision}.
    \vskip -0.1in
    \item Light regime:  $(\sigma_1 \sqrt{d_0}) > 1$, heavy-tails are harmful for the generalization bound.
\end{itemize}

This shows that, depending on the setting and the structure of the dynamics, heavy tails may have a different impact on the generalization error.

\textbf{Comparison with existing works.} In \citep{raj_algorithmic_2023}, the authors studied  \cref{eq:multifractal_dynamics}, with $\sigma_1=1$, $\sigma_2=0$, and $f$ Lipschitz continuous.\footnote{\citet{raj_algorithmic_2023} consider a Lipschitz loss $\ell$ and a surrogate $f$, that has a dissipativity property; we can frame it within our setting by assuming that $f$ is Lipschitz in $w$. }
Informally, the obtained bound is:
\begin{align}
    \label{eq:levy_anant}
    \Eof[S,U]{L(W^S_\infty) - \el(W^S_\infty)} \leq \frac{\normof{\ell}_{\text{Lip}} A R_{\alpha,d} }{n},
\end{align}
where $A$ is a quantity that has a complex dependence on various constants appearing in the assumptions, $\normof{\ell}_{\text{Lip}}$ is the Lipschitz constant of $\ell$, which is assumed finite, and $R_{\alpha,d}$ is a constant, explicitly given in \cref{sec:raj_comparison_appendix}, where we also show that it satisfies $R_{\alpha, d} = \mathcal{O}_{d\to\infty}(d^{\frac{1+\alpha}{2}})$.

We already mentioned, in \cref{sec:intro}, some differences between \cref{eq:levy_anant} and our results. We additionally emphasize that \textit{(i)} we do not require a Lipschitz assumption, and \textit{(ii)} The constant $R_{\alpha,d}$ has a worse dependence on the dimension $d$ than the constant $K_{\alpha,d}$, appearing in our theorems.
\cref{eq:levy_anant} cannot explain generalization in an overparameterized regime, \ie when $d>n$. Moreover, in the limit $\alpha \to 2^-$, it does not yield the known dimension dependence for Langevin dynamics \citep{mou_generalization_2017,pensia_generalization_2018,farghly_time-independent_2021}, while \cref{lemma:alpha_limit} shows that $K_{\alpha,d}$ becomes independent of $d$ when $\alpha \to 2^-$.

To have a fair comparison, we shall note that \eqref{eq:levy_anant} does not increase with the time horizon $T$, whereas it is the main drawback of our bounds. Nevertheless, the results of \cref{sec:toward_time_uniform,sec:time-dependence_discussion} show that this point might have room for improvement, which we leave as future work.

\section{Empirical Analysis}
\label{sec:experiments}

 \textbf{Setup.} We numerically approximate \cref{eq:multifractal_dynamics}, using its Euler-Maruyama discretization \citep{duan_introduction_2015}, $\forall k \in \set{1,\dots,N}$, 
\begin{align}
    \label{eq:Euler-Maruyama}
    \widehat{W}^S_{k+1} = \widehat{W}^S_k - \gamma \nabla \ef (\widehat{W}^S_k) - \eta \gamma \widehat{W}^S_k + \gamma^{\frac{1}{\alpha}} \sigma_1 L_1^\alpha,
\end{align}
where $\gamma > 0$ and $N \in \mathds{N}$ are fixed learning rate and number of iterations. 
Our main experiments were conducted with $2$ layers fully-connected networks (FCN$2$) trained on the MNIST dataset \citep{lecun_gradient-based_1998}. Additional experiments, using MNIST, FasionMNIST  \citep{xiao_fashion-mnist_2017} and CIFAR$10$ datasets \cite{krizhevsky_cifar-10_2014}, as well as linear models and deeper networks, are presented in \cref{sec:additional_experiments}. 
We choose the objective $\ell$ as the $0$-$1$ loss and the surrogate $f$ (that we used for training) as the cross entropy loss. These choices make our experiments as close as possible to our theoretical setting, still allowing us to have a varying number of parameters $d$. 
Each experiment is run with $10$ different random seeds. All hyperparameters details can be found in \cref{sec:hyperparameters}.

We provide, in \cref{sec:discrete_case}, an additional analysis justifying that our continuous-time theory is still pertinent to study the discrete one, 
\cref{eq:Euler_Maruyama_batch_size}, thus providing sufficient theoretical foundations for our experiments. 

The estimation of the accuracy is subject to important noise, due to the jumps incurred by $\levy$. To act against this noise, we first use $\alpha \in [1.6, 2]$. This range is also coherent with estimated tail indices in practical settings by \citet{raj_algorithmic_2023-1,barsbey_heavy_2021-1}. Moreover, the accuracy gap is (robustly) averaged over the last iterations, see \cref{sec:robust_mean_estimation}.

As shown in \cref{eq:Euler-Maruyama}, we use the full dataset $S$ at each iteration, in accordance with the model that we study in this paper. Moreover, it has been argued in several studies \citep{gurbuzbalaban_heavy-tail_2021,hodgkinson_multiplicative_2020,barsbey_heavy_2021-1} that SGD may create heavy-tailed behavior, an effect whose interaction with the noise $\levy$ would be unclear. Our setting allows us to isolate the effect of $\levy$ on the generalization error. To make our experiments tractable, we sub-sample $10\%$ of the MNIST and FashionMNIST datasets to run our main experiments.
To show that our theory may stay pertinent in more practical settings, we estimated our bound when training a FCN$5$ on the whole MNIST dataset, with smaller batches, see \cref{sec:additiional_full_batch_appendix}. 

\paragraph{Lévy processes simulation.} As shown by \cref{eq:Euler-Maruyama}, the numerical approximation of $\widehat{W}_k^S$ requires the simulation of the Lévy process $\levy$, which is estimated by independent realization of $L_1^{\alpha}$. Simulating $\alpha$-stable Lévy processes is standard in probabilistic simulation. In our case, we use the method described by \citet[Section $1$]{Nolan2013MultivariateEC}. More precisely, we sample $L_1^{\alpha}$ as $L_1^{\alpha} = \sqrt{A} G$, where $G \sim \mathcal{N}(0, I_d)$ and $A$ is a skewed stable distribution given by:
\begin{align*}
    A \sim S\left( \frac{\alpha}{2}, 1, 2 \cos \left( \frac{\pi \alpha}{4} \right)^{2/\alpha}, 0 \right),
\end{align*}
where $S(\alpha, \beta,c, \mu)$ denotes stable distributions, with $\beta$ the skewness parameter and $c$ the scale parameter, see \cite{duan_introduction_2015} for more details. This model was in particular used to generate the Lévy processes in \cref{fig:levy-processes}.

\begin{figure}[t]
    \centering
    \includegraphics[width=0.9\columnwidth]{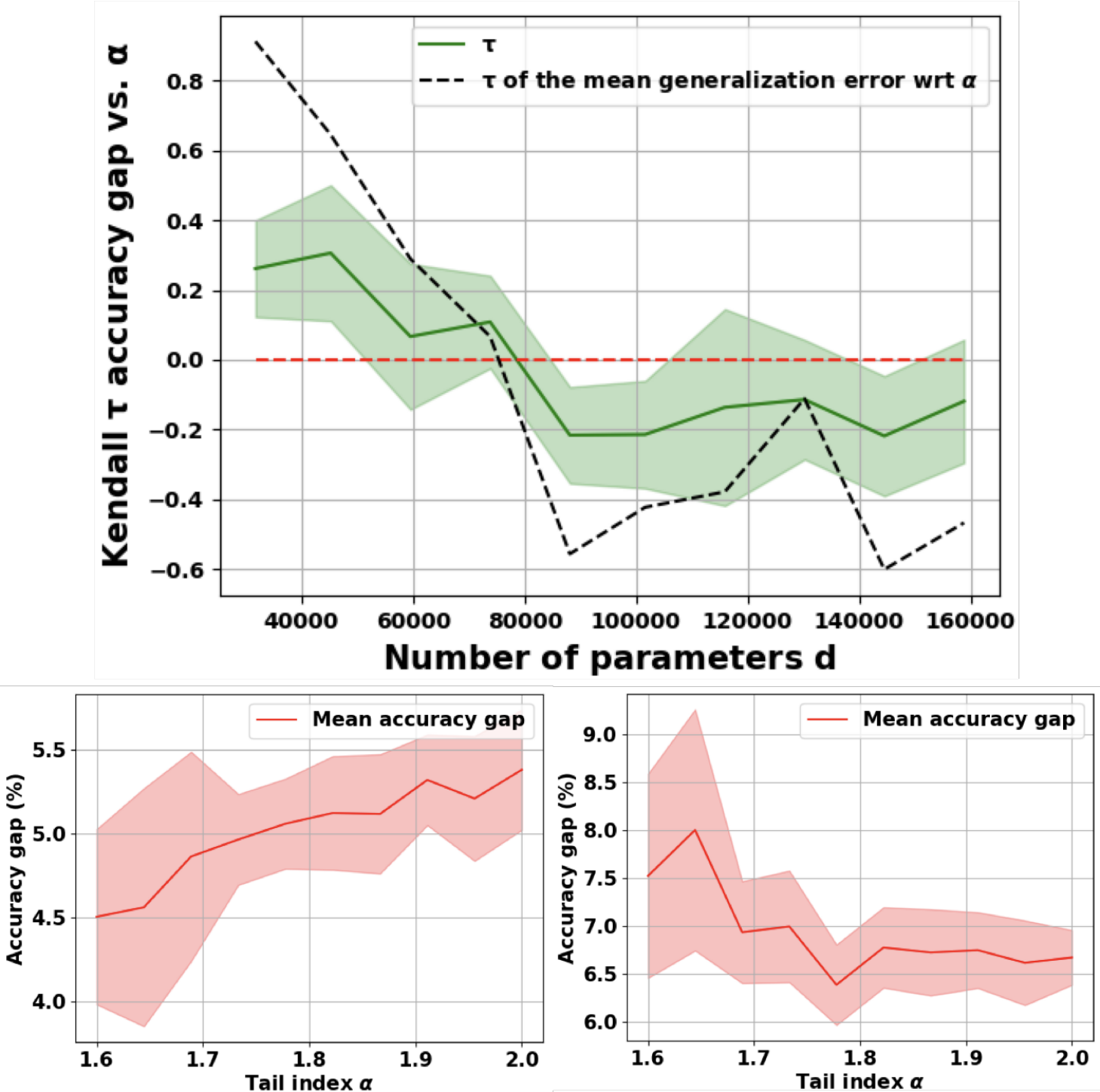}
    \vspace{-7pt}
    \caption{\textit{(up)} Correlation (Kendall's $\tau$) between $\alpha$ and the accuracy gap, for different values of $d$, with a FCN$2$ trained on MNIST. The green curve is the average $\tau$ over $10$ random seeds. The black curve is the correlation between $\alpha$ and the average accuracy gap over $10$ seeds. \textit{(bottom)} Accuracy gap with respect to $\alpha$ for $d=3 \cdot 10^4$ \textit{(left)} and $d=15 \cdot 10^4$ \textit{(right)}.}
    \label{fig:correlation_main}
    \vskip -0.1in
\end{figure}

\begin{figure}[!b]
    \centering
    \vspace{-10pt}
    \includegraphics[trim={0 0cm 0 0},clip,width=\columnwidth]{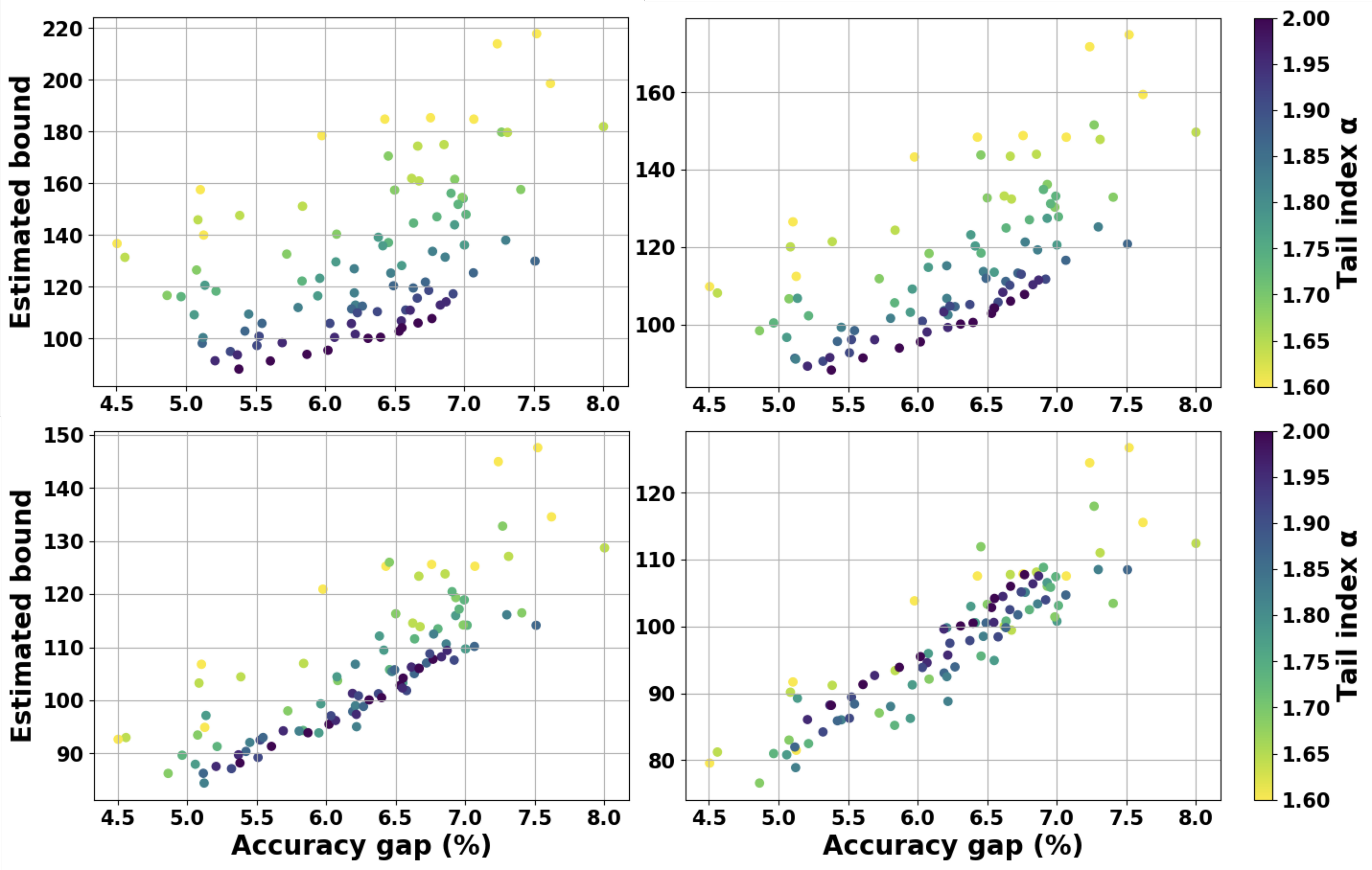}
    \vspace{-20pt}
    \caption{Bound estimated from \cref{eq:G_estimation_main} versus accuracy gap for a FCN$2$ on MNIST, for different values of $R$: $1$ \textit{(top left)}, $3$ \textit{(top right)}, $7$ \textit{(bottom left)}, $15$ \textit{(bottom right)}.}
    \label{fig:all_R}
\end{figure}

\textbf{Results.} We test our theory through $3$ types of experiments.
We present in this section their results for a FCN$2$ trained on MNIST. \cref{sec:additional_experiments} contains additional experiments.

First, on \cref{fig:correlation_main}, we compute the correlation between $\alpha$ and the accuracy gap, measured in term of a Kendall's $\tau$ coefficient\footnote{The sign of $\tau$ corresponds to the sign of the correlation.}\citep{kendall_new_1938}.
We use a FCN$2$ and let the width vary to compute $\tau$ for different values of the dimension $d$. 
The detailed procedure to obtain \cref{fig:correlation_main} can be found in \cref{sec:procedure_details,sec:hyperparameters}.
We observe that the phase transition between positive and negative correlation, predicted in \cref{sec:qualitative_analysis}, happens for a value of the dimension $d \simeq 8 \cdot 10^4$, which we will use to further estimate $R$.
We observe that the positive correlation of the heavy regime seems to be stronger than the negative correlation of the light regime. As an additional experiment, we also provide the same plot as \cref{fig:correlation_main} in \cref{sec:additional_experiments}, but using the Pearson correlation coefficient instead of $\tau$. These results, displayed in \cref{fig:pearson_plot}, yield the same empirical results than \cref{fig:correlation_main}.

\begin{figure}[t]
    \centering
    \includegraphics[trim={1cm, 0cm, 0cm, 0cm},clip, width=0.9\columnwidth]{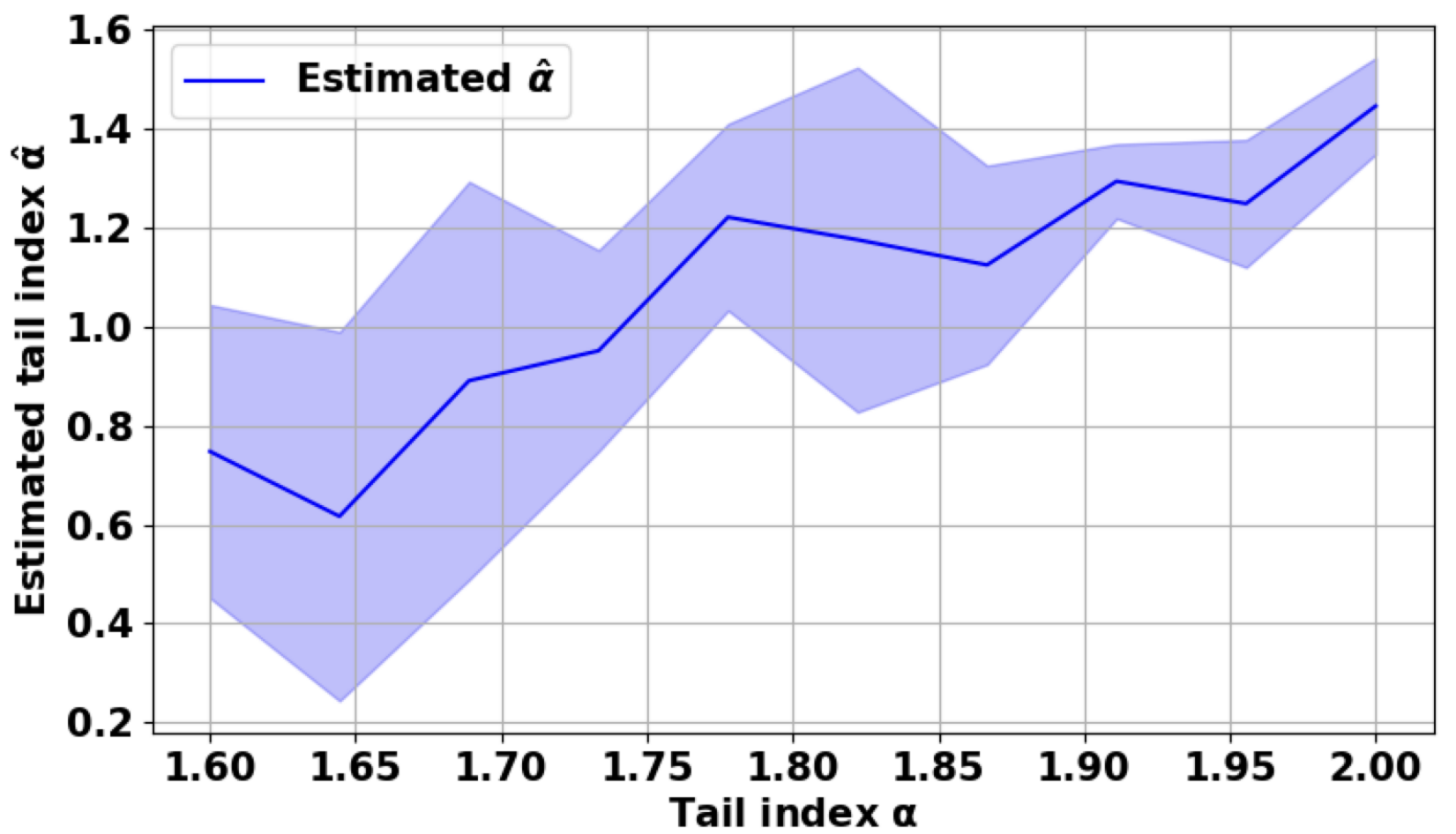}
    \vspace{-7pt}
    \caption{Regression of the tail-index $\alpha$ from the accuracy error, for a FCN$2$ trained on MNIST.}
    \label{fig:regression_from_d}
    \vskip -0.2in
\end{figure}

The bound of \cref{thm:bound_under_jv_assumption} is computable in practice, we estimate it by the formula (in that case $s=1/2$):
\begin{align}
\label{eq:G_estimation_main}
\widehat{G} := \sqrt{\frac{P_\alpha d^{1 - \frac{\alpha}{2}} \gamma }{n \sigma_1 R^{2 - \alpha} } \sum_{k=1}^N \normof{\nabla \ef (\widehat{W}_k^S)}^2 }.
\end{align}
On \cref{fig:all_R}, we plot \cref{eq:G_estimation_main} \wrt to the accuracy gap, for several values of $R$. We use $R=1$ as a default choice, as $R$ is unknown \textit{a priori}, it shows, for each value of $\alpha$, a good correlation with the accuracy gap. Nonetheless, based on \cref{fig:correlation_main}, conducted with the same setting as \cref{fig:all_R}, we can estimate the value of $R$ to be in $\simeq [2.8, 7]$. If we use these values in \cref{eq:G_estimation_main}, the observed correlation is much stronger. 
If we use a slightly larger value ($R=15$), we see, in \cref{fig:all_R}, that the correlation becomes almost perfect, which we interpret as the phase transition being correctly taken into account. This shows that the right corrective term in \cref{eq:estimation_formula} is indeed of the form $R^{2-\alpha}$.
We note that the reason why our bound over-estimates the accuracy gap, is that it increases with $T$. However, as the bounds presented in  \cref{sec:toward_time_uniform,sec:time-dependence_discussion} are time-uniform and have similar constants and dependence on $\nabla \ef $ as in \cref{thm:bound_under_jv_assumption}, we believe that this issue can be alleviated by extending \cref{thm:bound_under_jv_assumption} in a similar direction, which we leave as an open question.

Finally, to obtain \cref{fig:regression_from_d}, we fixed $\sigma$ to $0.01$ and let $d$ vary in a fixed range. Based on \cref{eq:bound_in_bounded_case}, we expect the accuracy error, denoted $G_S$, to be proportional to $d^{1/2 - \alpha / 4}$. This suggests to perform the linear regression, $\log(G_S) \simeq \widehat{r} \log(d) + C,$ to compute an estimate $\widehat{\alpha} := 2 - 4 \widehat{r}$ of the tail-index $\alpha$. The blue curve in \cref{fig:regression_from_d} shows $\widehat{\alpha}$ in terms of $\alpha$. This shows a strong correlation between the estimated and the ground-truth tail-index, in particular, we retrieve the expected monotonicity. 
However, $\widehat{\alpha}$ seems to underestimate the true value of $\alpha$, by a term independent of $\alpha$. We suspect that this may be because other terms in the bound have a dependence on $d$, or because our bound is not a strict equality, which we assumed to compute $\widehat{\alpha}$ from $G_S$.

\section{Conclusion}
\label{sec:conclusion}
In this paper, we proved generalization bounds for heavy-tailed SDEs. Our results are the first to be both in high-probability and computable. Moreover, they allow for a more flexible setup and have a better dimension-dependence than existing works. We analyzed the constants appearing in our theorems, which led us to predict the existence of a phase transition in terms of the effect of the tail index on the generalization. We supported our theory with various numerical experiments.
Several directions remain to be studied in the future. In particular, obtaining new functional inequalities, such as presented in \cref{sec:time-dependence_discussion}, could improve the time-dependence of the bounds. Moreover, understanding the interaction, between small batches and the stable noise $\levy$, would be a natural extension of the theory.

\section*{Acknowledgments}
We thank Paul Viallard, Maxime Haddouche and Isabelle Tristani for valuable discussions.
 U.\c{S}. is partially supported by the French government under management of
Agence Nationale de la Recherche as part of the ``Investissements d'avenir'' program, reference
ANR-19-P3IA-0001 (PRAIRIE 3IA Institute). B.D. and U.\c{S}. are partially supported by the European Research Council Starting Grant
DYNASTY – 101039676.

\section*{Impact Statement} This work is largely theoretical, it does not have any direct social or ethical impact.

\bibliography{main.bib}
\bibliographystyle{icml2024}

\newpage

\newpage
\appendix
\onecolumn

\textbf{Organization of the appendix:} The appendix starts with a short notations section. The remainder of the document is then organized as follows:
\begin{itemize}
    \item In \cref{sec:background_appendix}, some technical background is presented. The technical background is divided into three main topics: PAC-Bayesian bounds, Lévy processes, and $\Phi$-entropies inequalities. We also include a small subsection on the Euler's $\Gamma$ function.
    \item \cref{sec:proofs-subgaussian-pb} presents the proof of our version of a PAC-Bayesian generalization bound for subgaussian losses.
    \item \cref{sec:proof_of_main_results} is the core of the appendix, we prove the main result, along with all intermediary lemmas, and introduce the notations necessary to understand those proofs. Moreover, a few additional theoretical results are given, which are a refinement of the main results. In particular, the extension of the theory to a discrete setting is discussed in \cref{sec:discrete_case}, while \cref{sec:noise_mixing} presents bounds in the case $\sigma_2>0$, which are different than those of \cref{sec:brownian case}.
    \item In \cref{sec:proofs-qualitative_analysis}, we provide details on how to obtain the theoretical results of \cref{sec:proofs-qualitative_analysis}.
    \item In \cref{sec:time-dependence_discussion}, we discuss the time dependence of \cref{cor:gen_brownian_time,thm:bound_under_jv_assumption} and mention that this time dependence could be largely improved by assuming that a certain class of inequality holds. 
    \item Finally, \cref{sec:experiments_appendix} presents some details on the experimental setting, as well as a few additional experiments.
\end{itemize}

\section*{Notations}

In order to simplify the notations, we will sometimes omit the $x$ variable when integrating with respect to the Lebesgue measure in $\Rd$, \ie we will write invariable $\int f$ or $\int f dx$, instead $\int f(x) dx$. These conventions are meant to ease the notations throughout the paper.

We will use the following convention regarding the Fourier transform, for $\phi: \mathds{R}^d \longrightarrow \R$, regular enough, we set:
\begin{align}
    \label{eq:fourier_transform}
    \mathcal{F}\phi (\xi) := \int e^{-i x \cdot \xi} \phi(x) dx.
\end{align}

The partial derivative $\partial / \partial t$ will often be shortened as $\partial_t$. Similarly, $\partial_i$ may denote $\partial / \partial x_i$.

Let's also precise some notations introduced in the main part of the document. As mentioned in \cref{sec:intro}, the data space is denoted $\zcal$. More precisely, $\zcal$ is a measurable space, endowed with a $\sigma$-algebra $\mathcal{F}$. The data distribution, $\mu_z$, is a probability measure on $(\zcal, \mathcal{F})$.

\section{Technical background}
\label{sec:background_appendix}

\subsection{Information-theoretic terms and PAC-Bayesian bounds}
\label{sec:it_terms_and_pb_bounds}

The concept of PAC-Bayesian analysis has been introduced in \cref{sec:pac_bayesian_boudns}. In this section, we detail two particular PAC-Bayesian bounds that we use for the derivation of our main results. For a more detailed introduction to those subjects, the reader is invited to consult \citep{alquier_user-friendly_2021}.

We start by defining the information theoretic (IT) quantities appearing in the aforementioned theorems, see \citep{van_erven_renyi_2014} for more details. Let $\mu$ and $\nu$ be two probability measures, on the same space, such that $\mu$ is absolutely continuous with respect to $\nu$. We define the Kullback-Leibler (KL) divergence as:
\begin{align}
    \label{eq:def_kl}
    \klb{\mu}{\nu} := \int \log\left(\frac{d \mu}{d \nu}\right) d\mu,
\end{align}
where $d\mu/d\nu$ denotes the Radon-Nykodym derivative between $\mu$ and $\nu$.

Next, we define the Renyi divergences, for some $\beta > 1$, as \footnote{There are definitions of Renyi divergences for other values of $\beta$, but we won't need them in this paper, see \citep{van_erven_renyi_2014}.}:
\begin{align}
    \label{eq:definition_renyi}
    \renyi[\beta]{\mu}{\nu} := \frac{1}{\beta - 1} \log \left( \int \left(\frac{d \mu}{d \nu}\right)^\beta d\nu \right).
\end{align}
By convention, we set $\renyi[1]{\cdot}{\cdot} := \klb{\cdot}{\cdot}$, so that, by \citep[Theorem $3$]{van_erven_renyi_2014}, $\renyi[\beta]{\cdot}{\cdot}$ is nondecreasing in $\beta$.

In the following, to mimic the notations of the rest of the paper, we consider a probability measure $\pi$ on $\Rd$ and a family of data-dependent probability measures on $\Rd$, $(\rho_S)_{S\in\zcal^n}$, where $\zcal$ has been defined in \cref{sec:intro}. We mainly require this family to satisfy the following properties:
\begin{enumerate}
    \item Absolute continuity: for (almost-)all $S$, we have $\rho_S \ll \pi$.
    \item Markov kernel property, for all Borel set $B \subset \Rd$, the map $S \longmapsto \rho_S(B)$ is $\mathcal{F}^{\otimes n}$-measurable. Recall that $\mathcal{F}$ is the $\sigma$-algebra on the data space $\zcal$. 
\end{enumerate}

In the following, as we do in the rest of the paper, we refer to $\pi$ as a prior distribution, and to $\rho_S$ as posterior distributions.

The next theorem is a generic PAC-bayesian bound due to \citet{germain_pac-bayesian_2009}.

\begin{theorem}[General PAC-Bayesian bound]
\label{thm:pac_bayesian_kl}
Let $\zeta \in (0,1)$ and $\varphi: \Rd \times \Zcal^n \to \R$ a measurable function, integrable with respect to the posterior distributions. With probability at least $1-\zeta$, over $S \sim \mu_z^{\otimes n}$, we have:
\begin{align*}
\Eof[w \sim\rho_{S}]{\varphi(w, S)} \leq \log(1/\zeta) + \klb{\rho_{S}}{\pi} + \log \mathds{E}_S \Eof[w \sim \pi]{e^{\varphi(w, S)}},
\end{align*}
where the KL divergence has been defined by \cref{eq:def_kl}.
\end{theorem}

\begin{theorem}[Disintegrated PAC-Bayesian bound]
\label{thm:disintegrated_pac_bayes__bound}
Let $\zeta \in (0,1)$ and $\varphi: \Rd \times \Zcal^n \to \R$ a measurable function. With probability at least $1-\zeta$, over $S \sim \mu_z^{\otimes n}$ and $w \sim \rho_S$, we have:
\begin{align*}
 \frac{\beta}{\beta{-}1}\varphi(w, S) \le \frac{2\beta-1}{\beta-1}\log(2/\zeta) + \renyi[\beta]{\rho_{S}}{\pi} + \log \mathds{E}_S \Eof[w \sim \pi]{e^{\frac{\beta}{\beta-1}\varphi(w, S)}},
\end{align*}
where the Renyi divergence has been defined by \cref{eq:definition_renyi}.
\end{theorem}

We give below one particular instance of \cref{thm:pac_bayesian_kl}, using the notations introduced in \cref{sec:intro}, for $\mu_z$, $S$,  $f$, $L$ and $\el$. This theorem was first proven by \citep{mcallester_pac-bayesian_2003,maurer_note_2004}.

\begin{theorem}
    \label{thm:PAC_Bayesian_bounded_loss}
    Assume that the objective $f$ is bounded in $[0,1]$, then, with probability at least $1 - \zeta$ over $S \sim \datadist$, we have:
    \begin{align*}
        \Eof[\rho_S]{L(w) - \el(w)} \leq \sqrt{\frac{\klb{\rho_S}{\pi} + \log \frac{2 \sqrt{n}}{\zeta}}{2n}}.
    \end{align*}
\end{theorem}

\subsection{Background on Levy processes and associated pseudo-differential operators}
\label{sec:background_levy_markov}

In this section, we introduce some basic notions related to the study of Lévy process. In particular, we insist on the case of stable Lévy processes and their associated operators, namely the Laplacian and fractional Laplacian. Therefore, we make the link between the SDE \eqref{eq:multifractal_dynamics} and the PDE \eqref{eq:empirical_lfp} as clear as possible for the reader. We will also set up several notations used throughout the sequel.

\subsubsection{Levy processes}

In this subsection, we recall some basic notions related to Lévy process, in order to make our main results as clear as possible. The main goal is to get an understanding of \cref{eq:empirical_lfp}. For a more detailed introduction to those subjects, we refer the reader to \citep{schilling_feller_1998, xiao_random_2004, bottcher_levy_2013}. 

\begin{definition}[Lévy process]
    A Lévy process $(L_t)_{t\geq 0}$, in $\Rd$, is a stochastic process such that:
    \begin{itemize}
        \item $L_0 = 0$,
        \item the increments are independents, \ie, for all $t_1 < \dots < t_K$, the processes $L_{t_i} - L_{t_{i-1}}$ are independent,
        \item the increments are stationary, \ie for $0\leq s < t$, we have $L_t - L_s \equald L_s$, where $\equald$ denotes the equality in distribution,
        \item the process is stochastically continuous, by which we mean that for any $s\geq 0$ and $\delta > 0$, we have:
        \begin{align*}
           \lim_{t\to s} \Pof{\normof{L_t - L_s} > \delta} = 0.
        \end{align*}
    \end{itemize}
    Equivalently, one can show that stochastic continuity is equivalent to the process having a modification with cadlag paths\footnote{cadlag means right continuous and having a left limit everywhere.}, therefore, the paths of Lévy processes may exhibit jumps.
\end{definition}

Levy processes are closely related to the notion of \emph{infinitely divisible distributions}. A probability distribution is said to be infinitely divisible if, for any $N \in \mathds{N}^\star$, it can be seen that $F$ is the distribution of the sum of $N$ \iid random variables. 

Lévy processes have infinitely divisible distributions\footnote{This is an equivalence: every infinitely divisible distribution is naturally associated with a Lévy process.}. Following \citet[Corollary $2.5$]{schilling_symbol_2010}, it can be deduced that their characteristic function can be expressed as (with $\xi \in \Rd$):
\begin{align*}
    \Eof{e^{i\xi \cdot L_t }} = e^{-t\psi(\xi)},
\end{align*}
where the function $\psi$ is called the \emph{characteristic exponent}. It characterizes the Lévy process $(L_t)$ and plays a great role in our analysis.

The following theorem is the fundamental result in the study of characteristic exponents. In particular, it introduces the notion of Lévy measure, which we use repeatedly. Several conventions or notations may exist for this formula, we follow those of \citep[Theorem $2.2$]{bottcher_levy_2013}.

\begin{theorem}[Lévy-Khintchine formula]
    \label{thm:levy_khintchine}
    Let $(L_t)_{t\geq 0}$ be a Lévy process as above. The characteristic exponent of $L$ has the following form:
    \begin{align*}
        \forall \xi \in \Rd,~\psi(\xi) = -i l \cdot \xi + \frac{1}{2} \xi \cdot Q \xi + \int_{\Rd\backslash\set{0}} \left( 1 - e^{i z\cdot\xi} + i \xi \cdot z \chi(\normof{z}) \right) d\nu(z),
    \end{align*}
    where $l \in \Rd$, $Q \in \mathds{R}^{d \times d}$ is a symmetric positive semi-definite matrix and $\nu$ is a positive measure on $\Rd\backslash\set{0}$ such that
    \begin{align*}
        \int_{\Rd\backslash\set{0}} \min(1,\normof{z}^2) d\nu(z) < +\infty.
    \end{align*}
    Finally, $\chi$ is a \emph{truncation function}, such that $\chi(s)$ and $s\chi(s)$ are bounded and there exists a constant $\kappa > 0$ such that $0 \leq 1 - \chi(s) \leq \kappa \min(1,s)$.

    The triplet $(l,Q,\nu)$ is called the \emph{Lévy triplet} associated to $L$, and $\nu$ is the \emph{Lévy measure}.
\end{theorem}

\begin{remark}
    As it is mentioned in \citep[Theorem $2.2$]{bottcher_levy_2013}, the truncation function $\chi$ is arbitrary and only influences the drift $l \in \Rd$. In our paper, we only consider Lévy process and infinitely divisible distributions with no drift (\ie $l=0$), the choice of $\chi$ therefore has no impact. A typical choice would be $\chi(s) = \min(1,s^2)$.
\end{remark}

\begin{remark}
    It is clear, from the above discussion, that any infinitely divisible distribution can be associated with a Lévy triplet, as in \cref{thm:levy_khintchine}.
\end{remark}

\begin{example}[Brownian motion]
    The Lévy triplet $(0, I_d, 0)$ corresponds to the standard Brownian motion in $\Rd$, denoted $(B_t)_{t\geq 0}$.
\end{example}

We end this subsection by defining stable Lévy processes, which are the main object of our study, see \citep[Example $2.4.d$]{bottcher_levy_2013}.

\begin{definition}[Stable Lévy processes]
    \label{def:stable_levy_processes}
    Let $\alpha \in (0,2]$, the (isotropic) $\alpha$-stable Lévy process $(\levy)_{t\geq 0}$, is defined by the following expression of its characteristic exponent: $\psi(\xi) = normof{\xi}^\alpha$. Its Lévy triplet is given by:
    \begin{itemize}
        \item If $\alpha = 2$, then the triplet is $(0, 2I_d, 0)$, in which case we have $L_t^2 = \sqrt{2}B_t$.
        \item If $\alpha \in (0, 2)$, the triplet is $(0, 0, \nu_\alpha)$, with:
        \begin{align}
            \label{eq:stable_levy_measure_c_constant}
            d\nu_\alpha(z) := C_{\alpha,d} \frac{dz}{\normof{z}^{d+\alpha}}, \quad C_{\alpha, d} := \alpha 2^{\alpha-1}\pi^{-d/2} \frac{\Gamma\left(  \frac{\alpha + d}{2}\right)}{\Gamma \left( 1 - \frac{\alpha}{2} \right)}.
        \end{align}
    \end{itemize}
\end{definition}

\subsubsection{Generator of the semigroup and fractional Laplacian}

In this subsection, we introduce the notion of fractional Laplacian. This is related to the study of Lévy processes through the notion of "infinitesimal generator of the semigroup", which we first define.

Given a temporally homogeneous Markov process $(X_t)_{t\geq 0}$ we define its \emph{semigroup} $(P_t)_{t\geq 0}$ \citep{xiao_random_2004, schilling_introduction_2016-1}, as the following operators, defined for bounded measurable function $f$:
\begin{align*}
    P_t f(x) = \mathds{E}^x [f(X_t)],
\end{align*} 
where $\mathds{E}^x$ denotes the initialization of the process at $x$ (\ie conditionally on $X_0 = x$). 

If $X$ is a Lévy process, or more generally a Feller process, see \citep{schilling_introduction_2016-1}, such a semigroup is characterized by its infinitesimal generator.

\begin{definition}
    \label{def:infinitesimal_generator}
    Let $\mathcal{C}^0_\infty(\Rd)$ be the space of continuous functions vanishing to zero at infinity. As soon as it exists, we define the generator of the semigroup $(P_t)_t$ as the ensuing limit:
    \begin{align*}
        Af (x) := \lim_{t \to 0} \frac{P_t f - f}{t},
    \end{align*}
    where the limit is understood in the uniform norm on $\mathcal{C}^0_\infty(\Rd)$. The domain of the generator, denoted $\mathcal{D}(A) \subset \mathcal{C}^0_\infty(\Rd)$, is the set of functions for which the above limit is defined\footnote{When the context allows it, the generator may be naturally extended to other spaces.}.
\end{definition}

It is known that the generator of $L_t^2 = \sqrt{2}B_t$ is $A \phi = \Delta \phi$, where $\Delta$ denotes the Laplacian. Following \citep{bottcher_levy_2013, duan_introduction_2015, umarov_beyond_2018}, we can express the generator of $(\levy)$, for $\alpha \in (0,2)$, on the appropriated domain, which at least contains $\mathcal{C}_c^\infty(\Rd)$\footnote{$\mathcal{C}_c^\infty(\Rd)$ denotes the set of infinitely many times differentiable functions with compact support.},
\begin{align}
    \label{eq:generator_stable_levy_process}
    A \phi(x) = C_{\alpha,d}\int_{\Rd} \left( \phi(x+z) - \phi(x) - \nabla \phi(x) \cdot z \chi(\normof{z}) \right) d\nu_\alpha(z),
\end{align}
with the same notations as in \cref{def:stable_levy_processes}. Following \citep{lischke_what_2019, umarov_beyond_2018}, \cref{eq:generator_stable_levy_process} is one of the possible equivalent definitions of the fractional Laplacian. Note that the term fractional Laplacian is actually an abuse of notations, the correct terminology would be the negative fractional negative Laplacian, as we can see in the following definition:
\begin{align}
    \label{eq:fractional_laplacian_first_def}
    - \fraclap \phi(x) := C_{\alpha,d}\int_{\Rd} \left( \phi(x+z) - \phi(x) - \nabla \phi(x) \cdot z \chi(\normof{z}) \right) d\nu_\alpha(z),
\end{align}

\begin{remark}[Equivalent definitions of $\fraclap$]
    There are several equivalent definitions of the fractional Laplacian, we refer the reader to \citep{lischke_what_2019, teymurazyan_fractional_2023} for all details. We only mention two that are commonly used:
    \begin{itemize}
        \item Principal value integral this is the definition used by \citet{raj_algorithmic_2023}, even though their sign convention is different:
        \begin{align*}
            \fraclap \phi(x) = C_{\alpha,d} \lim_{\epsilon \to 0} \int_{\Rd \backslash B_\epsilon(0)} \frac{u(x) - u(x+z)}{\normof{z}^{d+\alpha} } dz.
        \end{align*}
        \item Fourier transform representation: $\mathcal{F}(\fraclap \phi)(\xi) = \normof{\xi}^\alpha \mathcal{F}\phi(\xi)$.
    \end{itemize}
\end{remark}

\subsubsection{Lévy-driven diffusions}

In this subsection, we consider a function $V: \Rd \longrightarrow \R$, which we call potential function. We consider the following stochastic differential equation (SDE):
\begin{align}
    \label{eq:sde_technical_background}
    dX_t = -\nabla V(X_t)dt + \sigma_1 d\levy + \sigma_2 \sqrt{2} dB_t,
\end{align}
with $\alpha \in (0,2)$. This equation admits a strong solution (in the Itô sense), as soon as $V$ is smooth, \ie it satisfies \cref{ass:smooth}, see \citep{schilling_symbol_2010}. 

Let us denote by $\rho = \rho(t,x)$ the probability density of the process $X$. Under certain regularity conditions, this function $\rho$ is known to satisfy the following Fokker-Planck equation, at least in a weak sense (\ie in the sense of distributions):
\begin{align}
    \label{eq:fokker_planck_appendix}
    \partial_t \rho = \nabla \cdot (\rho \nabla V) + \op{\rho},
\end{align}
with the operator $\rho$ being defined using the self adjoint operator of the driving process of \cref{eq:sde_technical_background} as:
\begin{align}
    \label{eq:I_operator}
    \op{\rho} = \sigma_2^2 \Delta \rho + C_{\alpha, d} \sigma_1^\alpha \int \left( \rho(x+z) - \rho(x) - \nabla \rho(x) \cdot z \xi(\normof{z}) \right) d\nu_\alpha(z),
\end{align}
\cref{eq:fokker_planck_appendix} is exactly \cref{eq:empirical_lfp}, with $V = F_S$.

We will use the self-adjointness of such an operator, recalled in the following lemma, of which the reader may find more precise formulations in \citep{gentil_logarithmic_2008,lischke_what_2019,tristani_fractional_2013}.

\begin{lemma}
    As soon as it is well defined, the operator $I$, defined by \cref{eq:I_operator} is self-adjoint, \ie for $\varphi$ and $\psi$, regular enough, we have:
    \begin{align*}
        \int \varphi \op{\psi} dx = \int \psi\op{\varphi} dx. 
    \end{align*}
\end{lemma}

We now discuss the validity of this equation in the following two remarks. Those remarks may be skipped without hurting the general understanding of the paper. They are meant to explain what needs to be assumed to be as rigorous as possible in our treatment of the fractional Fokker-Planck equation.

\begin{remark}[Justification of the equation]
    Using the main result of \citet{kuhn_solutions_2018}, it can be argued that, under \cref{ass:smooth}, $X$ is a Feller process. Its generator can be expressed as \citep{schilling_symbol_2010, duan_introduction_2015, umarov_beyond_2018}:
    \begin{align*}
        A \phi(x) = -\nabla V(x) \cdot \nabla \phi(x) + \sigma_2^2 \Delta \phi(x) - \sigma_1^\alpha \fraclap \phi(x).
    \end{align*}
    If we denote by $(P_t)_t$ the semi-group associated with $X$, then we have the Kolmogorov backward equation, for $f$ in the generator $\mathcal{D}(A)$ of $A$:
    \begin{align*}
        \frac{d}{dt} P_t f = A P_t f = P_t A f,
    \end{align*}
    taking the $L^2$-adjoint of this equation leads to $\partial_t \rho = A^\star \rho$ (at least in a weak sense). A direct computation of the adjoint $A^\star$, using properties of the fractional Laplacian, leads to \cref{eq:fokker_planck_appendix}. 
\end{remark}

\begin{remark}[Validity of \cref{eq:fokker_planck_appendix}]
    Let us quickly discuss the domain of validity of the Fokker-Planck equations considered in this paper. We are in particular interested in the (local) regularity of the potential solutions, as they are necessary to give meaning to several computations made in our proofs. Let us first mention that in the simpler case of a quadratic potential, the regularity easily comes from the analytic solution of such equations \citep{lafleche_fractional_2020}. Moreover, equations such as \cref{eq:fokker_planck_appendix} are known to have regularization properties and have the ability to generate smooth solutions.
    
    More precisely, as it is mentioned in \citep{umarov_beyond_2018}, under regularity assumptions, it is known that such an equation is satisfied in a weak sense, namely in the sense of distributions \citep{halperin_introduction_1952}. Therefore, one may ask whether smooth solutions do actually exist. Smoothness in the $x$ variable has been proven for the Ornstein-Uhlenbeck drift by \citep{xie_regularity_2015}. Space-time regularity has been proven in \citep{imbert_non-local_2005} in the case of a bounded force field $\nabla F_S$. An example of space regularity was achieved in the case of a force field with bounded derivatives of positive order, in \citep{wang_existence_2017}. Finally, let us mention the work of \citep{lafleche_fractional_2020}, which provides further regularity conditions.
\end{remark}

\subsection{Logarithmic Sobolev inequalities and $\Phi$-entropy inequalities}
\label{sec:log_sobolev_inequalities_phi_entropy}

Logarithmic Sobolev inequalities (LSI) and Poincaré inequalities are central tools in probability theory. LSIs were historically introduced by \citet{gross_logarithmic_1975} and have famously been applied to the analysis of Markov processes \citep{bakry_analysis_2014}, the study of evolution equations \citep{markowich_trend_2004}, and have been connected to optimal transport and geometry \citep{villani_optimal_2009}. For a short introduction, the reader may consult the tutorials \citep{chafai_entropies_2004, chafai_logarithmic_2017}. In this section, we quickly introduce a few of these results, with an emphasis on a generalization to infinitely divisible distributions, playing an important role in our study.

\subsubsection{The classical inequalities}

We first quickly recall the classical Poincaré inequality and LSI. They hold for the standard Gaussian measure\footnote{Using very simple arguments, it holds for every Gaussian, up to changes in the constants.}, defined by $\gamma_d = \mathcal{N}(0, I_d)$. We first give Poincaré inequality:
\begin{theorem}
    \label{thm:classical_poincaré inequality}
    Let $f$ be a function such that $\nabla f \in L^2(\gamma_d)$, we have:
    \begin{align*}
        \int_\Rd f^2 d\gamma_d - \left( \int_\Rd f d\gamma_d \right)^2 \leq \int_\Rd \normof{\nabla f}^2 d\gamma_d.
    \end{align*}
\end{theorem}

\begin{example}
    If we define the convex function $\Phi(x) = \Phi_2(x) := \frac{x^2}{2}$, as in \cref{sec:toward_time_uniform}, then the left hand side of the inequality of \cref{thm:classical_poincaré inequality} can be understood as $2\entphi[\gamma_d]{f}$. 
\end{example}

The next theorem is the classical LSI for the Gaussian measure $\gamma_d$:

\begin{theorem}
    \label{thm:classical_LSI}
    Let $f \in \mathcal{C}^1(\Rd,R)$ be non-negative and integrable, then we have, for $\Phi(x) = \Phi_{\log}(x) = x\log(x)$:
    \begin{align*}
        \entphi[\gamma_d]{f} \leq \frac{1}{2} \int_\Rd \frac{\normof{\nabla f}^2} f d\gamma_d.
    \end{align*}
\end{theorem}

\begin{example}
    With the convex function $\Phi$ used in the above theorem, the right-hand side of the inequality is, up to the constant, the Fisher information, which we call $\Phi$-information in \cref{lemma:big_decomposition}.
\end{example}

\subsubsection{Generalization to infinitely divisible distributions}

Part of our analysis is based on a generalization of those inequalities to infinitely divisible distributions. Let us recall that those distributions have been defined in \cref{sec:background_levy_markov} and can be equivalently seen as the distribution of Lévy processes. This is how we connect those inequalities to our theory.

Let us first recall the definition of $\Phi$-entropies, \ie \cref{def:phi_entropies}.

\defPhiEntropy*

Note that, by Jensen's inequality, such a term is always non-negative.

The following theorem was proved by \citet{wu_new_2000} and \citet{chafai_entropies_2004}, it generalizes Poincaré and logarithmic Sobolev inequalities to infinitely divisible distributions.

\begin{theorem}[Generalized LSI]
    \label{thm:generalized_lsi}
    Let $\mu$ be an infinitely divisible law on $\Rd$, with associated triplet denoted $(b, Q, \nu)$, in the sense of \cref{thm:levy_khintchine}. We further assume that $\Phi : \R_+ \longrightarrow \R$ is a convex function that satisfies the following set of assumptions:
    \begin{align}
        \label{eq:phi-asmpt-for-heavy-tailed-lsi}
        \left\{
            \begin{aligned}
            &(u,v) \longmapsto \Phi(u + v) - \Phi(v) - u\Phi'(v) \text{  is convex on its domain of definition} \\
            &(u,x) \longmapsto \Phi''(u) x \cdot Qx \text{  is convex on its domain of definition}. \\
            \end{aligned}
        \right.
    \end{align}   
    Then, for every smooth enough function $v$, we have:
    \begin{align*}
        \entphi{v} \leq \frac{1}{2} \int \Phi''(v) \nabla v \cdot Q \nabla v ~d\mu + \iint \bregman{v(x+z)}{v(x)} d\nu(z) d\mu(x). 
    \end{align*}
\end{theorem}

Note that, when $\mu = \gamma_d$, then its Lévy triplet is $(0, I_d, 0)$, so that \cref{thm:generalized_lsi} implies \cref{thm:classical_poincaré inequality,thm:classical_LSI}.
The assumptions given by \eqref{eq:phi-asmpt-for-heavy-tailed-lsi} where in particular identified by \cite{chafai_entropies_2004} and \cite{gentil_logarithmic_2008}, they contain in particular the functions $x \mapsto x\log(x)$ (so that we recover the usual LSI in the Gaussian case), and the functions $x \mapsto x^p$, with $1<p \leq 2$, which is the case we will consider in this study.

\subsection{Some properties of the Gamma function}
\label{sec:gamma_function}

The Euler gamma function is classically defines by:
\begin{align*}
    \forall x > 0, ~\Gamma(x) = \int_0^\infty t^{x-1} e^{-t} dt.
\end{align*}
It has a natural extension to $\mathds{C} \backslash (-\mathds{N})$. In this subsection, we give a few properties of this function, which will be useful to prove the results of \cref{sec:qualitative_analysis} in \cref{sec:proofs-qualitative_analysis}. 

\textbf{One particular value:}
\begin{align*}
    \Gamma \left( \frac{1}{2} \right) = \sqrt{\pi}.
\end{align*}

\begin{lemma}[Euler's reflection formula]
    \label{lemma:euler_reflection}
    For all $z \in \mathds{R} \backslash \mathds{Z}$, we have:
    \begin{align*}
        \Gamma (1 - z) \Gamma (z) = \frac{\pi}{\sin (\pi z)}.
    \end{align*}
\end{lemma}

\begin{lemma}[Stirling's formulas]
    \label{lemma:stirling}
    We have have the two following asymptotic formulas:
    \begin{align*}
        \Gamma(x + 1) \underset{x \to \infty}{\sim} \sqrt{2 \pi x} \left( \frac{x}{e} \right)^x,
    \end{align*}
    and, for all\footnote{This formula is actually true for all $\alpha \in \mathds{C}$, considering the extension of the $\Gamma$ function into the complex plane.} $\alpha \in \R$:
    \begin{align*}
        \Gamma(x + \alpha)  \underset{x \to \infty}{\sim} \Gamma(x) x^\alpha.
    \end{align*}
\end{lemma}

\section{A PAC-Bayesian bound for sub-gaussian losses}
\label{sec:proofs-subgaussian-pb}

In this section, we give proof of a PAC-Bayesian bound, which applies to sub-gaussian losses. The notations for $\zcal$, $S$,  $\mu_z$, $\ell$, $L$ and $\el$ are the same as in \cref{sec:intro}, \ie:
\begin{align*}
    L(w) := \Eof[z\sim\mu_z]{\ell(w,z)}, \quad, \el(w) := \frac{1}{n} \sum_{i=1}^n \ell(w,z_i),
\end{align*}
with $S = (z_1,\dots,z_n) \in \zcal^n$. We consider a prior distribution $\pi$ as well as a family of posterior distributions, as it has been introduced in \cref{sec:it_terms_and_pb_bounds}.

By \cref{thm:PAC_Bayesian_bounded_loss}, which has been proven by \citet{mcallester_pac-bayesian_2003, maurer_note_2004}, we know that with probability at least $1 -  \zeta$ under $S \sim \datadist$, we have:
\begin{align*}
        \Eof[\rho_S]{L(w) - \el(w)} \leq \sqrt{\frac{\klb{\rho_S}{\pi} + \log \frac{2 \sqrt{n}}{\zeta}}{2n}}.
    \end{align*}
Unfortunately, to the best of our knowledge, no such bounds (\ie without the variable $\lambda$ like in \cref{thm:pac_bayesian_kl,thm:disintegrated_pac_bayes__bound}), exist for sub-gaussian losses, which is the assumption we make in our paper. As an additional theoretical contribution, we present the following PAC-Bayesian bound for sub-gaussian losses. We believe it may be useful for other works.

\begin{theorem}
    \label{thm:pb_for_subgaussian}
    We assume that $f$ is $s^2$-subgaussian, in the sense of \cref{ass:subgaussian}. Then, with probability at least $1 -  \zeta$ under $S \sim \datadist$, we have:
    \begin{align*}
        \Eof[\rho_S]{L(w) - \el(w)} \leq 2s\sqrt{\frac{\klb{\rho_S}{\pi} + \log \frac{3}{\zeta}}{n}}.
    \end{align*}
\end{theorem}

\begin{proof}
    The proof follows very closely that of \citep[Proposition $2.5.2$]{vershynin_high-dimensional_2020}, which we adapt to our particular case to exhibit the exact absolute constants. 

    We start by fixing $0 <a < 1/(2s^2)$ and applying \cref{thm:pac_bayesian_kl} to the function $\phi(w,S) := an\left(L(w) - \el(w) \right)^2$, which gives that, with probability at least $1 -  \zeta$ under $S \sim \datadist$:
    \begin{align*}
       an  \Eof[\rho_S]{\left(L(w) - \el(w) \right)^2} \leq \klb{\rho_S}{\pi} + \log(1/\zeta) + \log \E_\pi \Eof[S]{e^{an\left(L(w) - \el(w) \right)^2}}.
    \end{align*}
    The above holds as soon as the last term is defined and finite, this will be an outcome of our computations. Let us denote $\Delta := |L(w) - \el(w) |$ and estimate the expected exponential term, by Tonelli's theorem:
    \begin{align*}
        \Eof[S]{e^{an \Delta^2}} = 1 + \sum_{k=1}^\infty \frac{a^k n^k}{k!} \Eof[S]{\Delta^{2k}}.
    \end{align*}
    Let us fix some $p > 1$, we note that we have:
    \begin{align*}
        \Eof[S]{\Delta^p} = \int_0^\infty \Pof{\Delta^p \geq \epsilon} d\epsilon 
    \end{align*}
    Now, by Hoeffding's inequality and several changes of variables, we have:
    \begin{align*}
        \Eof[S]{\Delta^p} &\leq \int_0^\infty \Pof{\Delta \geq \epsilon^{1/p}} d\epsilon  \\
        &\leq 2\int_0^\infty e^{-\frac{n \epsilon^{2/p}}{2s^2}} d\epsilon \\
        &= p \int_0^\infty e^{-\frac{n v}{2s^2}} v^{\frac{p}{2}-1}dv \\
        &= p \left( \frac{2 s^2}{n} \right)^{p/2}\int_0^\infty e^{-t} t^{\frac{p}{2}-1}dt \\
        &= p \left( \frac{2 s^2}{n} \right)^{p/2} \Gamma \left( \frac{p}{2} \right) \\
        &= 2 \left( \frac{2 s^2}{n} \right)^{p/2} \Gamma \left( 1 +\frac{p}{2} \right),
    \end{align*}
    were the last two inequalities follow from the definition and properties of the $\Gamma$ function, as stated in \cref{sec:gamma_function}.
    If we plug this into our previous computations, we find that:
     \begin{align*}
        \Eof[S]{e^{an \Delta^2}} &\leq 1 + 2\sum_{k=1}^\infty \left( \frac{2 s^2}{n} \right)^k \frac{a^k n^k}{k!}k! \\
        &\leq  1 + 2\sum_{k=1}^\infty (2s^2 a)^k \\
        &= 1 + 2 \frac{2 s^2 a}{1 - 2 s^2 a},
    \end{align*}
    where the last line holds because we assumed $2as^2 < 1$. We now make the following particular choice $a := 1/(4s^2)$ and we get:
    \begin{align*}
        \Eof[S]{e^{an \Delta^2}} \leq 3.
    \end{align*}
    Now, by Jensen's inequality and Fubini's theorem, this implies that, with probability at least $1 - \zeta$ over $S \sim \datadist$:
    \begin{align*}
        \frac{n}{4s^2} \Eof[\rho_S]{L(w) - \el(w)}^2 \leq \klb{\rho_S}{\pi} + \log(1/\zeta) + \log(3),
    \end{align*}
    which immediately implies the desired result.
\end{proof}

\begin{remark}
    If we assume, as in \cref{thm:PAC_Bayesian_bounded_loss}, that the function $f$ is bounded in $[0,1]$, then, by Hoeffding's lemma, $f$ is $s^2$-subgaussian with $s = 1/2$. Therefore, our bound implies:
    \begin{align*}
        \Eof[\rho_S]{L(w) - \el(w)} \leq \sqrt{\frac{\klb{\rho_S}{\pi} + \log \frac{3}{\zeta}}{n}},
    \end{align*}
    which has a slightly less good constant than \cref{thm:PAC_Bayesian_bounded_loss}, but we improve the term $\log(2\sqrt{n})$ into $\log(3)$.
\end{remark}

By combining the previous computations with the disintegrated bound of \cref{thm:disintegrated_pac_bayes__bound}, we immediately obtain:

\begin{theorem}
    \label{thm:pb_for_subgaussian_disintegrated}
    We assume that $f$ is $s^2$-subgaussian, in the sense of \cref{ass:subgaussian}. Then, with probability at least $1 -  \zeta$ under $S \sim \datadist$ and $w \sim \rho_S$, we have:
    \begin{align*}
        L(w) - \el(w) \leq 2s\sqrt{\frac{\renyi[2]{\rho_S}{\pi} + \log \frac{24}{\zeta^3}}{n}}.
    \end{align*}
\end{theorem}

\section{Proofs of the main theorems and additional results}
\label{sec:proof_of_main_results}

In all the proofs, we use the following notation:
\begin{align*}
    \op{u} = \sigma_2^2 \Delta u + C_{\alpha, d} \sigma_1^\alpha \int \left( u(x+z) - u(x) - \nabla u(x) \cdot z \xi(\normof{z}) \right) d\nu_\alpha(z),
\end{align*}
Note that this is the generator of the process driving \Cref{eq:multifractal_dynamics}, \eg  $\sigma_2 B_t + \sigma_1 \levy$. We will use the following notations for the drift terms:
\begin{align*}
    V_S(w) = \ef(w) + \frac{\eta}{2} \normof{w}^2, \quad \quad V(w) =  \frac{\eta}{2} \normof{w}^2.
\end{align*}

We use the notations $u_t^S, \uinftybar, \rho_t^S, \pi$ and $v_t^S$ in the same way as they have been introduced in \cref{sec:background_notations}. Moreover, we remind the reader that we often denote $v$ instead of $v_t^S$, and $u$ instead of $u^S_t$, the dependence in the time $t$ and the data $S$ being implicit.

\subsection{The main decomposition}
\label{sec:proofs-main_decomposition}

Before proving our main results, we define the Bregman divergence, associated with a convex function $\Phi$. We will use it repeatedly in our proofs and statements. This notion also justifies that we call the term $B_\Phi^\alpha$, appearing in \cref{lemma:big_decomposition}, the "Bregman" integral.

\begin{definition}
    Given a convex interval $I$ and $\Phi: I \longrightarrow \mathds{R}$ a convex function, we define the Bregman divergence as:
    \begin{align*}
        \bregman{a}{b} := \Phi(a) - \Phi(b) - \Phi'(b) (a - b).
    \end{align*}
\end{definition}

We first prove \cref{lemma:big_decomposition}, which is the main decomposition that we use, in order to derive our main results. This result follows the computations of \citet{gentil_logarithmic_2008}, which are adapted to the comparison of two dynamics.

\thmDecomposition*

Before proving this theorem, let us give the complete expression of the Bregman term $B_\Phi^\alpha (v)$. It is given by the following formula:
\begin{align*}
    B_\Phi^\alpha(v)  = C_{\alpha, d} \iint \bregman{v(x)}{v(x+z)} \uinftybar d\nu_\alpha(z) dx.
\end{align*}

\begin{proof}
    In all the proof, we omit the time dependence of $u$ and $v$, as already mentioned above. Let us first recall the Fokker-Planck equation satisfied by $u$. Rewritten with the notations of this section, this equation is:
    \begin{align*}
        \partial_t u = \op{u} + \nabla \cdot (u \nabla V_S).
    \end{align*}
    Similarly, the stationary Fokker-Planck equation satisfied by $\uinftybar$ is:
    \begin{align*}
        0 = \op{\uinftybar} + \nabla \cdot (\uinftybar \nabla V).
    \end{align*}
    We finally remind the reader that $v := u/\uinftybar$. 
    Using the definition of $v$ and the FPE of $u$ and $\uinftybar$, we have:
    \begin{align*}
        \partial_t v &= \frac{1}{\uinftybar} \left( \op{\uinftybar v} + \nabla \cdot (\uinftybar v \nabla V_S) \right) \\
        &= \frac{1}{\uinftybar} \left( \op{\uinftybar v} + \nabla \cdot (\uinftybar v \nabla V) \right) + \frac{1}{\uinftybar} \nabla \cdot (\uinftybar v \nabla (V_S - V))  
    \end{align*}
    Therefore, the entropy flow is equal to:
    \begin{align*}
         \frac{d}{dt} \entphi{v} &= \frac{d}{dt} \int \Phi(v) \uinftybar \\
         &= \int \Phi'(v) \uinftybar \partial_t v dx \\
            &= \int \Phi'(v)\left( \op{\uinftybar v} + \nabla \cdot (\uinftybar v \nabla V) \right)   dx + \int \Phi'(v)  \nabla \cdot (\uinftybar v \nabla (V_S - V)) dx \\
            &=: A_1 + A_2.
    \end{align*}

    Note that the derivation under the integral is perfectly justified, from our $\Phi$-regularity assumption.
    
    Using the self-adjointness of $I$ (see \cref{sec:background_levy_markov}) and the integration by parts formula, we have: 
    \begin{align*}
        A_1 = \int \op{\Phi'(v)} v \uinftybar dx - \int \Phi''(v) v \uinftybar \nabla v \cdot \nabla V dx.
    \end{align*}
    Now we use the formula $r \Phi''(r) = (r\Phi'(r) - \Phi(r))'$ and get:
    \begin{align*}
        A_1 &= \int \op{\Phi'(v)} v \uinftybar dx + \int (v \Phi'(v) - \Phi(v)) \nabla \cdot \uinftybar \nabla V dx \\
        &= \int \op{\Phi'(v)} v \uinftybar dx - \int (v \Phi'(v) - \Phi(v)) \op{\uinftybar} dx \\
        &= \int \left( \op{\Phi'(v)} v - \op{v \Phi'(v)} + \op{\Phi(v)} \right) \uinftybar dx.
    \end{align*}
    Therefore, we have to compute the quantity $ f(v) := \op{\Phi'(v)} v - \op{v \Phi'(v)} + \op{\Phi(v)} $. Using the definition of $I$, we get:
    \begin{equation*}
        \begin{aligned}
        f(v) = \sigma_2^2  \left( v \Delta \Phi'(v) - \Delta (v \Phi'(v)) + \Delta \Phi(v) \right) & \\
        + C_{\alpha, d} \sigma_1^\alpha \int \big\{ &v(x) \Phi'(v(x+z)) - v(x) \Phi'(v(x)) - v(x) \nabla \Phi'(v(x)) \cdot z \chi(\normof{z}) \\
        & -v(x+z) \Phi'(v(x+z)) + v(x) \Phi'(v(x)) + \nabla(v\Phi'(v))(x)) \cdot z \chi(\normof{z}) \\
        &+ \Phi(v(x+z)) - \Phi(v(x)) - \nabla(\Phi(v))(x) \cdot z \chi(\normof{z}) \big\} d\nu_\alpha(z) .
        \end{aligned}
    \end{equation*}
    which is equal, after easy computations, to:
    \begin{align*}
        f(v) = -\sigma_2^2 \Phi''(v) \normof{\nabla v}^2 - C_{\alpha, d} \sigma_1^\alpha \int \bregman{v(x)}{v(x+z)} d\nu_\alpha(z).
    \end{align*}
    Therefore:
    \begin{align*}
        A_1 = -\sigma_2^2 \int \Phi''(v) \normof{\nabla v}^2 \uinftybar dx - C_{\alpha, d} \sigma_1^\alpha \iint \bregman{v(x)}{v(x+z)} \uinftybar d\nu_\alpha(z) dx.
    \end{align*}
    for the second term, we have:
    \begin{align*}
        A_2 &=  \int \Phi'(v)  \nabla \cdot (\uinftybar v \nabla (V_S - V)) dx \\
        &= -\int \Phi''(v) v \uinftybar \nabla v \cdot \nabla (V_S - V) dx.
    \end{align*}
    The result follows.
\end{proof}

In the sequel, the quantity $\frac{d}{dt} \entphi{v} $ will be called \emph{entropy flow}.

\begin{remark}
    If, in Equations \eqref{eq:empirical_lfp} and \eqref{eq:steady_state}, we where using the multifractal process,
    $\sigma \sqrt{2} dB_t + \sum_{i=1}^N \sigma_i L_t^{\alpha_i}$,
    instead of $\sqrt{2} \sigma_2 B_t + \sigma_1 \levy$, then the result would become:
    $$
    \frac{d}{dt} \entphi{v} = -\sigma^\alpha \int \frac{\normof{\nabla v}^2}{v} \uinftybar dx - \sum_{i=1}^N C_{\alpha_i,d}  \sigma_i^{\alpha_i} \iint \bregman{v(x)}{v(x+z)} \uinftybar d\nu_{\alpha_i}(z) dx - \int \nabla v \cdot \nabla \ef \uinftybar.
    $$
\end{remark}

\subsection{Proof of \cref{cor:gen_brownian_time}}

An easy consequence of \cref{lemma:big_decomposition} is the following generalization bound, namely \cref{cor:gen_brownian_time}, which was first stated in \cref{sec:brownian case}. It holds only when $\sigma_2 > 0$. 

\corKLbrownian*

\begin{proof}
    Thanks to Theorem \ref{thm:pb_for_subgaussian}, we have, for a time $t$ and $\zeta \in (0,1)$, with probability at least $1 - \zeta$ over $S\sim\datadist$, that:
    \begin{align*}
         \Eof[\rho_S^t]{L(w) - \el(w)} \leq 2s \sqrt{\frac{\klb{\rho_S^t}{\pi} + \log(3/\zeta)}{n}},
    \end{align*}
    where the posterior $\rho_S^t$ is the distribution with density $x \mapsto u(t,x)$, described by Equation \eqref{eq:empirical_lfp}, and the prior $\pi$ is the distribution with density $\uinftybar$.
    Now we set $\Phi(u) = u\log(u)$, so that we have $\Phi''(u) = 1/u$ and, if $v = u/\uinftybar$:
    \begin{align*}
        \entphi{v} = \klb{\rho_S^t}{\pi}.
    \end{align*}
    Therefore, by Theorem \ref{lemma:big_decomposition}, we have for all $t>0$ that:
    \begin{align*}
       \frac{d}{dt} \klb{\rho_S^t}{\pi} =  -\sigma_2^2 \int \frac{\normof{\nabla v}^2}{v} \uinftybar dx -  C_{\alpha,d}  \sigma_1^\alpha \iint \bregman{v(x)}{v(x+z)} \uinftybar d\nu_\alpha(z) dx + \int \nabla v \cdot \nabla(V - V_S) \uinftybar.
    \end{align*}
     Using the non-negativity of the Bregman divergence of the convex function $\Phi$, along with Cauchy-Schwarz and Young's inequalities, we have, for any constant $C > 0$:
     \begin{align*}
         \frac{d}{dt} \klb{\rho_S^t}{\pi} \leq  -\sigma_2^2 \int \frac{\normof{\nabla v}^2}{v} \uinftybar dx + \frac{C}{2} \int \frac{\normof{\nabla v}^2}{v} \uinftybar dx + \frac{1}{2C} \int \normof{\nabla V -  \nabla V_S}^2 u
     \end{align*}
    Note that Assumptions \ref{ass:phi_regularity} and \ref{ass:phi_risk_integrability} ensure that the above integrals are finite.
     
     By choosing $C = 2\sigma_2^2$ and using the definition of $u$ and $\rho_S^t$, we then have, with $\Lambda := \limsup_{t\to 0} \klb{\rho_S^t}{\pi}$:
     \begin{align*}
         \klb{\rho_S^t}{\pi} \leq \Lambda + \frac{1}{4\sigma^2}\int_0^t \Eof[\rho_S^u]{\normof{\nabla V(w) -  \nabla V_S(w)}^2} du.
     \end{align*}
     Finally, the results immediately follows by applying \cref{thm:pb_for_subgaussian}, as described at the beginning of the present proof.
\end{proof}

\subsection{Bounds on the Bregman integral - introduction of the functional $J_{\Phi,v}$}
\label{sec:bregman_integral_bounds}

The previous computations are interesting, but we can note that the tail index $\alpha$, as well as a scale $\sigma_1$ of the heavy-tailed noise, play no role in the derived bound. In particular, this approach cannot help us to derive bounds that hold in the case $\sigma_2 = 0$, \ie a pure heavy-tailed dynamics. This section is meant to introduce the main tools for a step toward this direction. In particular, it is in this section that we justify the introduction of the functional $J_{\Phi,v}$ of \cref{sec:pure_levy_case}. In this section, we fix $\Phi(x) = \Phi_{\log}(x) = x \log(x)$, while some computatios are valid in a more general setting, see \ref{sec:proofs-poincare-inequality}.

We first remark that, if we want the integral term of \cref{cor:gen_brownian_time} to appear in the bound (in order to have a strongly interpretable bound), namely,
\begin{align*}
    \int_0^T \Eof[U]{\normof{\nabla\ef(W_t^S)}^2 } dt ,
\end{align*}
then we can still use Young's inequality on the last term of \cref{lemma:big_decomposition}, and write that, for any $C > 0$:

\begin{align*}
    \int \nabla v \cdot \nabla(V - V_S) \uinftybar \leq \frac{C}{2} \int \frac{\normof{\nabla v}^2}{v} \uinftybar + \frac{1}{2C} \int \normof{\nabla(V - V_S)}^2 u.
\end{align*}

Therefore, in order to improve our bounds, we need to understand how the Bregman divergence integral can be used to compensate for the Fisher information given by $\int \frac{\normof{\nabla v}^2}{v} \uinftybar $. The following lemma is a first step toward that direction.

\begin{lemma}[Spherical representation of the Bregman divergence integral]
    \label{lemma:spherical_representation}
    With the same notations and assumptions as in Theorem \ref{lemma:big_decomposition} and with $\Phi(x) = x \log(x)$, we have, for all $R \in (0, +\infty]$:
    \begin{align*}
         \iint \bregman{v(x)}{v(x+z)} \uinftybar d\nu_\alpha(z) dx \geq  \int_0^{R}    \int_0^r  \int_s^r \int \int_\Sd \frac{\theta \cdot \nabla v(x + s\theta)  \theta \cdot \nabla v(x + u\theta)}{v(x + u\theta)}  \uinftybar(x)  d\theta dx du ds \frac{dr}{r^{\alpha+1}}.
    \end{align*}
    If $R = +\infty$, then this is even an equality.
\end{lemma}

The proof of the following lemma is based on a second-order approximation of the Bregman integral appearing in the Bregman integral term above. Such computations were first hinted by \citep{chafai_entropies_2004,Chafai_2006}, in particular through the notion of ``$\Phi$-calculus''.

\begin{proof}
    We use a spherical change of coordinates in $\Rd$, along with the fact that $d \nu_\alpha(z) = \normof{z}^{-d - \alpha} dz$.
    \begin{align*}
        \iint \bregman{v(x)}{v(x+z)} \uinftybar d\nu_\alpha(z) dx =  \int \int_\Sd \int_0^{\infty} \bregman{v(x)}{v(x+r\theta)} \uinftybar(x) \frac{dr}{r^{\alpha+1}} d\theta~ dx.
    \end{align*}
    Let us fix some $R>0$. By Tonelli's theorem and the positivity of the Bregman divergence, we can write that:
    \begin{align*}
        \iint \bregman{v(x)}{v(x+z)} \uinftybar d\nu_\alpha(z) dx \geq \int_0^R \int_\Sd \int_\Rd \bregman{v(x)}{v(x+r\theta)} \uinftybar(x) dx d\theta \frac{dr}{r^{\alpha+1}} ~ .
    \end{align*} 
    Let us fix $x$, $r$ and $\theta$, by the $\Phi$-regularity assumption, we have that $v$ is differentiable, therefore, we have:
    \begin{align*}
        \bregman{v(x)}{v(x+r\theta)} &= \Phi(v(x)) - \Phi(v(x + r\theta)) - \Phi'(v(x + r\theta)) (v(x) - v(x+r\theta)) \\
        &= - \int_0^r \Phi'(v(x + s\theta)) \theta \cdot \nabla v(x + s\theta) ds + \Phi'(v(x + r\theta))  \int_0^r   \theta \cdot \nabla v(x + s\theta) ds \\
        &=  \int_0^r   \theta \cdot \nabla v(x + s\theta) \left(  \Phi'(v(x + r\theta))  -  \Phi'(v(x + s\theta))  \right) ds \\
        &= \int_0^r  \int_s^r \theta \cdot \nabla v(x + s\theta) \Phi''(v(x + u\theta)) \theta \cdot \nabla v(x + u\theta) du~ ds 
    \end{align*}
    Now we use the fact that $\Phi(x) = x \log(x)$, hence $\Phi''(x) = 1/x$, which gives:
    \begin{align*}
        \bregman{v(x)}{v(x+r\theta)}  =  \int_0^r  \int_s^r \frac{\theta \cdot \nabla v(x + s\theta)  \theta \cdot \nabla v(x + u\theta)}{v(x + u\theta)} du~ ds ,
    \end{align*}
    Finally, thanks to the $\Phi$-regularity assumption, the function:
    \begin{align*}
        x \longmapsto \frac{\theta \cdot \nabla v(x + s\theta)  \theta \cdot \nabla v(x + u\theta)}{v(x + u\theta)} \uinftybar(x),
    \end{align*}
    is integrable for each $(s,u,\theta)$, moreover, by the dominated convergence theorem and the $\Phi$-regularity assumption, the function:
    \begin{align*}
    (s,u,\theta) \longmapsto \int_\Rd \left| \frac{\theta \cdot \nabla v(x + s\theta)  \theta \cdot \nabla v(x + u\theta)}{v(x + u\theta)}  \right| \uinftybar (x) dx,
    \end{align*}
    is continuous. Now, as we integrate those variables over compact sets, we have:
    \begin{align*}
       \int_\Sd \int_0^r \int_s^r \int_\Rd \left| \frac{\theta \cdot \nabla v(x + s\theta)  \theta \cdot \nabla v(x + u\theta)}{v(x + u\theta)}  \right| \uinftybar (x) dx du ds d\theta < +\infty.
    \end{align*}
    The result follows from the application of Fubini's theorem.
\end{proof}

The term appearing in the above lemma resembles a lot the Fisher information, 
\begin{align}
    \label{eq:fisher_information}
    \mathbf{J}_v :=  \int \frac{\normof{\nabla v}^2}{v} \uinftybar.
\end{align}
Note that, with $\Phi(x) = x \log(x)$, which is the case in all this section, we have, as introduced in \cref{sec:background_notations}:
\begin{align*}
    I_\Phi(v) = \mathbf{J}_v.
\end{align*}

This justifies the introduction of the following notion of "spherical information". 

\begin{definition}[Spherical Fisher information]
    \label{def:spherical_fisher_info}
    For $r\geq 0$, we introduce the following spherical Fisher information, for $r> 0$:
    \begin{align*}
        J_v(r) := \frac{2d}{r^2\sigma_{d-1}}\int_0^r \int_s^r \int_{\Rd} \int_\Sd \frac{\theta \cdot \nabla v(x + s\theta)  \theta \cdot \nabla v(x + u\theta)}{v(x + u\theta)} \uinftybar(x) d\theta dx du ds,
     \end{align*}
     where $\sigma_{d-1}$ is the surface area of the $(d-1)$-dimensional hyper-sphere $\Sd \subset \Rd$, it is given by:
     \begin{align}
         \label{eq:sphere_area}
         \sigma_{d-1} = \frac{2\pi^{d/2}}{\Gamma(d/2)}.
     \end{align}
     $J_v$ is the function $J_{\Phi,v}$ introduced in the main part of the paper, for the particular case $\Phi(x) = x \log(x)$.
\end{definition}

The following lemma justifies the normalization used in this definition.

We also define:
\begin{align*}
        g(s,u) = \int_{\Rd} \int_\Sd \frac{\theta \cdot \nabla v(x + s\theta)  \theta \cdot \nabla v(x + u\theta)}{v(x + u\theta)} \uinftybar(x) d\theta dx
\end{align*}

We will denote by $\partial g / \partial s$ (resp. $\partial g / \partial u$) the partial derivative of $g$ with respect to its first (resp. second) variable.

\begin{lemma}
    \label{lemma:g_c1}
    Under Assumption \ref{ass:phi_regularity}, $g$ is differentiable in the second variable and both $g$ and $\frac{\partial g}{\partial u}$ are jointly continuous in $(s,u) \in \R_+^2$.
\end{lemma}

\begin{proof}
    This follows from the $\Phi$-regularity condition. More precisely, let us denote:
    \begin{align*}
        h(s,u;x,\theta) = \frac{\theta \cdot \nabla v(x + s\theta)  \theta \cdot \nabla v(x + u\theta)}{v(x + u\theta)} \uinftybar(x).
    \end{align*}
    By the $\Phi$-regularity condition, we know that this function is continuous in $(s,u)$. Moreover, let $V \subset \R_+^2$ be a bounded open set of $\R_+^2$.  By the $\Phi$-regularity condition, the mappings $(s,u) \longmapsto |h(s,u;x,\theta)|$ are uniformly dominated on $V$ by a function $\chi_V \in L^1(\Sd \times \Rd)$. Therefore, by the dominated convergence theorem, we have the joint continuity of $g$.

    Now, again by the $\Phi$-regularity condition, $h$ is differentiable in the second variable and, $u$, and the partial derivatives $\frac{\partial g}{\partial u}$ are uniformly (\wrt $(s,u) \in V$) dominated by a function in $ L^1(\Sd \times \Rd)$. Therefore, we can differentiate under the integral sign and get that $g$ is differentiable in its second variable, $u$.

    We get the joint continuity of $\frac{\partial g}{\partial u}$ by a very similar argument than before, it is again a consequence of our $\Phi$-regularity condition.
\end{proof}

\begin{remark}
    The $\Phi$-regularity assumption, \cref{ass:phi_regularity}, has been designed, in particular, to get enough regularity of this function $g$, which is a central tool in our proofs. This is the central reason why we need that much regularity of the functions $v_t$.
\end{remark}

\begin{lemma}
    \label{lemma:spherical_information_at_zero}
    The function $J_v(r)$ can be continuously extended to $[0,+\infty)$, by setting:
    \begin{align*}
        J_v(0) = \mathbf{J}_v = \int \frac{\normof{\nabla v}^2}{v} \uinftybar,
    \end{align*}
    the obtained function is still denoted $J_v$.
\end{lemma}

\begin{proof}
    We first compute:
    \begin{align*}
        g(0,0) = \int_{\Rd} \int_\Sd \frac{\left(\theta \cdot \nabla v(x) \right)^2}{v(x)} \uinftybar(x) d\theta dx.
    \end{align*}
    As the distribution that is considered on the sphere is the uniform distribution, the invariance by rotation, along with Tonelli's theorem, implies that:
    \begin{align*}
        g(0,0) =\int \frac{\normof{\nabla v}^2}{v} \uinftybar dx \int_\Sd  \theta_1^2 d\theta.
    \end{align*}
    Now we easily compute:
    \begin{align*}
        \int_\Sd  \theta_1^2 d\theta = \frac{1}{d} \int_\Sd  \sum_{i=1}^d \theta_i^2 d\theta =  \frac{1}{d} \int_\Sd d\theta = \frac{\sigma_{d-1}}{d},
    \end{align*}
    so that:
    \begin{align*}
        g(0,0) = \frac{\sigma_{d-1}}{d}\int \frac{\normof{\nabla v}^2}{v} \uinftybar dx.
    \end{align*}
    Let us now define the following function:
    \begin{align*}
        H(r) := \int_0^r \int_s^r g(s,u) du ds = \int_0^r \int_0^u g(s,u) ds du 
    \end{align*}
    It is clear that $H(0) = 0$ and that:
    \begin{align*}
         H'(r) = \int_0^r g(s,r) ds.
    \end{align*}
    Therefore we also have $H'(0) = 0$. Let us fix some $a > 0$, from the $\Phi$-regularity condition, we justify that the function $g(s,u)$ and $\frac{\partial g}{\partial u}$ are continuous, and therefore uniformly continuous on the compact $[0,a] \times [0,a]$ (by Heine's theorem). Thus, let us fix some $\epsilon > 0$ and compute, for $0\leq r < a$:
    \begin{align*}
        \frac{H'(r+\epsilon) - H'(r)}{\epsilon} &= \frac{1}{\epsilon} \int_r^{r + \epsilon} g(s,r+\epsilon) ds + \frac{1}{\epsilon} \int_0^r (g(s, r+\epsilon) - g(s,r)) ds \\
        &= \frac{1}{\epsilon} \int_r^{r + \epsilon} g(s,r) ds +   \frac{1}{\epsilon} \int_r^{r + \epsilon} (g(s,r+\epsilon) -  g(s,r) ) ds + \frac{1}{\epsilon} \int_0^r (g(s, r+\epsilon) - g(s,r)) ds 
    \end{align*}
    For the first term, we clearly have, by definition of the derivative:
    \begin{align*}
        \frac{1}{\epsilon} \int_r^{r + \epsilon} g(s,r) ds \underset{\epsilon \to 0}{\longrightarrow} g(r,r).
    \end{align*}
    From the fact that $\frac{\partial}{\partial u}$ is continuous, on $[0,a] \times [0,a]$, we deduce that it also uniformly continuous on this set. Therefore, we have:
    \begin{align*}
        \left| \frac{1}{\epsilon} \int_r^{r + \epsilon} (g(s,r+\epsilon) -  g(s,r) ) ds \right| \leq \epsilon \normof{\frac{\partial}{\partial u}}_{L^\infty ([0,a]^2)} \underset{\epsilon \to 0}{\longrightarrow} 0.
    \end{align*}
    From the dominated convergence, thanks to Lemma \ref{lemma:g_c1}, we can differentiate under the integral sign and get that:
    \begin{align*}
         \frac{1}{\epsilon} \int_0^r (g(s, r+\epsilon) - g(s,r)) ds \underset{\epsilon \to 0}{\longrightarrow} \int_0^r \frac{\partial g}{\partial u} (s,r) ds.
    \end{align*}
    Finally, we have: $H''(0) = g(0,0)$. This also implies:
    \begin{align*}
        J_v(r) = \frac{2d}{r^2\sigma_{d-1}} H(r) \underset{r \to 0}{\longrightarrow} \frac{2d}{\sigma_{d-1}} \frac{1}{2} g(0,0) = \mathbf{J}_v =  \int \frac{\normof{\nabla v}^2}{v} \uinftybar.
    \end{align*}
    This is the desired result.
\end{proof}

We therefore have the following integral representation of the Bregman integral term:
\begin{align}
    \label{eq:bregman_integral_jv_representation}
    \iint \bregman{v(x)}{v(x+z)} \uinftybar d\nu_\alpha(z) dx = \frac{\sigma_{d-1}}{2d} \int_0^\infty J_v(r) \frac{dr}{r^{\alpha - 1}}.
\end{align}
Note moreover that, by non-negativity of the Bregman divergence, the function $J_v$ is non-negative.

With the notations of \cref{lemma:big_decomposition}, we have:
\begin{align}
    \label{eq:spherical_representation_final_formula}
    B_\Phi^\alpha(v) = C_{\alpha, d} \frac{\sigma_{d-1}}{2d} \int_0^\infty J_v(r) \frac{dr}{r^{\alpha - 1}}.
\end{align}

Therefore, \cref{eq:bregman_integral_jv_representation} is the justification of \cref{eq:integral_representation_main} given in \cref{sec:pure_levy_case}.

\begin{remark}[About the value of $R$ in \cref{ass:jv_assumption}]
    \label{remark:value-of-R}
    In \cref{sec:main_results}, we argue that the value of the parameter $R$ introduced in \cref{ass:jv_assumption} can be sufficiently large when the function $v$ is reasonable close to a constant function (\ie, $v \equiv 1$). Let us make this argument slightly more formal. Let us fix such a function $v$ and assume that there exists $\varepsilon_1, \varepsilon_2 > 0$ such that, uniformly on $\Rd$, we have:
    \begin{align}
        \label{eq:v_almost_constant_formalization}
        v \geq \varepsilon_1, \quad \normof{\nabla v} \leq \varepsilon_2, \quad \normof{\nabla^2 v} \leq \varepsilon_2.
    \end{align}
    We also assume that $v$ is twice continuously differentiable.
    Then, based on the proof of \cref{lemma:spherical_information_at_zero}, we have:
    \begin{align*}
        |J_v(r) - J_v(0)| \leq J_v(r) := \frac{2d}{r^2\sigma_{d-1}}\int_0^r \int_s^r \int_{\Rd} \int_\Sd \Delta(x,\theta,s,u) \uinftybar(x) d\theta dx du ds,
    \end{align*}
    where:
    \begin{align*}
        \Delta(x,\theta,s,u) := \left| \frac{\theta \cdot \nabla v(x + s\theta)  \theta \cdot \nabla v(x + u\theta)}{v(x + u\theta)}  - \frac{(\theta \cdot \nabla v(x) )^2}{v(x)} \right|
    \end{align*}
    Using the conditions \eqref{eq:v_almost_constant_formalization}, we can easily see that:
    \begin{align*}
        \Delta(x,\theta,s,u) \leq C_1 \normof{\theta}^2 \left( u + s \right),
    \end{align*}
    with $C_1$ a constant depending on $\varepsilon_1$ and $\varepsilon_2$. Therefore, we have $|J_v(r) - J_v(0)| \leq C_2 r$, with $C_2$ a constant depending on $\varepsilon_1$, $\varepsilon_2$ and $d$. This shows that the derivative of $J_v$ in $0$ can be controlled, hence allowing $R$ to be big enough.
\end{remark}

\subsection{Pure Levy case: $\sigma_2 = 0$ - Additional results with Bregman Fisher inequalities}
\label{sec:proofs-pure-levy-case-bregman-fisher}

Before proving our main results, \ie the results of \cref{sec:pure_levy_case}, we quickly present a more general point of view. The message of this section is the following, any inequality of the form:
\begin{align*}
     I_\Phi(v) \lesssim B_\Phi^\alpha(v),
\end{align*}
where $B_\Phi^\alpha(v)$ and $I_\Phi(v)$ have been defined in \cref{lemma:big_decomposition}. In all this section, we assume $\sigma_2 = 0$.

We now deduce generalization bounds in the case where we do not have any Brownian part in the bounds. We denote, as before: 
\begin{align*}
    v_t^S := \frac{u_t^S}{\uinftybar}.
\end{align*}
As we did repeatedly until now, we often omit the dependence of $v$ on $t$ and $S$.

In this subsection, we introduce one of the main ingredient behind our proof of generalization bounds in the pure heavy-tailed case, \ie $\sigma_2 = 0$.

The main argument is that such bounds appear if we can use the Bregman integral to control the Fisher information coming from Young's inequality in the proof of \cref{cor:gen_brownian_time}. We formalize the connection between such a functional inequality by the following definition. We will then see how the results of the previous sections can make this inequalities happen in practice.

\begin{definition}
    Given a smooth convex function $\Phi$, we introduce the notion of Bregman-Fisher inequality, denoted $\text{BF}_\Phi(\gamma, \alpha)$, for $\gamma > 0$. For a (smooth enough) function $v$, we say that $v$ satisfies $\text{BF}_\Phi(\gamma, \alpha)$, with respect to $\uinftybar$, when:
    \begin{align*}
         \iint \bregman{v(x)}{v(x+z)} \uinftybar d\nu_\alpha(z) dx \geq \frac{\gamma}{C_{\alpha,d}}  \int \frac{\normof{\nabla v}^2}{v} \uinftybar dx.
    \end{align*}
\end{definition}

It is clear that we have the following result:

\begin{theorem}
    \label{thm:result_from_bregman_fisher}
  We make Assumptions \ref{ass:subgaussian}, \ref{ass:smooth} and \ref{ass:phi_regularity}, where the $\Phi$-regularity is for the  $\Phi(x) := x\log(x)$. We further assume that there exists a constant $\gamma > 0$ for each $t>0$ and $S \in \zcal^n$, the function $v_t^S$ satisfies $\text{BF}_\Phi(\gamma, \alpha)$ with respect to $\uinftybar$. Then, with probability at least $1 - \zeta$ over $\datadist$, we have:
    \begin{align*}
        \Eof[U]{L(W^S_T) - \el(W^S_T)} \leq 2s \sqrt{\frac{1}{4n\sigma_1^\alpha \gamma} \int_0^T \Eof[U]{\normof{ \nabla \ef(W^S_t) }^2 } dt  + \frac{\Lambda + \log(3/\zeta)}{n}}.
    \end{align*}
    where $\Lambda := \limsup_{t\to 0} \klb{\rho_t^S}{\uinftybar}$.
\end{theorem}

\begin{proof}
    From Theorem \ref{lemma:big_decomposition}, the BF condition and Young's inequality, we immediately get, with $v = v_t^S$:
    \begin{align*}
        \frac{d}{dt} \entphi{v} &\leq -  C_{\alpha,d}  \sigma_1^\alpha \iint \bregman{v(x)}{v(x+z)} \uinftybar d\nu_\alpha(z) dx + \int \Phi''(v) v \nabla v \cdot \nabla(V - V_S) \uinftybar \\ 
        &\leq -\gamma   \sigma_1^\alpha  \int \frac{\normof{\nabla v}^2}{v} \uinftybar dx +\int  \nabla v \cdot \nabla(V - V_S) \uinftybar \\ 
        &\leq \frac{1}{4 \gamma \sigma_1^\alpha} \int \normof{\nabla V - \nabla V_S}^2 du_t^S.
    \end{align*}
    Therefore, by using the same reasoning as in Corollary \ref{cor:gen_brownian_time}, we get the results.
\end{proof}

\subsection{Pure Levy case: $\sigma_2 = 0$ - Omitted proofs of \cref{sec:pure_levy_case}}

Functional inequalities like $\text{BF}_\Phi(\gamma, \alpha)$ are not trivial at all to get in practice. In the rest of this subsection, we will justify that, under a reasonable assumption, we can satisfy an almost identical identity. This will give us an idea of the rate that we expect for our bounds.

The idea is the following: the results of the previous section, namely Equation \eqref{eq:bregman_integral_jv_representation} and Lemma \ref{lemma:spherical_information_at_zero} point us toward the following informal computation, for some $R>0$:
\begin{align*}
    \iint \bregman{v(x)}{v(x+z)} \uinftybar d\nu_\alpha(z) dx \geq \frac{\sigma_{d-1}}{2d} \int_0^R J_v(r) \frac{dr}{r^{\alpha - 1}} \simeq \frac{\sigma_{d-1}R^{2 - \alpha}}{2d(2 - \alpha)} J_v(0).
\end{align*}
Therefore, we see that we can control the Fisher information terms, coming from Young's inequality, using the above integral. However, we need an additional assumption to control the behavior of the function $J_v$, uniformly with respect to the data and the time.

We formalize this idea with the following assumption:

\begin{assumption}
    \label{ass:jv_assumption_appendix}
    We assume there exists an absolute constant $R > 0$ such that, for all $t> 0$ and $\datadist$-almost all $S \in \zcal^n$, we have:
    \begin{align*}
        \forall r \in [0,R],~J_{v^S_t}(r) \geq \frac{1}{2} \mathbf{J}_{v^S_t}.
    \end{align*}
    This assumption is exactly a reformulation of \cref{ass:jv_assumption}, with $\Phi(x) = \Phi_{\log} = x \log(x)$.
\end{assumption}

If we fix $t$ and $S$, this assumption is trivial by the continuity of $J_v$. The above condition is essentially a kind of weak uniformity in $t$ and $S$ of this continuity. The uniformity in $t$ would be justified in case of convergence of $u_t^S$ to a limit distribution. The strongest part of the assumption is the uniformity in $S$. Note that it is common in the learning theory literature to assume uniformity of various constant in the data. 

\thmLevyCaseKL*

\begin{proof}
    Let us fix $t> 0$ and $S \in \zcal^n$ and denote $v$ for $v_t^S$, as before. By Theorem \ref{lemma:big_decomposition}, we have:
    \begin{align*}
        \frac{d}{dt} \entphi{v} &\leq -  C_{\alpha,d}  \sigma_1^\alpha \iint \bregman{v(x)}{v(x+z)} \uinftybar d\nu_\alpha(z) dx + \int \nabla v \cdot \nabla(V - V_S) \uinftybar \\
    \end{align*}
    By Young's inequality, if $C>0$ is a constant:
    \begin{align*}
        \frac{d}{dt} \entphi{v} &\leq -  C_{\alpha,d}  \sigma_1^\alpha \iint \bregman{v(x)}{v(x+z)} \uinftybar d\nu_\alpha(z) dx + \frac{C}{2} \int \frac{\normof{\nabla v}^2}{v} \uinftybar dx + \frac{1}{2C}\int \normof{\nabla V - \nabla V_S}^2 du_t^S. \\
    \end{align*}
    By Assumption \ref{ass:jv_assumption} and \cref{lemma:spherical_representation}, we have:
    \begin{align*}
        \frac{d}{dt} \entphi{v} &\leq -  C_{\alpha,d}  \sigma_1^\alpha \frac{\sigma_{d-1}}{2d} \int_0^R J_v(r) \frac{dr}{r^{\alpha-1}}  + \frac{C}{2} \int \frac{\normof{\nabla v}^2}{v} \uinftybar dx + \frac{1}{2C}\int \normof{\nabla V - \nabla V_S}^2 du_t^S. \\
        &\leq -  C_{\alpha,d}  \sigma_1^\alpha \frac{\sigma_{d-1}}{4d} \int_0^R J_v(0) \frac{dr}{r^{\alpha-1}} + \frac{C}{2} \int \frac{\normof{\nabla v}^2}{v} \uinftybar dx  + \frac{1}{2C}\int \normof{\nabla V - \nabla V_S}^2 du_t^S.\\
        &= - \frac{C_{\alpha,d}  \sigma_1^\alpha\sigma_{d-1}R^{2-\alpha}}{4d (2 - \alpha)}  \int \frac{\normof{\nabla v}^2}{v} \uinftybar dx + \frac{C}{2} \int \frac{\normof{\nabla v}^2}{v} \uinftybar dx  + \frac{1}{2C}\int \normof{\nabla V - \nabla V_S}^2 du_t^S.\\
    \end{align*}
    So we make the choice:
    \begin{align*}
        C = \frac{C_{\alpha,d}  \sigma_1^\alpha\sigma_{d-1}R^{2-\alpha}}{2d (2 - \alpha)} ,
    \end{align*}
    and, putting everything together, we have:
    \begin{align*}
        \frac{d}{dt} \entphi{v} \leq \frac{K_{\alpha,d}}{\sigma_1^\alpha} \int \normof{\nabla V - \nabla V_S}^2 du_t^S,
    \end{align*}
   with:
   \begin{align*}
       K_{\alpha, d} := \frac{d(2 - \alpha)}{C_{\alpha,d}  \sigma_{d-1}R^{2-\alpha}}
   \end{align*}
    We conclude by the same PAC-Bayesian arguments as in the proof of \cref{cor:gen_brownian_time}, \ie, we use Theorem \ref{thm:pb_for_subgaussian}.

    Regarding the value of the constant $K_{\alpha,d}$, we remind the reader that we have:
    \begin{align*}
        C_{\alpha, d} = \alpha 2^{\alpha-1}\pi^{-d/2} \frac{\Gamma\left(  \frac{\alpha + d}{2}\right)}{\Gamma \left( 1 - \frac{\alpha}{2} \right)}, \quad \sigma_{d-1} = \frac{2\pi^{d/2}}{\Gamma(d/2)}.
    \end{align*}
    A simple computation gives:
    \begin{align*}
            K_{\alpha,d} = \frac{1}{ R^{2 - \alpha}} \frac{2 - \alpha}{\alpha 2^{\alpha}}\Gamma \left( 1 - \frac{\alpha}{2} \right) \frac{d \Gamma\left(  \frac{d}{2}\right)}{\Gamma\left(  \frac{\alpha + d}{2}\right)} ,
    \end{align*}
    which is the desired result.
\end{proof}

\subsection{Extension to the discrete-time case}
\label{sec:discrete_case}

In this section, we quickly demonstrate that our methods, developped in the time-continuous setting, can be extended to the discrete setting. This gives more theoretical foundations to our experimental analysis. We treat this case slightly less formally than the rest of the paper, our main goal is to make a first step toward the understanding of the discrete heavy-tailed algorithms. Let us simplify the notations of \cref{sec:experiments} and consider the following discrete recursion:
 \begin{align}
 \label{eq:discrete_recursion_appendix}
    w_{k+1}^S = w_k^S - \gamma g(w_k^S) - \gamma \eta w_k^S + \gamma^{1/\alpha} \sigma_1 L_1^\alpha,
\end{align}
where $g(w_k^S) := \nabla \ef (w_k^S)$ and $S \sim \datadist$. We assume that $\gamma\eta < 1$.

\begin{remark}
    We could also consider that $g_k$ is an unbiased estimate of the true gradient, \ie using random batches independent of the stable noise. Most of our analysis would also hold in this case. However, we focus on the full-batch case, both for simplicity and to stick to the theoretical foundations of our experimental work. As mentioned in \cref{sec:experiments}, the used of mini-batches could result in gradient with heavy-tailed noise, which could interfere in an unclear way with the stable noise $\levy$. 
Similarly, the same techniques could be extended to varying learning rate and noise scale, but we stick to a setting close to both the time-continuous and the experimental settings.

\end{remark}

Extensions from the discrete 
In this section, we adapt the technique presented in \citep[Section $5$]{mou_generalization_2017} to the heavy-tailed setting. This will highlight that our technical contribution in the continuous case are directly useful for the discrete case. The strategy is the following:
\begin{enumerate}
    \item We will construct a Levy driven Ornstein-Uhlenbeck process interpolating between the density of two successive iterates.
    \item We apply the analysis of \cref{sec:pure_levy_case} and use the associated FPE to bound the KL divergences of each iterate $w_k$.
\end{enumerate}

Interpolating techniques have also been used by \citet{nguyen_first_2019}, in the study of the discretization of heavy-tailed SDEs.

Let us fix some $\sigma'>0$ and consider the process, for a fixed $z \in \Rd$:
\begin{align*}
    d X_t = -\eta X_t dt - g(z) dt + \sigma'd\levy,
\end{align*}
where, as defined above, $g = \nabla \ef$.
Note that $X$ depends on $z$. We can express the solution as:
\begin{align*}
    X_t + \frac{g(z)}{\eta} = e^{-\eta t } \left(  X_0 + \frac{g(z)}{\eta} \right) + \underbrace{\sigma' \int_0^t e^{-\eta (t - s)} dL_s^\alpha}_{:= \mathcal{O}_t}.
\end{align*}
We can compute the characteristic function of the integral term, using computations similar as in \citep[Lemma $9$]{raj_algorithmic_2023-1}, for all $\xi \in \Rd$:
\begin{align*}
            \Eof{e^{i \xi \cdot \mathcal{O}_t}} = \exp \left\{ -\int_0^t \normof{\sigma' e^{-\eta s} \xi}^\alpha  ds \right\} = \exp \left\{ -\sigma'^\alpha \normof{ \xi}^\alpha  \frac{1 - e^{-\eta \alpha t}}{\alpha \eta} \right\}.
\end{align*}
Let us now fix one iteration $k$ and denote by $u_k^S$ the probability density of $w_k^S$. 
We set the initial condition $X_0 \sim u_k^S$. Then, by the two previous equations, we have that, for a fixed $\tau >0$:
\begin{align*}
    X_\tau \sim e^{-\eta \tau} u_k^S + \frac{1 - e^{-\eta \tau}}{\eta} g(u_k^S) + \sigma' \left\{  \frac{1 - e^{-\eta \alpha \tau}}{\alpha \eta} \right\}^{\frac{1}{\alpha}} L_1^\alpha.
\end{align*}
Our goal is that $X$ interpolates between $u_k^S$ and $u_{k+1}^S$, therefore we set:
\begin{align*}
    e^{-\eta \tau} = 1 - \gamma \eta, \quad \sigma' \left\{  \frac{1 - e^{-\eta \alpha \tau}}{\alpha \eta} \right\}^{\frac{1}{\alpha}}= \sigma_1 \gamma^{\frac{1}{\alpha}},
\end{align*}
so that is reproduced \cref{eq:discrete_recursion_appendix}.
Thus $X_\tau$ and $u_{k+1}^S$ have the same distributions. 

Let us denote by $h_t^{k,S}$ the density of $X_t$ at time $t$, were we made explicit its dependence on the data $S$ and the iteration number $k$.
Now, if we proceed as in \citep[Theorem $9$]{mou_generalization_2017}, and integrate the FPE of $X$ with respect to $u_k^S$, we get that the density $h_t^{k,S}$ of $X_t$ satisfies, provided we can switch the differential operators and the integration over $u_k^S$:
\begin{align*}
    \partial_t h_t^{k,S} = - \sigma'^\alpha \fraclap h_t^{k,S} + \eta \nabla \cdot (h_t^{k,S} w) + \nabla \cdot (h_t^{k,S} \Eof{g(w_k^S)|X_t = w})
\end{align*}

This fractional FPE has exactly the form studied in this paper, therefore, we can express the associated entropy flow as (with a slight abuse of notation, we identify $h_t$ with the associated probability distribution):
\begin{align*}
    \frac{d}{dt} \klb{h_t^{k,S}}{\pi} = -\iint B_\Phi (v_t^{k,S}(x), v_t^{k,S}(x+z)) \uinftybar(x) \frac{dz}{\normof{z}^{\alpha + d}} dx - \int \nabla v_t^{k,S} \cdot  \Eof{g(w_k^S)|X_t = w} dw,
\end{align*}
with:
\begin{align*}
    v_t^{k,S} := \frac{h_t^{k,S}}{\uinftybar}.
\end{align*}

This leads us to formulate the following assumption, which is the extension of \cref{ass:jv_assumption} to the discrete case.

\begin{assumption}
\label{ass:jv_discrete_case}
    Let us fix $\Phi(x) := x\log(x)$.
    We assume that there exists a constant $R$ such that, for all $k$, all $t$ and all dataset $S$, the functions $v_t^{k,S}$, constructed by the above procedure, satisfy:
    \begin{align*}
        \forall r < R, \quad J_{\Phi, v_t^{k,S}} (r) \geq \frac{1}{2} J_{\Phi, v_t^{k,S}} (0).
    \end{align*}
    For technical reasons, we also assume (only in this subsection) that for all $k$, the functions $t \mapsto \klb{h_t^{k,S}}{\pi} $ are continuous at $t=0$. While this assumption should be mild in several applications, it does not always hold and is necessary for our computations in the discrete setting.
\end{assumption}
We now omit, as we did in the time continuous setting, the dependence of $v_t^{k,S}$ on $k$ and $S$, and just denote it $v_t$. We do the same for $h_t^{k,S}$, simply denoted $h_t$.
Therefore, by the proof of \cref{thm:bound_under_jv_assumption}, we have, for $t>0$:
 \begin{align*}
        \frac{d}{dt} \entphi{v_t} \leq \frac{K_{\alpha,d}}{\sigma'^\alpha} \int \normof{\Eof{g(w_k^S)|X_t = w}}^2 dh_t,
    \end{align*}
with $K_{\alpha, d}$ defined as in \cref{thm:bound_under_jv_assumption}, using the constant $R$ coming from \cref{ass:jv_discrete_case}. Let us denote by $h_t(w,z)$ the joint density of $X_t$ and $w_k^S$, by Cauchy-Schwarz's inequality, we have, for $t>0$:
 \begin{align*}
        \frac{d}{dt} \entphi{v_t} &= \frac{K_{\alpha,d}}{\sigma'^\alpha} \int \normof{\int \frac{h_t(w,z)}{h_t(w)} g(z) dz}^2 h_t(w) dw \\
        &\leq \frac{K_{\alpha,d}}{\sigma'^\alpha} \int h_t(w) \left( \int \frac{h_t(w,z)}{h_t(w)^2} dz\right)   \left( \int h_t(w,z) g(z) dz\right)dw \\
        &= \frac{K_{\alpha,d}}{\sigma'^\alpha} \iint h_t(w,z) g(z) dz dw \\
        &= \frac{K_{\alpha,d}}{\sigma'^\alpha}\Eof[U]{\normof{g(w_k^S)}^2}.
\end{align*}
Let us denote by $\rho_k$ the density of $w_k$.
Recall that in \Cref{ass:jv_discrete_case}, we assumed that $t \mapsto \klb{h_t^{k,S}}{\pi} $ is continuous at $t=0$.
Therefore, by integrating between $0$ and $\tau$, using that $h_0$ is the density of $w_k$, and that $h_\tau$ is the density of $w_{k+1}$, we can write that:
\begin{align*}
    \klb{\rho_{k+1}}{\pi} \leq \klb{\rho_{k}}{\pi} + \tau \frac{K_{\alpha,d}}{\sigma'^\alpha} \Eof[U]{\normof{g(w_k^S)}^2},
\end{align*}
where, as in the rest of the paper, $U$ denotes the randomness coming from the stable noise, \ie the noise due to $\levy$.

Therefore, a telescopic sum immediately gives that, for a fixed number $N$ of iterations:
\begin{align*}
    \klb{\rho_N}{\pi} &\leq \Lambda + \tau \frac{K_{\alpha,d}}{\sigma'^\alpha} \sum_{k=0}^{N-1}\Eof[U]{\normof{g(w_k^S)}^2} \\
    &= \Lambda +\tau \frac{K_{\alpha,d}}{\sigma'^\alpha} \sum_{k=0}^{N-1}\Eof[U]{\normof{g(w_k^S)}^2}\\
    &= \Lambda +\frac{K_{\alpha,d}}{\sigma_1^\alpha} \frac{1}{\gamma \eta} \log \left( \frac1{1 - \gamma \eta} \right) \left( \frac{1 - (1 - \gamma \eta)^\alpha}{\alpha \eta} \right)\sum_{k=0}^{N-1}\Eof[U]{\normof{g(w_k^S)}^2},
\end{align*}
with $\Lambda = \klb{\rho_0}{\pi}$.

Therefore, under the subgaussian assumption, \cref{ass:subgaussian}, we have proven that, with probability at least $1 - \zeta$ over $S \sim \datadist$, we have:

\begin{align}
    \label{eq:bound_discrete_case}
    \Eof[U]{G_S(w_N^S)} \leq 2s\sqrt{\frac{K_{\alpha,d}}{\sigma_1^\alpha} \frac{1}{\gamma \eta} \log \left( \frac1{1 - \gamma \eta} \right) \left( \frac{1 - (1 - \gamma \eta)^\alpha}{\alpha \eta} \right) \sum_{k=0}^{N-1}\Eof[U]{\normof{\nabla \ef (w_k^S)}^2} + \frac{\Lambda + \log(3/\zeta)}{n}}.
\end{align}

Finally, let us notice that when $\gamma \eta$ is small, which is the case in our experiments, see \cref{sec:hyperparameters}, we have the following asymptotic development:
\begin{align*}
    \frac{1}{\gamma \eta} \log \left( \frac1{1 - \gamma \eta} \right) \left( \frac{1 - (1 - \gamma \eta)^\alpha}{\alpha \eta} \right) \underset{\gamma \eta \to 0}{\sim} \gamma.
\end{align*}
Therefore, the computations made in this section, while slightly informal, give a solid justification to \cref{eq:G_estimation_main}, which we use in our experiments. The qualitative behavior with respect to $\alpha$ and $d$ is unchanged.

\begin{remark}
    The analysis presented in this section could also be extended to SDEs with a non trivial Brownian contribution to the noise, \ie the setting of \cref{sec:brownian case} where $\sigma_2 >0$.
\end{remark}

\subsection{Mixing of Brownian and heavy-tailed noise}
\label{sec:noise_mixing}

Because the Fisher information term, appearing in \cref{lemma:big_decomposition}, is non-negative, the proof of \cref{thm:bound_under_jv_assumption} also applies when $\sigma_2>0$, \ie in the setting of \cref{sec:brownian case}. 
However, not that this requires to make \cref{ass:jv_assumption}, while such an assumption is not necessary to derive \cref{cor:gen_brownian_time}. Therefore, the results of \cref{sec:brownian case} still present the advantage to hold under lighter assumptions.

Nevertheless, this motivates to write down the result in the case $\sigma_2 > 0$, under \cref{ass:jv_assumption}. This allows to get insights on how both noise affect the generalization error. This leads to the following theorem.

\begin{theorem}
    \label{thm:noise_mixing}
    Let us consider the dynamics of \cref{eq:multifractal_dynamics}, with both $\sigma_1 > 0$ and $\sigma_2 > 0$.
    We make Assumptions \ref{ass:phi_regularity}, \ref{ass:phi_risk_integrability} and \ref{ass:jv_assumption}.
    Then, with probability at least $1 - \zeta$ over $\datadist$, we have
    \begin{align*}
        G_S(T) \leq 2s \sqrt{\frac{M(\sigma_1, \sigma_2, d, \alpha)}{n} \int_0^T \Eof[U]{\normof{\ef (W_t^S)}^2} dt + \frac{\log(3/\zeta) + \Lambda}{n}}
    \end{align*}
    with $\Lambda = \klb{\rho_0}{\pi}$, and the noise mixing constant $M(\sigma_1, \sigma_2, d)$ is given by:
    \begin{align*}
        M(\sigma_1, \sigma_2, d, \alpha) := \frac{1}{4\sigma^2 + \frac{\sigma_1^\alpha \alpha 2^\alpha \Gamma \left( \frac{d + \alpha}{2}  \right) R^{2 - \alpha}}{(2 - \alpha) \Gamma \left( 1 - \frac{\alpha}{2}  \right) d \Gamma \left( \frac{d}{2}  \right)  } }.
    \end{align*}
\end{theorem}

\begin{proof}
    The proof follows by the exact same computations than the proofs of \cref{cor:gen_brownian_time} and \cref{thm:bound_under_jv_assumption}.
\end{proof}

Based on the above theorem, we can make an interesting observation. The constant $M(\sigma_1, \sigma_2, d, \alpha)$ clearly shows the relative contribution of both noises. Therefore, we can argue that they have the same contribution when the following condition is fulfilled:
\begin{align*}
    4\sigma_2^2 = \frac{\sigma_1^\alpha \alpha 2^\alpha \Gamma \left( \frac{d + \alpha}{2}  \right) R^{2 - \alpha}}{(2 - \alpha) \Gamma \left( 1 - \frac{\alpha}{2}  \right) d \Gamma \left( \frac{d}{2}  \right)  }.
\end{align*}
Using the asymptotic developments discussed in \cref{sec:qualitative_analysis}, and proven in \cref{sec:proofs-qualitative_analysis}, in the limit $d \to \infty$, we can write this condition as:
\begin{align*}
    4 \sigma_2^2 \underset{d \to \infty}{\sim} \frac1{R^{2 - \alpha}}{P_\alpha} \frac{\sigma_1^\alpha d ^{\frac{\alpha}{2}}}{d},
\end{align*}
where the pre-factor $P_\alpha$ has been defined and studied in \cref{sec:qualitative_analysis}. It is now meaningful to write it as:
\begin{align}
    \label{eq:equl_noise_contribution}
    (\sigma_2 \sqrt{d_0})^2 = \frac1{4 P_\alpha} (\sigma_1 \sqrt{d_0})^\alpha,
\end{align}
with $d_0 := d/(R^2)$, the "reduced dimension introduced in \cref{sec:qualitative_analysis}. This is the equal noise contribution condition. What is particularly noticeable in \cref{eq:equl_noise_contribution} is that, as was the case in the study of the phase transition, in \cref{sec:qualitative_analysis}, the meaningful quantities, for both noises, is \textbf{(scale of the noise)} $\times$ $\sqrt{\textbf{dimension}}$.

\subsection{Time-uniform bounds deduced from a generalized Poincaré inequality - Omitted proofs of \cref{sec:toward_time_uniform}}
\label{sec:proofs-poincare-inequality}

In this section, we fix the convex function $\Phi$ to be $\Phi(x) = \frac{1}{2} x^2$. In that case, the Bregman divergence becomes symmetric and satisfies:
\begin{align*}
    \bregman{a}{b} = \bregman{b}{a} = \frac{1}{2} (a - b)^2.
\end{align*}

Combined with the logarithmic Sobolev inequality, given by \Cref{thm:generalized_lsi}, this allows us to prove a time uniform bound on the chi-squared distance between the prior and the posterior.

We denote, as above:
\begin{align*}
    V_S(w) := \ef(w) + \frac{\eta}{2} \normof{w}^2, \quad \text{and: } V(w) := \frac{\eta}{2} \normof{w}^2,
\end{align*}
with $\eta > 0$ the regularization coefficient.

Before bounding the $\Phi$-entropy, we first need the following lemma. It is close to remarks already made in \citep{gentil_logarithmic_2008} and \citep{tristani_fractional_2013}, for instance.

\begin{lemma}
    \label{lemma:steady_state_regularized}
    In this setting, $\uinftybar$ is the density of an infinitely divisible probability distribution, whose characteristic exponent is given by:
    \begin{align*}
        \forall \xi \in \Rd,~ \psi(\xi) = \frac{\sigma_2^2}{2\eta} \normof{\xi}^2 + \frac{\sigma_1^\alpha}{\alpha \eta} \normof{\xi}^\alpha.
    \end{align*}
\end{lemma}

\begin{proof}
    Let us first recall some notations. $\pi$ is the (prior) density distribution whose density with respect to the Lebesgue measure is $\uinftybar$. The existence of such the steady state $\uinftybar$ was obtained in \citep{gentil_logarithmic_2008, tristani_fractional_2013}. 
    Moreover, \citet{gentil_logarithmic_2008} showed that $\pi$ is infinitely divisible, therefore, it makes sense to introduce its characteristic exponent, denoted $\psi$, and defined by:
    \begin{align*}
       \forall \xi \in \Rd,~ \Eof[X \sim \pi]{e^{i \xi \cdot X}} = e^{-\psi(\xi)}.
    \end{align*}
    Following computations of \citet{gentil_logarithmic_2008, tristani_fractional_2013}. We quickly show how the expression of the characteristic exponent can easily be obtained in our particular case. First, because of the symmetry of \cref{eq:steady_state}, both $\uinftybar$ and $\psi$ must be symmetric (\ie even functions). Therefore, if we define the Fourier transform by:
    \begin{align*}
        \mathcal{F}(u)(\xi) = \int_\Rd e^{-i w \cdot \xi} u(w) dw,
    \end{align*}
    then we have:
    \begin{align*}
        \mathcal{F}(\uinftybar)(\xi) = e^{-\psi(x)}.
    \end{align*}
    We can take, at least formally \citep{tristani_fractional_2013}, the Fourier transform of \cref{eq:steady_state} and write that, in the sense of distributions:
    \begin{align*}
         \mathcal{F}(\uinftybar)(\xi) \left( -\sigma_1^\alpha \normof{\xi}^\alpha - \sigma_2^2 \normof{\xi}^2 \right) + \eta \mathcal{F}(\nabla \cdot (w \uinftybar(w)))(\xi) = 0.
    \end{align*}
     In the above, we used the expression of Fourier transform for the Laplacian and the fractional Laplacian. Using properties of the Fourier transform, this leads to:
     \begin{align*}
         \mathcal{F}(\uinftybar)(\xi) \left( -\sigma_1^\alpha \normof{\xi}^\alpha - \sigma_2^2 \normof{\xi}^2 \right) - \eta \nabla \mathcal{F}(\uinftybar)(\xi). 
     \end{align*}
     Therefore:
     \begin{align*}
         \eta\xi \cdot \nabla \psi(\xi) = \sigma_1^\alpha \normof{\xi}^\alpha + \sigma_2^2 \normof{\xi}^2.
     \end{align*}
     Given that $\psi$ is symmetric, we can write it as $P(\normof{\xi}^2)$, we can then rewrite the previous equation as (for $a > 0$):
     \begin{align*}
         P'(a) = \frac{\sigma_1^\alpha a^{\alpha/2 - 1} + \sigma_2^2}{2\eta}.
     \end{align*}
     Integrating, and noting that we have $P(0) = 0$, leads to the result.
\end{proof}

In this setting, the $\Phi$-entropy inequality, \ie \Cref{thm:generalized_lsi}, becomes the following statement.

\begin{corollary}[Generalized Poincaré's inequality]
    \label{cor:generalized_poincaré}
    Let $\mu$ be an infinitely divisible law on $\Rd$, with associated triplet denoted $(b, Q, \nu)$, then, for every smooth enough function $v$, we have:
    \begin{align*}
        \frac{1}{2}\int v^2 d\mu - \frac{1}{2}\left( \int v d\mu \right)^2 \leq  \int \nabla v \cdot Q \nabla v d\mu + \iint \bregman{v(x)}{v(x+z)} d\nu(z) d\mu(x). 
    \end{align*}
    where we used the symmetry of the Bregman divergence, for $\Phi(x) = \frac{1}{2} x^2$, to revert the arguments of this divergence, appearing in \Cref{thm:generalized_lsi}.
\end{corollary}

If $v$ is chosen to be the Radon-Nykodym derivative of another measure $\nu$, with respect to $\mu$, we recognize the chi-squared distance, defined by:
\begin{align}
    \label{eq:chi-squared}
    \chisq{\nu}{\mu} := \int \left( \frac{d \nu}{d \mu} \right)^2 d\mu - 1.
\end{align}

We also remind the definition of Renyi entropies, for $\beta > 1$:
\begin{align}
    \label{eq:renyi_entropies}
    \renyi{\nu}{\mu} = \frac{1}{\beta - 1} \log \int \left( \frac{d \nu}{d \mu} \right)^\beta d\mu .
\end{align}

Note that, by convention, we also set $\renyi[1]{\nu}{\mu} = \klb{\nu}{\mu} $, see \citep{van_erven_renyi_2014} for an extensive review of the properties of those entropies.

The following inequalities are clear, the first one being proven in \citep{van_erven_renyi_2014}, it is a direct consequence of Jensen's inequality:
\begin{align}
    \label{eq:chi_entropy_inequalities}
    \klb{\nu}{\mu} \leq \renyi[2]{\nu}{\mu} = \log (\chisq{\nu}{\mu} + 1) \leq \chisq{\nu}{\mu}.
\end{align}

\subsubsection{Warmu -up: the case $\sigma_2 > 0$}
\label{sec:proofs-poincare-brownian-case}

In order to mimic the resoning of \cref{sec:brownian case,sec:pure_levy_case}, we first handle the case where $\sigma_2 > 0$. This is an additional theoretical result.

\begin{restatable}{theorem}{thmChiSquaredBrownian}
    \label{thm:chi_sq_brownian}
    We make Assumptions \ref{ass:phi_regularity} and \ref{ass:phi_risk_integrability}, accordingly to this choice of convex function.
    Then, with probability at least $1 - \zeta$ over $S \sim \datadist$ and $w \sim \rho_T^S$, we have:
    \begin{align*}
        G_S(w) \leq 2s \sqrt{\frac{2}{n\sigma_2^2} \bar{I}(T,S) + \frac{4 e^{-\eta T}\Lambda + \log \frac{24}{\zeta^3}}{n}},
    \end{align*}
    with $\Lambda = \entphi{\rho_0}$, and:
    \begin{align*}
        \bar{I}(T,S) := \int_0^T e^{-\eta (T - t)} \Eof[\pi]{(v_t^S)^2 \normof{\nabla \ef}^2  } dt.
    \end{align*}
\end{restatable}

\begin{proof}
    As before, we fix $t$ and $S$ and ease the notations by denoting $u$ instead of $u_t^S$, and similarly for $v$. For our particular choice of function $\Phi$,by \Cref{lemma:big_decomposition}, the entropy flow is equal to:
    \begin{align*}
         \frac{d}{dt} \entphi{v} &= -\sigma_2^2 \int  \normof{\nabla v}^2 \uinftybar  -  C_{\alpha,d}  \sigma_1^\alpha \iint \bregman{v(x)}{v(x+z)} \uinftybar d\nu_\alpha(z) dx + \int v  \uinftybar \nabla v \cdot \nabla \ef. 
    \end{align*}
    By Young's inequality, we have, for any $C>0$ and $\gamma \in (0,1]$:
    \begin{align*}
        \frac{d}{dt} \entphi{v} &\leq -\left(\sigma_2^2 - \frac{C}{2} \right) \int  \normof{\nabla v}^2 \uinftybar  -  \gamma C_{\alpha,d}  \sigma_1^\alpha \iint \bregman{v(x)}{v(x+z)} \uinftybar d\nu_\alpha(z) dx + \frac{1}{2C}\int v^2 \normof{\nabla \ef}^2  \uinftybar .  
    \end{align*}
    Thanks to Lemma \ref{lemma:steady_state_regularized}, we know that the Levy triplet of $\uinftybar$ is given by $(0,\sigma_2^2/(\eta), (C_{\alpha,d} \sigma_1^\alpha / (\alpha \eta)) \nu_\alpha)$. Therefore, the Poincaré's inequality of Corollary \ref{cor:generalized_poincaré} is given by:
    \begin{align*}
        \entphi{v} \leq \frac{\sigma_2^2}{2\eta} \int  \normof{\nabla v}^2 \uinftybar + \frac{C_{\alpha,d} \sigma_1^\alpha}{\alpha \eta} \iint \bregman{v(x)}{v(x+z)} \uinftybar d\nu_\alpha(z) dx.
    \end{align*}
    Therefore, we have:
    \begin{align}
        \label{eq:chi_sq_brownian_proof_step1}
        \frac{d}{dt} \entphi{v} &\leq -\gamma \alpha \eta \entphi{v}   - \left(\sigma_2^2 - \frac{C}{2} - \frac{\alpha \gamma \sigma_2^2}{2} \right) \int  \normof{\nabla v}^2 \uinftybar  + \frac{1}{2C}\int v^2 \normof{\nabla \ef}^2  \uinftybar. \\ 
    \end{align}
    We make the choice $\gamma = 1/\alpha$ and $C = \sigma_2^2$, this gives:
    \begin{align*}
        \frac{d}{dt} \entphi{v} \leq - \eta \entphi{v} + \frac{1}{2\sigma_2^2}\int v^2 \normof{\nabla \ef}^2  \uinftybar. 
    \end{align*}
    Solving this differential inequality implies that, with $\Lambda := \entphi{\rho_0}$ and $T> 0$,
    \begin{align*}
     \entphi{v_T^S} \leq e^{-\eta T} \Lambda + \frac{1}{2\sigma_2^2} \bar{I}(T,S) ,
    \end{align*}
    with: 
    \begin{align*}
        \bar{I}(T,S) := \int_0^T e^{-\eta (T - t)} \Eof[\pi]{(v_t^S)^2 \normof{\nabla \ef}^2  } dt.
    \end{align*}
    From Equation \eqref{eq:chi-squared}, this implies that:
    \begin{align*}
        \frac{1}{2} \chisq{\rho_T^S}{\pi} \leq e^{-\eta T} \Lambda + \frac{1}{2\sigma_2^2}\int_0^T e^{-\eta (T - t)}I(t,S) dt.
    \end{align*}
    We conclude by applying \cref{thm:pb_for_subgaussian_disintegrated} and noting that, in our notations, we have:
    \begin{align*}
        \renyi[2]{\rho_T^S}{\pi} \leq 2 \chisq{\rho_S^t}{\pi} = 2 \entphi{v}.  
    \end{align*}
\end{proof}

In order to make the above theorem more interpretable, we present, as an additional theoretical result, the following corollary. It states that, if the loss is Lipschitz in $w$ and if the regularization coefficient is big enough, then we have a simpler time-uniform bound. The main interest for the following result is that it it fully comparable with our other results. This is an additional result, it was not presented in the main part of the document.
 
\begin{restatable}{corollary}{corHighRegularization}
    We make the same assumptions than in Theorem \ref{thm:chi_sq_brownian} and further assume that the loss $f(w,z)$ is $L_f$-Lipschitz in $w$, uniformly with respect to $z$. Then, if
    \begin{align}
        \label{eq:lambda_condition}
        a := \eta - \frac{L_f^2}{2\sigma_2^2} > 0,
    \end{align}
    then, with probability at least $1 - \zeta$ over $S \sim \datadist$ and $w \sim \rho_T^S$, we have, for any $\lambda > 0$:
    \begin{align*}
       G_S(w) \leq \frac{2s}{\sqrt{n} }\left\{ \frac{L_f^2}{2\sigma_2^2} \frac{1 - e^{-aT}}{a} + \frac{3}{2} \log\left(\frac{2}{\zeta}\right) +  \frac{\Lambda}{e^{\eta T}} \right\}^{\frac{1}{2}}.
    \end{align*}
    with $\Lambda = \entphi{\rho_0}$.
\end{restatable}

\begin{proof}

    We follow the exact same step than the proof of Theorem \ref{thm:chi_sq_brownian}, up to Equation \eqref{eq:chi_sq_brownian_proof_step1}. We make the same choices for $\gamma$ and $C$ and obtain:
    \begin{align*}
        \frac{d}{dt} \entphi{v} &\leq - \eta \entphi{v} + \frac{1}{2\sigma_2^2}\int v^2 \normof{\nabla \ef}^2  \uinftybar. \\ 
    \end{align*}
    By the Lipschitz assumption, this is:
    \begin{align*}
        \frac{d}{dt} \entphi{v} &\leq - \eta \entphi{v} + \frac{L_f^2}{2\sigma_2^2}\int v^2   \uinftybar. \\ 
        &=  - \left(\eta -  \frac{L_f^2}{2\sigma_2^2}\right) \entphi{v} + \frac{L_f^2}{2\sigma_2^2}. \\ 
    \end{align*}
    The result is deduced by the exact same reasoning as the last steps of the proof of Theorem \ref{thm:chi_sq_brownian}, by applying Theorem \ref{thm:pb_for_subgaussian_disintegrated}.
\end{proof}

\subsubsection{Time uniform bounds in the case $\sigma_2 = 0$ - Proof of \cref{thm:chi_sq_levy}}

In the previous subsection, we used heavily the presence of a non-trivial Brownian part in the operator $I$. However, we face the same problem that we had with Theorem \ref{thm:chi_sq_brownian}: while providing a generalization bound for multifractal dynamics, which is our main goal, this theorem is not very informative about the impact of the tail-index $\alpha$ on the generalization performance. This is why, in this section, we focus, like in \cref{sec:pure_levy_case}, on the case $\sigma_2 = 0$. This follows the steps presented in \cref{sec:pure_levy_case}, with a change in the choice of the convex function $\Phi$.

To overcome this issue, we show how it is possible to derive a bound, even without the Brownian part, by using a reasoning similar to Section \ref{sec:bregman_integral_bounds}. Under Assumption \ref{ass:phi_regularity}, we can apply the same reasoning with the choice of function: $\Phi(x) = x^2 / 2 $.

With this change in the convex function $\Phi$, the relevant quantity is no more the Fisher information, but the following $L^2$-norm on the gradient of $v$, this is the term that appears in Poincaré's inequality:
\begin{align}
    \label{eq:expected_gradient}
    \mathbf{G}_v := \int \normof{\nabla v}^2 \uinftybar dx.
\end{align}

This quantity, $\mathbf{G}_v$, correspond to the $\Phi$-information term, $I_\Phi(v)$, introduced in \Cref{lemma:big_decomposition}, for the particular choice of convex function that we make in this section, namely $\Phi(x) = x^2/2$

The following lemma is proven by the exact same lines than Lemma \ref{lemma:spherical_representation}, only the function $\Phi$, and the corresponding $\Phi$-regularity assumption, changes.

\begin{lemma}
    \label{lemma:spherical_representation_chi_squared}
    Assume that $v$ is $\Phi$-regular, for $\Phi(x) = x^2 / 2 $. Then, we have the following representation:
    \begin{align*}
        \iint \bregman{v(x)}{v(x+z)} d\nu_\alpha(z) \uinftybar(x) dx = \frac{\sigma_{d-1}}{2d} \int_0^\infty G_v(r) \frac{dr}{r^{\alpha - 1}},
    \end{align*}
    where the function $G_v$ is non-negative and defined by:
    \begin{align*}
        G_v(r) := \frac{2d}{r^2\sigma_{d-1}}\int_0^r \int_s^r \int_{\Rd} \int_\Sd \theta \cdot \nabla v(x + s\theta)  \theta \cdot \nabla v(x + u\theta) \uinftybar(x) d\theta dx du ds.
     \end{align*}
     This function can be continuously extended at $0$ by:
     \begin{align*}
         G_v(0) = \mathbf{G}_v.
     \end{align*}
     This function, $G_v: [0,\infty) \longrightarrow [0,\infty)$, corresponds to the function $J_{\Phi, v}$, introduced in the main part of the paper, for the particular choice of convex function $\Phi(x) = \normof{x}^2 / 2 $.
 \end{lemma}

 The following assumption is the equivalent of Assumption \ref{ass:jv_assumption} for the function $G_v$, instead of $J_v$.

 \begin{assumption}
    \label{ass:gv_assumption}
     There exists an absolute constant $R> 0$ such that, for $\datadist$-almost all $S$ and all $t\geq 0$, we have:
     \begin{align*}
         \forall r \in [0,R], ~G_v(r) \geq \frac{1}{2} G_v(0).
     \end{align*}
     This assumption is exactly a reformulation of \cref{ass:jv_assumption}, rewritten with $\Phi(x) = \frac{1}{2}x^2$. Using the denominations of \cref{sec:toward_time_uniform}, it is \cref{ass:jv_assumption}-$\Phi_2$.
 \end{assumption}

 We can now prove \cref{thm:chi_sq_levy}, which is a time-uniform generalization bound obtained in the pure heavy-tailed case ($\sigma_1 = 0$).

 \thmLevyCaseChi*

\begin{proof}
    We start, again, with the expression of the entropy flow.
    \begin{align*}
         \frac{d}{dt} \entphi{v} &=  -  C_{\alpha,d}  \sigma_1^\alpha \iint \bregman{v(x)}{v(x+z)} \uinftybar d\nu_\alpha(z) dx + \int v  \uinftybar \nabla v \cdot \nabla \ef. \\ 
    \end{align*}
    We split in two the Bregman integral term and apply the generalized Poincaré inéquality, along with Young's inequality on the last term, for any $C > 0$:
    \begin{align*}
        \frac{d}{dt} \entphi{v} &\leq - \frac{\alpha \eta}{2} \entphi{v} - \frac{C_{\alpha,d}  \sigma_1^\alpha}{2} \iint \bregman{v(x)}{v(x+z)} \uinftybar d\nu_\alpha(z) dx + \int v  \uinftybar \nabla v \cdot \nabla \ef. \\
        &\leq - \frac{\alpha \eta}{2} \entphi{v} - \frac{C_{\alpha,d}  \sigma_1^\alpha}{2} \iint \bregman{v(x)}{v(x+z)} \uinftybar d\nu_\alpha(z) dx +\frac{C}{2} \mathbf{G}_v + \frac{1}{2C}\int v^2 \normof{\nabla \ef}^2  \uinftybar dx.\\
    \end{align*}
    We now use Lemma \ref{lemma:spherical_representation_chi_squared} and Assumption \ref{ass:gv_assumption}, this gives:
    \begin{align*}
        \frac{d}{dt} \entphi{v} &\leq - \frac{\alpha \eta}{2} \entphi{v} - \frac{C_{\alpha,d}  \sigma_1^\alpha}{2}  \frac{\sigma_{d-1}}{2d}\int_0^R G_v(r) \frac{dr}{r^{\alpha - 1}} +\frac{C}{2} \mathbf{G}_v + \frac{1}{2C}\int v^2 \normof{\nabla \ef}^2  \uinftybar dx \\
        &\leq - \frac{\alpha \eta}{2} \entphi{v} - \frac{C_{\alpha,d}  \sigma_1^\alpha \sigma_1^\alpha R^{2 - \alpha}}{8d (2 - \alpha)} \mathbf{G}_v + \frac{C}{2} \mathbf{G}_v + \frac{1}{2C}\int v^2 \normof{\nabla \ef}^2  \uinftybar dx. \\
    \end{align*}
    Therefore, we make the choice:
    \begin{align*}
        C := \frac{C_{\alpha,d}  \sigma_1^\alpha \sigma_1^\alpha R^{2 - \alpha}}{4d (2 - \alpha)} ,
    \end{align*}
    and we get:
     \begin{align*}
        \frac{d}{dt} \entphi{v} &\leq - \frac{\alpha \eta}{2} \entphi{v} + \frac{2d(2 - \alpha)}{C_{\alpha,d}  \sigma_1^\alpha \sigma_1^\alpha R^{2 - \alpha}}\int v^2 \normof{\nabla \ef}^2  \uinftybar dx. \\
    \end{align*}
    We introduce the same constant as the one introduced in Theorem \ref{thm:bound_under_jv_assumption}, namely:
    \begin{align*}
        K_{\alpha,d} = \frac{d (2 - \alpha)}{C_{\alpha,d} \sigma_{d-1} R^{2 - \alpha}}.
    \end{align*}
    Therefore, solving the above differential inequality gives:
    \begin{align*}
        \entphi{v} \leq e^{-\frac{\alpha \eta T}{2}} \Lambda + \frac{2K_{\alpha,d} }{\sigma_1^\alpha}  \bar{I}(T,S),
    \end{align*}
    where $\bar{I}(T,S)$ and $\Lambda$ are defined in the same way than in Theorem \ref{thm:chi_sq_brownian}.

    We conclude by applying \cref{thm:pb_for_subgaussian_disintegrated} and noting that $\renyi[2]{\rho_S^T}{\pi} \leq 2 \entphi{v^S_T}$.
    
\end{proof}

\section{Asymptotics of the constants - Omitted proofs of Section \ref{sec:qualitative_analysis}}
\label{sec:proofs-qualitative_analysis}

This section is devoting to studying the limit behavior of the constants introduced in our generalization bounds, in terms of $\alpha$ and $d$. This corresponds to the proof of the theoretical results of \cref{sec:qualitative_analysis}.

We first study the asymptotic behavior of the constant when $d \to \infty$.

\constantDimensionEquivalent*

\begin{proof}
    By the definitions of \cref{sec:pure_levy_case,sec:qualitative_analysis}, we have:
    \begin{align*}
        \bar{K}_{\alpha,d} = \frac{(2 - \alpha) \Gamma \left(  1 - \frac{\alpha}{2}\right)}{\alpha 2^\alpha} \frac{d \Gamma \left( \frac{d}{2} \right)}{\Gamma \left( \frac{d+\alpha}{2} \right)}.
    \end{align*}
    By \cref{lemma:stirling}, we have:
    \begin{align*}
         \frac{d \Gamma \left( \frac{d}{2} \right)}{\Gamma \left( \frac{d+\alpha}{2} \right)} &\underset{d \to \infty}{\sim}  \frac{d \Gamma \left( \frac{d}{2} \right)}{\Gamma \left( \frac{d}{2} \right) \left(  
        \frac{d}{2} \right)^{\alpha / 2}} \\
        &= 2^{\alpha/2}d^{1 - \frac{\alpha}{2}},
    \end{align*}
    from which we immediately deduce the result.
\end{proof}

We now prove \cref{lemma:alpha_limit}, which study the low-tail limit of our generalization bounds for heavy-tailed dynamics.

\constantAlphaLimit*

\begin{proof}
    By the definitions of \cref{sec:pure_levy_case,sec:qualitative_analysis}, we have:
    \begin{align*}
        \bar{K}_{\alpha,d} = \frac{(2 - \alpha) \Gamma \left(  1 - \frac{\alpha}{2}\right)}{\alpha 2^\alpha} \frac{d \Gamma \left( \frac{d}{2} \right)}{\Gamma \left( \frac{d+\alpha}{2} \right)}.
    \end{align*}
    We have $\alpha \in [1,2)$, therefore, we know that $\frac{\alpha}{2} \notin \mathds{Z}$. 
    We can then apply Euler's reflection formula, \cref{lemma:euler_reflection}, and get that:
    \begin{align*}
        \Gamma \left(  1 - \frac{\alpha}{2}\right)\Gamma \left( \frac{\alpha}{2}\right) &= \frac{\pi}{\sin \left( \frac{\pi \alpha}{2} \right)} \\
        &= \frac{\pi}{\sin \left( \pi - \frac{\pi \alpha}{2} \right)} \\
        &\underset{\alpha \to  2^-}{\sim}  \frac{1}{ 1 - \frac{\alpha}{2}}.
    \end{align*}
    On the other hand, by continuity of the $\Gamma$ function and using the identity $\Gamma(x+1) = x \Gamma(x)$, we also have:
    \begin{align*}
        \Gamma \left( \frac{d+\alpha}{2} \right) \underset{\alpha \to  2^-}{\sim}  \Gamma \left( \frac{d}{2} + 1 \right) = \frac{d}{2}  \Gamma \left( \frac{d}{2}  \right).
    \end{align*}
    Putting everything together, and using that $\Gamma(1) = 1$, we get that:
    \begin{align*}
        \bar{K}_{\alpha,d}\underset{\alpha \to  2^-}{\sim} \frac{1}{2}. 
    \end{align*}
    As $K_{\alpha,d} = R^{\alpha - 2} \bar{K}_{\alpha,d}$, with $R > 0$, this implies the result.
\end{proof}

Finally, the next lemma presents the main properties of the function $\alpha \longmapsto P_\alpha$, introduced in \cref{sec:qualitative_analysis}. As a reminder, we have:
\begin{align*}
    P_\alpha := \frac{(2 - \alpha) \Gamma \left( 1 - \frac{\alpha}{2} \right)}{\alpha 2^{\alpha / 2}}
\end{align*}

\begin{lemma}
    \label{lemma:p_decreasing}
    The function $\alpha \longmapsto P_\alpha$ is, on the interval $[1, 2)$, continuous, decreasing and satisfies:
    \begin{align*}
        \forall \alpha \in [1, 2),~\frac{1}{2} \leq P_\alpha \leq \sqrt{\frac{\pi}{2}}.
    \end{align*}
\end{lemma}

\begin{proof}
    The continuity follows from the continuity of the $\Gamma$ function on $(0, \frac{1}{2}]$. Moreover, we have, using the particular values of the $\Gamma$ function recalled in \ref{sec:gamma_function}:
    \begin{align*}
        P_1 = \frac{\Gamma \left( \frac{1}{2} \right)}{\sqrt{2}} = \frac{\sqrt{\pi}}{\sqrt{2}}.
    \end{align*}
    Using the proof of the \cref{lemma:alpha_limit}, we also have:
    \begin{align*}
        P_\alpha \underset{\alpha \to  2^-}{\longrightarrow} \frac{1}{2}.
    \end{align*}
    
    Now we fix $\alpha \in (1,2)$, we use Euler's reflection formula, \cref{lemma:euler_reflection}, and the formula $\Gamma(1 + z) = z\Gamma(z)$ to get:
    \begin{align*}
    P_\alpha &= \frac{(2 - \alpha) \Gamma \left( 1 - \frac{\alpha}{2} \right)}{\alpha 2^{\alpha / 2}}\\
            &= \frac{2 - \alpha}{\alpha 2^{\alpha / 2}} \frac{\pi}{\sin \left( \frac{\pi \alpha}{2} \right) \Gamma \left( \frac{\alpha}{2} \right)} \\
            &= \frac{\pi \left(  1 - \frac{\alpha}{2}  \right)}{\sin \left( \frac{\pi \alpha}{2} \right)} \frac{1}{2^{\alpha/2}\Gamma \left( 1 + \frac{\alpha}{2} \right)}
\end{align*}
From the general properties of the $\Gamma$ function, we know that $\alpha \longmapsto \Gamma \left( 1 + \frac{\alpha}{2} \right)$ is increasing on $(1,2)$, and therefore it is enough to show that the function
\begin{align*}
    g(\alpha) = \frac{2 - \alpha}{\sin \left( \frac{\pi \alpha}{2} \right)},
\end{align*}
is decreasing, as both $\Gamma \left( 1 + \frac{\alpha}{2} \right)$ and $g$ are positive on $(1,2)$. A quick calculation reveals that:
\begin{align*}
    \sin^2 \left( \frac{\pi \alpha}{2} \right) g'(\alpha) = -\cos \left( \frac{\pi \alpha}{2} \right) \left\{ \tan \left( \pi \left(  \frac{\alpha}{2} - 1\right)\right) - \pi \left(1 -  \frac{\alpha}{2} \right) \right\},
\end{align*}
which implies the desired result, by noting that $\tan(x) \leq x$ on $(-\pi/2, 0)$.

\end{proof}

The function $\alpha \longmapsto P_\alpha$ is represented on \cref{fig:p_factor}. This plot confirms the calculations of \cref{lemma:p_decreasing}.

\begin{figure}[!ht]
    \centering
    \includegraphics[width=0.5\linewidth]{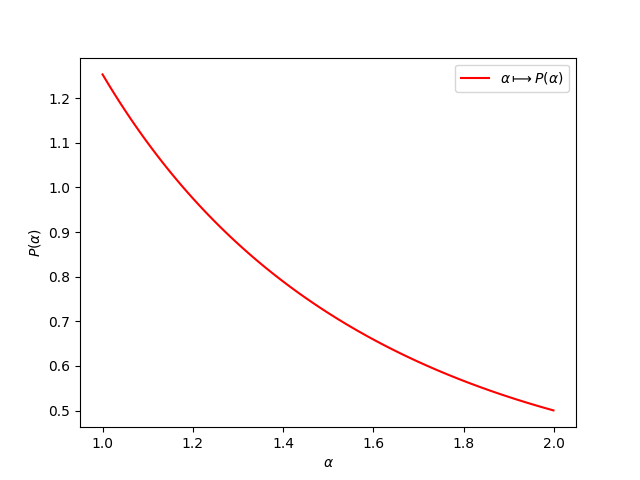}
    \caption{Graphical representation of the function $\alpha \longmapsto P_\alpha$. The leftmost value is $P_1 = \sqrt{\frac{\pi}{2}} \simeq 1.2533$ and the right limit is $P_{2^-} = \frac{1}{2}$.}
    \label{fig:p_factor}
\end{figure}

\subsection{Precisions on the low noise-regime}
\label{sec:low_noise_regime_precision}

In \cref{sec:qualitative_analysis}, we concluded the existence of two regimes predicted by our bounds. We first defined the low noise regime by the following condition:
\begin{align*}
    \frac{\sigma_1 \sqrt{d}}{R} < 1.
\end{align*}
However, the pre-factor $\alpha \longmapsto P_\alpha$ is decreasing. Therefore, we need to take it into account to accurately describe the regime where the generalization error is increasing with $\alpha$. This happens when we have:
\begin{align*}
    \frac{P_1 R}{\sigma_1 \sqrt{d}} < \frac{P_2R^2}{d \sigma_1^2}.
\end{align*}
Using \cref{lemma:p_decreasing}, we get the condition that was mentioned briefly in \cref{sec:qualitative_analysis}:
\begin{align*}
    \frac{\sigma_1 \sqrt{d}}{R} < \frac{1}{\sqrt{2\pi}}.
\end{align*}

\subsection{Comparison with the constants appearing in other works}
\label{sec:raj_comparison_appendix}

\begin{figure}[!ht]
    \centering
    \includegraphics[width=0.5\linewidth]{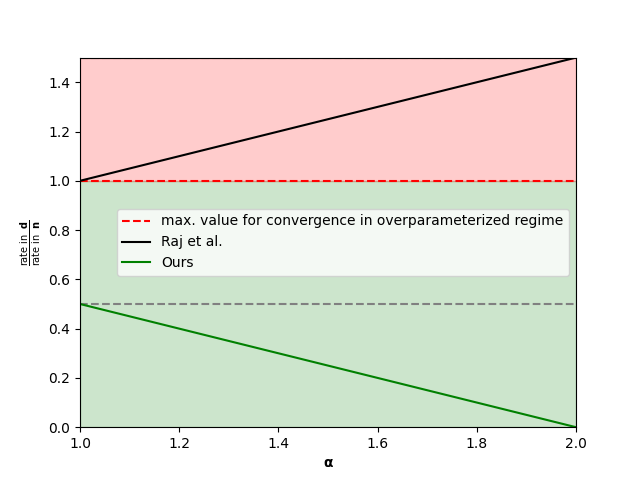}
    \caption{Comparison of the dimension dependence of our bounds and that of \citet{raj_algorithmic_2023}, under the assumption that $\ell$ is Lipschitz continuous, for both bounds.}
    \label{fig:rate_comparison}
\end{figure}

In\citep{raj_algorithmic_2023}, the authors prove expected generalization bounds for heavy-tail dynamics. The constant appearing in their bounds has a complex dependence on various constants defined in their assumptions, but we can sum up the dependence in $\alpha$ and $d$ as follows (by Applying Euler's reflection formula to their Lemma $7$):
\begin{align*}
    R_{\alpha, d} = \landau{1 + \frac{\sqrt{d}\Gamma \left( \frac{\alpha + d}{2} \right)}{(2 - \alpha)\Gamma \left( 1 - \frac{\alpha}{2} \right) \Gamma \left( \frac{d}{2} \right)}}
\end{align*}
By the proof of \cref{lemma:d_limit} and \ref{lemma:alpha_limit}, we see that:
\begin{align*}
    R_{\alpha, d} = \mathcal{O} \left(d^{\frac{1+\alpha}{2}}\right).
\end{align*}
It is to be noted that our bound has a rate in $n$ of $1/2$, \ie our bound is proportional to $1/\sqrt{n}$, while the bound of \citet{raj_algorithmic_2023} has a rate of $1$, \ie the bound is proportional to $1/n$. If we denote by $\Xi$ the ratio of the rate in $d$ to the rate in $n$ of the bound, we find (under a Lipschitz assumption for both bounds):
\begin{align}
    \label{eq:rate_ratio_both_bounds}
    \Xi^{\text{(Ours)}} = 1 - \frac{\alpha}{2}, \quad  \Xi^{\text{\citep{raj_algorithmic_2023}}} =\frac{1+\alpha}{2}.
\end{align}
The meaning of the coefficient $\Xi$ is that, in an overparameterized regime where $d> n$, the bound is non-vacuous only if $\Xi < 1$. We graphically represented this coefficient in \cref{fig:rate_comparison}.

\section{Discussion of the time dependence of the bounds}
\label{sec:time-dependence_discussion}

While \cref{thm:bound_under_jv_assumption} is the first high probability generalization bound for heavy-tailed dynamics, one may notice the time dependence of the bound, coming from the integral term, \ie
\begin{align}
    \label{eq:integral_term_appendix}
    \int_0^T \Eof[U]{\ef(W_t)} dt.
\end{align}
Because of this term, as explained in \cref{sec:qualitative_analysis}, the light-tail limit of this bound leads to an informal bound of the form:
\begin{align*}
    G_S(W_t) \underset{\alpha \to 2^-}{\lesssim} \sqrt{\frac{1}{n\sigma^2}  \int_0^T \Eof[U]{\ef(W_t)} dt} ,
\end{align*}
which is worse than the existing bounds on Langevin dynamics \citet{mou_generalization_2017, li_generalization_2020, farghly_time-independent_2021}, which achieve time-independent bounds. More explicitly, \citet{mou_generalization_2017} obtained a similar bound, but with an additional exponential time-decay in the integral term, \ie, informally,
\begin{align*}
    \int_0^T e^{-a(T-t)}\Eof[U]{\ef(W_t)} dt,
\end{align*}
where $a>0$ is a constant dependent of the problem.

To the best of our knowledge, there is no obvious theoretical argument affirming that the limit $\alpha \to 2^-$ should exactly recover the bounds for Langevin dynamics, while it is quite clear that this limit should not be infinite, as in \citet{raj_algorithmic_2023}. That being said, it is worth discussing why this time dependence happens in our bound, how it could be improved, and, most importantly, why a time independent bound is currently beyond the reach of the theory.

\cref{sec:toward_time_uniform} presents one step in the direction of time-uniform bounds. However, the integral term appearing in \cref{thm:chi_sq_levy} is not satisfactory, because of its lack of interpretability, compared to \cref{eq:integral_term_appendix}. 

On the other hand, a more detailed analysis of the results of \citet{mou_generalization_2017,li_generalization_2020} shows that time-independence is a direct consequence of the use of a logarithmic Sobolev inequality (LSI, see \cref{sec:log_sobolev_inequalities_phi_entropy}). If we translate these arguments in our setting, such an inequality would need to be satisfied by the prior distribution $\pi$. Therefore, it is natural to ask, whether we could apply \cref{thm:generalized_lsi} in our proofs. 

A similar computation has been proposed by \citet{gentil_logarithmic_2008}, in their study of the rate of convergence to equilibrium of an equation of the form of \cref{eq:empirical_lfp}. However, we argue that we cannot apply this reasoning in our setting. Let's explain it briefly: in the proof of \cref{thm:bound_under_jv_assumption}, the crucial term that appears in the computation of the entropy flow is what we called the Bregman integral, it is given by:
\begin{align*}
    B_\Phi^\alpha(v)  = C_{\alpha, d}  \iint \bregman{v(x)}{v(x+z)} \uinftybar d\nu_\alpha(z) dx.
\end{align*}
On the other hand, \cref{thm:generalized_lsi} uses the following term, if we apply it to $\pi$ (in a similar way to how it is done in \cref{sec:proofs-poincare-brownian-case}):
\begin{align*}
    \iint \bregman{v(x+z)}{v(x)} \uinftybar d\nu_\alpha(z) dx.
\end{align*}
Unfortunately, the Bregman divergence is not commutative in general, and, to the best of our knowledge, there is no obvious way of linking the two integrals appearing above. This discussion is the motivation behind \cref{sec:toward_time_uniform}, where we applied \cref{thm:generalized_lsi} and \cref{lemma:big_decomposition} to the case where the Bregman divergence becomes symmetric, \ie when $\Phi(x) \propto x^2$. 

Therefore, we conclude that we do not have the LSI that we would need to make our bounds time-independent. Let's end this short discussion by writing down the result that we would obtain with such an inequality.

\begin{theorem}
    We consider the same setting than in \cref{thm:bound_under_jv_assumption}, but we additionally assume that the prior $\pi$ satisfies the following LSI, with reversed Bregman divergence compared to \cref{thm:generalized_lsi}, for all $v$ smooth enough:
    \begin{align*}
        \entphi{v} \leq \frac{C_{\alpha,d}\sigma_1^\alpha}{\alpha \eta} \iint \bregman{v(x)}{v(x+z)} \uinftybar(x) d\nu_\alpha(z) dx.
    \end{align*}
    Then, with probability at least $1 - \zeta$ over $\datadist$, we have:
    \begin{align*}
      \Eof[U]{G_S(W_T^S)} \leq \sqrt{\frac{2K_{\alpha, d}}{\sigma_1^\alpha} \int_0^T e^{-\frac{\alpha \eta}{2} (T - t)} \Eof[U]{\normof{\nabla\ef(W^S_t)}^2 } dt + \frac{\Lambda + \log(3/\zeta)}{n}}
    \end{align*}
    with $\Lambda = \klb{\rho_0}{\uinftybar}$ and:
    \begin{align}
        K_{\alpha,d} = \frac{(2 - \alpha)\Gamma \left( 1 - \frac{\alpha}{2}\right) d \Gamma \left(\frac{d}{2}\right)}{\alpha 2^\alpha \Gamma \left( \frac{d +\alpha}{2}\right) R^{2 - \alpha} },
    \end{align}  
\end{theorem}

\section{Additional experimental details}
\label{sec:experiments_appendix}

\subsection{Hyperparameters details}
\label{sec:hyperparameters}
In this section, we give a list of the exact hyperparameters used to obtain our main experiments. In general, those hyperparameters were chosen to make our experimental setting as close as possible to our theoretical setting and assumptions.

\textbf{Linear model trained on MNIST} (\cref{fig:linear_model_mnist_all_R,fig:linear_model_sigma_correlation}) : We use a linear predictor, without bias, trained with a (multivariate cross-entropy loss). All $10$ classes of the MNIST dataset were used but we randomly subsample $10\%$ of the training set, to lower the computational cost of the experiments. We simulate \cref{eq:Euler-Maruyama} with $T = 5.10^3$, $\gamma = 10^{-2}$, $\eta = 10^{-3}$. The last $2000$ iterations were used to estimate the accuracy error, as described in \cref{sec:robust_mean_estimation}. During those experiments, we let $\alpha$ vary in $[1.6, 2]$, using a linear scale of $10$ values. The parameter $\sigma$ varies in a logarithmic scale of $10$ values, such that $\sigma \sqrt{d} \in [0.5, 40]$, where $d$ is the number of parameters in the model, this is inspired by the analysis of \cref{sec:qualitative_analysis}. This range was chosen to ensure both that the training is stable enough and that the heavy tail of the noise has an impact on the accuracy (\ie $\sigma$ not too small).

\textbf{$2$-layers FCN trained on MNIST:} This experiment was used to produce \cref{fig:correlation_main,fig:regression_from_d,fig:all_R}. We use a $2$-layer neural network with ReLU activation, without bias, trained with a (multivariate cross-entropy loss). All $10$ classes of the MNIST dataset were used but we randomly subsample $10\%$ of the training set, to lower the computational cost of the experiments. We simulate \cref{eq:Euler-Maruyama} with $T = 10^4$, $\gamma = 10^{-2}$, $\eta = 10^{-3}$. The last $2000$ iterations were used to estimate the accuracy error, as described in \cref{sec:robust_mean_estimation}. During those experiments, we let $\alpha$ vary in $[1.6, 2]$, using a linear scale of $10$ values. The width of the network varies in a linear scale of $10$ values, between $40$ and $200$.

\textbf{FashionMNIST experiments:} The same model was also used on the FashionMNIST dataset (a $10\%$ subsample of it), to obtain \cref{fig:fashion_mnist_all_R,fig:regression_from_d_fashion_mnist,fig:correlation_experiment_fashion_mnist}.

\textbf{$5$-layers FCN trained on MNIST:} We use a $5$-layer neural network with ReLU activation, without bias, trained with a (multivariate cross-entropy loss). $100\%$ of the MNIST dataset were used. We experimented using a batch-size in this case, \ie \cref{eq:Euler-Maruyama} is replaced by:
\begin{align}
    \label{eq:Euler_Maruyama_batch_size}
    \hat{W}^S_{k+1} = \hat{W}^S_k - \frac{\gamma}{b} \sum_{j=1}^b \nabla f (\hat{W}^S_k, z_{i_j}) - \eta \gamma \hat{W}^S_k + \gamma^{\frac{1}{\alpha} } \sigma_1 L_1^\alpha,
\end{align}
where $b\in \mathds{N}^\star$ is the batch size and $(i_1,\dots,i_b)$ are random indices drawn at each iteration.
We simulate \cref{eq:Euler-Maruyama} with $T = 10^5$, $\gamma = 10^{-2}$, $\eta = 10^{-3}$. The last $100$ iterations were used to estimate the accuracy error, as described in \cref{sec:robust_mean_estimation}. We used batch sizes $b \in \set{64,128,256,512}$. During those experiments, we let $\alpha$ vary in $[1.7, 2]$, using a linear scale of $10$ values. In order for the training to be stable enough for the experiment to converge, we noted that we had to use smaller values, those values were chosen on a logarithmic scale of $6$ values so that $\sigma \sqrt{d} \in [0.1, 3]$. It is probable that only the low noise regime, of \cref{sec:qualitative_analysis}, is then observable. Note that, in this case, \ie when using mini-batches, we estimated the gradient norms $\normof{\nabla \ef}$ using the batch gradients, appearing in \cref{eq:Euler_Maruyama_batch_size}, instead of the whole empirical risk $\ef$. This makes sense as one of the goals of using batches is to reduce the computational cost of the experiments.

\subsection{Accuracy evolution and robust mean estimation}
\label{sec:robust_mean_estimation}

\cref{fig:losses_depth_0,fig:losses_depth_1} show the evolution of both the train and test accuracy, during the first $2000$ iterations of training, for models and datasets similar to the one used for our main experiments (but with less data to make it easier to visualize). We show these curves for different values of the tail index $\alpha$ and of the quantity $\sigma \sqrt{d}$, which we argue in \cref{sec:qualitative_analysis} has a strong impact on the training dynamics. 

This allows us to observe two phenomenons, first, at least in these two experiments, there seems to be a transition, around the value $\sigma \sqrt{d} = 1$. Indeed, for $\sigma \sqrt{d} \lesssim 1$, we observe that the tail of the noise has little impact on the evolution of both accuracies and therefore on the generalization error. On the other hand, for higher values of $\sigma \sqrt{d}$, we see the impact of the tail index $\alpha$. Indeed, the heavier the tail (\ie the smaller $\alpha$), the more jumps are observed on both accuracies, which necessarily have an impact on the observed generalization. This seems to follow our theory developed in \cref{sec:qualitative_analysis}. 

More importantly, the presence of those jumps makes the estimation of the generalization error, $G_S(W_T^S)$, extremely noisy. To avoid this issue, we use the following two techniques:
\begin{enumerate}
    \item The generalization error is averaged among the $2000$ last iterations (recall that we use $N \geq 5000$ in all our experiments so that the model is close to converence when we estimate the generalization error).
    \item As the random jumps introduce values of the generalization error that are completely random and often much higher than the values between the jumps, we avoid them to bias the estimation by removing the $15\%$ upper quantile of the generalization error, in the last $2000$ iterations, before computing the mean.
\end{enumerate}

\begin{figure}[!ht]
\centering
\subfigure[$\sigma \sqrt{d} = 0.1$\label{fig:losses_depth_0_1}]{\includegraphics[trim={2cm 0 3cm 0},clip,width=0.9\linewidth]{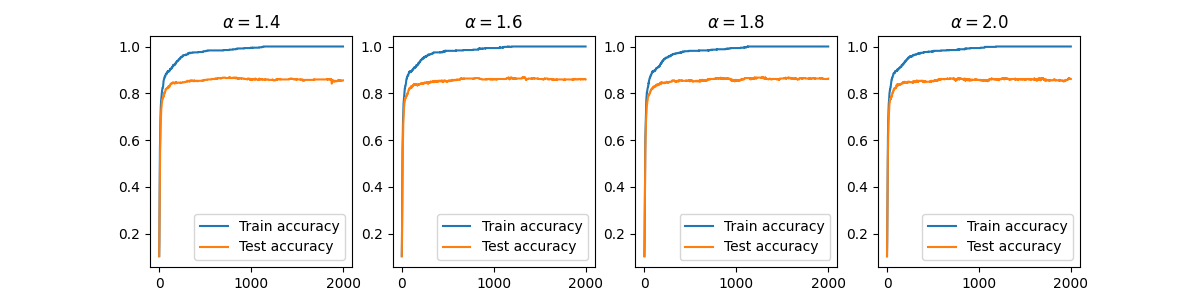}}
\vfill
\subfigure[$\sigma \sqrt{d} = 1$\label{fig:losses_depth_0_2}]{\includegraphics[trim={2cm 0 3cm 0},clip,width=0.9\linewidth]{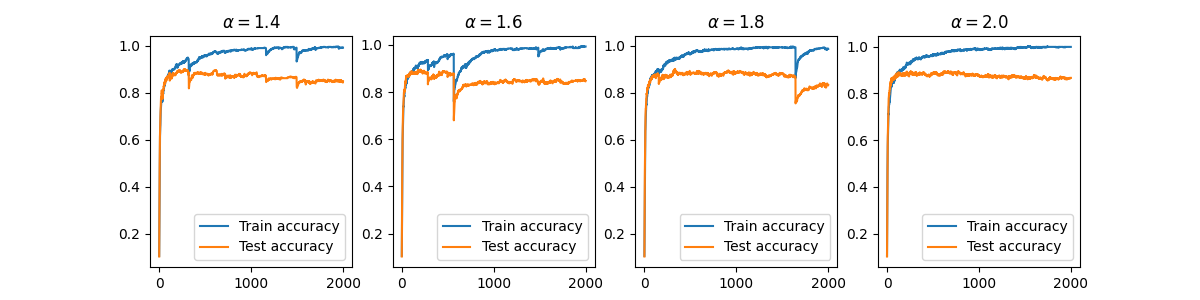}}
\vfill
\subfigure[$\sigma \sqrt{d} = 10$\label{fig:losses_depth_0_3}]{\includegraphics[trim={2cm 0 3cm 0},clip,width=0.9\linewidth]{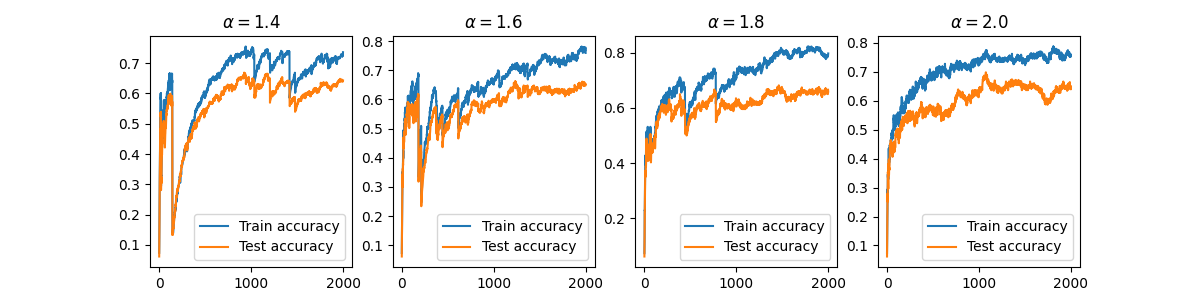}}
\caption{Evolution of the test and training accuracies, during the first $2000$ iterations, for the simulation of \eqref{eq:multifractal_dynamics} with a linear model trained with a cross-entropy loss, on the MNIST dataset.}
\label{fig:losses_depth_0}
\end{figure}

\begin{figure}[!ht]
\centering
\subfigure[$\sigma \sqrt{d} = 1$\label{fig:losses_depth_1_1}]{\includegraphics[trim={2cm 0 3cm 0},clip,width=0.9\linewidth]{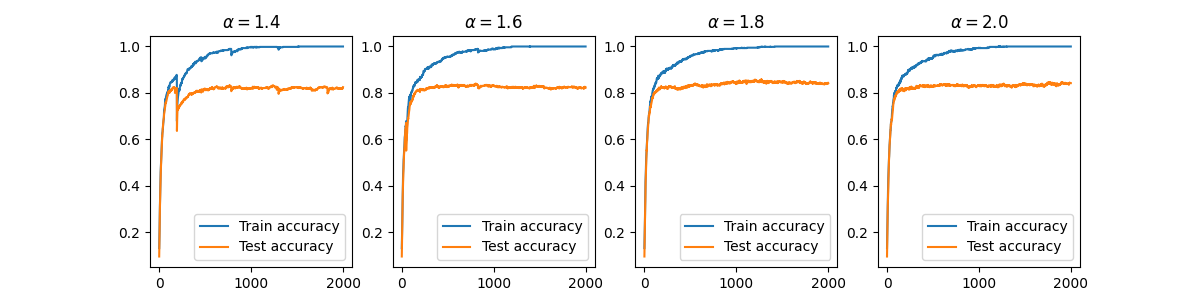}}
\vfill
\subfigure[$\sigma \sqrt{d} = 10$\label{fig:2}]{\includegraphics[trim={2cm 0 3cm 0},clip,width=0.9\linewidth]{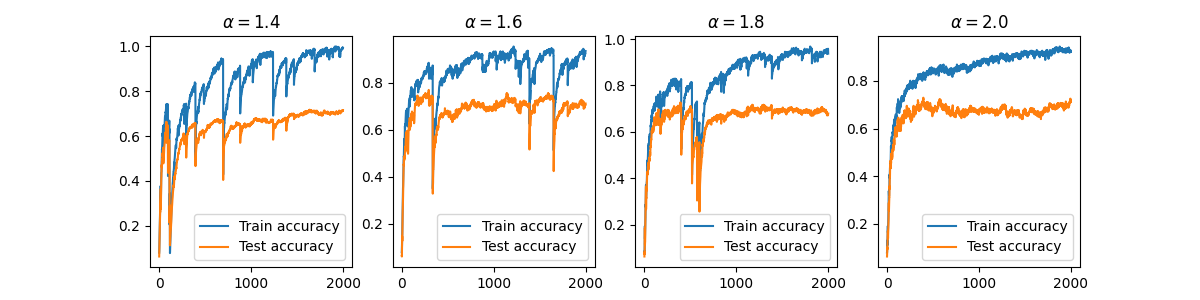}}
\caption{Evolution of the test and training accuracies, during the first $2000$ iterations, for the simulation of \eqref{eq:multifractal_dynamics} with a $2$-layer neural network, trained with a cross-entropy loss, on the MNIST dataset.}
\label{fig:losses_depth_1}
\end{figure}

\subsection{Experimental procedures details}
\label{sec:procedure_details}

In this section, we quickly give more details on how the figures of \cref{sec:experiments} were obtained.

\subsubsection{Bound estimation - \cref{fig:all_R,fig:linear_model_mnist_all_R,fig:fashion_mnist_all_R}}

To get \cref{fig:all_R,fig:linear_model_mnist_all_R,fig:fashion_mnist_all_R}, we estimated the computable part of \cref{eq:bound_in_bounded_case}. To do so, the gradient norm of all iterations where averaged, to get an estimate of the integral term:
\begin{align}
    \label{eq:integral_estimation}
    \widehat{I} := \gamma \sum_{k=1}^N \normof{\nabla \ef(\hat{W}_k^S)}^2 \simeq \int_0^T \normof{\nabla \ef(W_t^S)}^2 dt,
\end{align}
where $\gamma$ is the learning rate and $N$ the number of iterations.
The quantity that is plotted on the $y$-axis of those figures is:
\begin{align}
    \label{eq:estimation_formula}
    \widehat{G} :=  \sqrt{ \frac{P_\alpha d^{1 - \frac{\alpha}{2}}}{n\sigma_1^\alpha R^{2 - \alpha}} \widehat{I}}.
\end{align}
Moreover, this quantity is averaged over $10$ random seeds, with the same values of the hyperparameters $(\sigma, d, \alpha)$. Unless mentioned otherwise, the constant $R$, which is unknown a priori, is taken to be $R=1$ in the figures.

\subsubsection{Correlation experiments - \cref{fig:correlation_main,fig:correlation_experiment_fashion_mnist}}

\textbf{\cref{fig:correlation_main,fig:correlation_experiment_fashion_mnist}:} In this experiment, we train a $2$-layers neural network on MNIST. The tail-index $\alpha$ varies in $[1.6, 2]$ and the width of the network in $[40, 200]$. To obtain the green curve, we proceed as follows:
\begin{itemize}
    \item For each value fixed $w$ of the width, we compute Kendall's correlation coefficient, denoted $\tau$, between the accuracy error and the tail-index $\alpha$.
    \item This gives us the correlation in terms of the width, \ie a map $\tau(w)$.
    \item We repeat this procedure for $10$ random seeds.
    \item The green curve represents the mean value and standard deviation of $\tau(w)$.
\end{itemize}
For the dot black curve, to get a less noisy coefficient, we directly used, for each width $w$, the average of the accuracy error across all the random seeds, and we computed the correlation between $\alpha$ and this mean accuracy error. This is why there are no error bars in that case.

\subsection{Additional experiments}
\label{sec:additional_experiments}

In this subsection, we present a few additional experiments to further support the findings of \cref{sec:experiments}. These additional experiments are organized as follows:
\begin{enumerate}
    \item Additional experiments using a linear model.
     \item Additional experiments conducted on the FashionMNIST dataset (the experiments presented in \cref{sec:experiments} were all conducted on the MNIST dataset).
    \item More details figures regarding the use of mini-batches.
\end{enumerate}

\subsubsection{Additional experiment with linear models}
\label{sec:linear_model_additional_experiments}

We conducted experiments similar to that presented in \cref{sec:experiments} with a linear model, instead of a $2$-layers FCN. The results are presented in \cref{fig:linear_model_mnist_all_R,fig:linear_model_sigma_correlation}. More specifically, in these experiments, we let both the noise scale $\sigma_1$ and the tail-index $\alpha$ vary in a fixed grid of values (see \cref{sec:hyperparameters}). Note that, in this case, it is not possible to let the number of parameters vary without affecting the data, but we can act on $\sigma_1$ instead, and still test our theory. That being said, this model has the advantage, when using the cross-entropy as a surrogate loss, to fit nicely in our theoretical setting.

\begin{figure}[!ht]
    \centering
    \subfigure[$R=1$]{\includegraphics[width=0.4\linewidth]{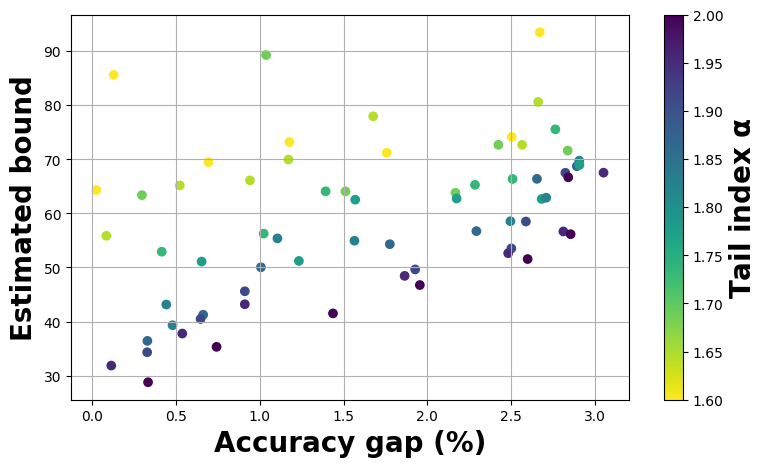}}
    \hfill
    \subfigure[$R=25$]{\includegraphics[width=0.4\linewidth]{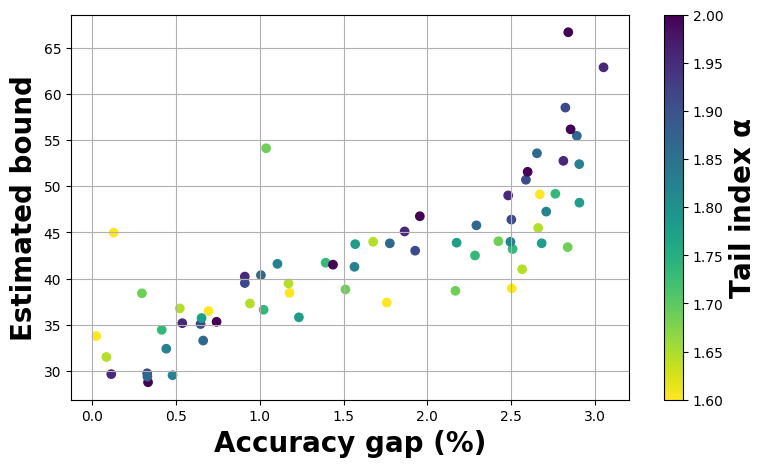}}
    \caption{Linear model on MNIST}
    \label{fig:linear_model_mnist_all_R}
\end{figure}

\begin{figure}
    \centering
    \includegraphics[width=0.5\linewidth]{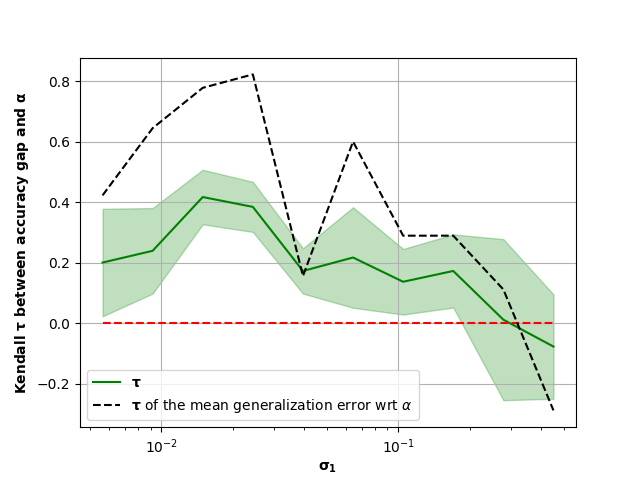}
    \caption{Correlation (Kendall's $\tau$) between the tail index $\alpha$ and the accuracy gap for different values of the noise scale $\sigma_1$.}
    \label{fig:linear_model_sigma_correlation}
\end{figure}

More specifically, we can see on \cref{fig:linear_model_mnist_all_R} that the estimated bound, according to \cref{eq:estimation_formula}, correlates very well with the generalization error. On \cref{fig:linear_model_sigma_correlation}, we show how the correlation between $\alpha$ and the accuracy gap varies when $\sigma_1$ is varying, see \cref{sec:procedure_details}. The phase transition, predicted in \cref{sec:qualitative_analysis} is visible in this figure. An additional interesting phenomenon is observed on \cref{fig:linear_model_sigma_correlation}. Indeed, on the left of the correlation plots (both the black and green curves), we can see that the Kendall's $\tau$ coefficient, measuring the correlation between $\alpha$ and $G_S$, decreases. We interpret this as a third regime, when $\sigma_1$ is too small for our theory to hold and the algorithm starts behaving more like a deterministic gradient flow.

As we did on previous experiments, we can use \cref{fig:linear_model_sigma_correlation} and the analysis of \cref{sec:qualitative_analysis} to estimate the value that the parameter $R$ takes in those experiments. Given that $d = 7840$ in this experiment and that the phase transition seems to be happening around $\sigma_1 \simeq 0.3$, we can estimate $R\simeq \sigma_1 \sqrt{d} \simeq 26$. On \cref{fig:linear_model_mnist_all_R}, we plotted the estimated bound both for $R=1$ (our default choice) along with the estimated bound for $R=25$. We observe on that second figure a much improved correlation, hence supporting the experimental results of \cref{sec:experiments}.

\begin{remark}
    With this linear model experiment, it is not possible to let the number of parameters vary (at least not without affecting the data), it is therefore not possible to test the dimension dependence, as in \cref{fig:regression_from_d}, in this case.
\end{remark}

\subsubsection{Additional experiments on the FashionMNIST dataset}
\label{sec:fashion_mnist_experiments_appendix}

\begin{figure}[!ht]
    \centering
    \subfigure[\label{fig:regression_from_d_fashion_mnist}Regression of the parameter $\alpha$ from the accuracy error, for a $2$-layers network trained on $10\%$ of the FashionMNIST dataset.]{
    \includegraphics[width=0.45\linewidth]{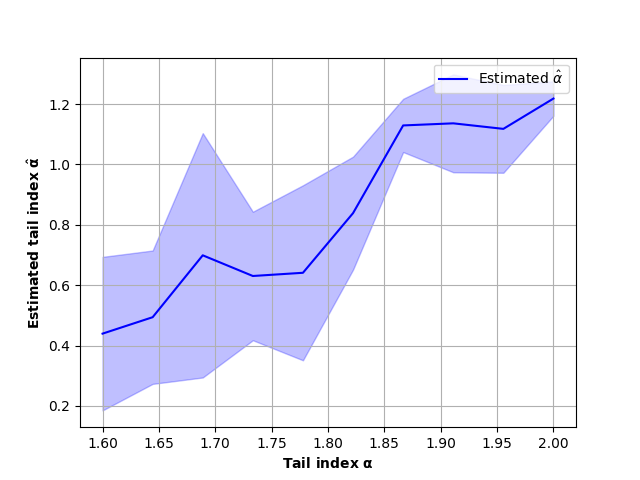}}
    \hfill
    \subfigure[\label{fig:correlation_experiment_fashion_mnist}Correlation between $\alpha$ and the accuracy error, for a $2$-layers network trained on $10\%$ of the FashionMNIST dataset, with varying width.]{\includegraphics[width=0.45\linewidth]{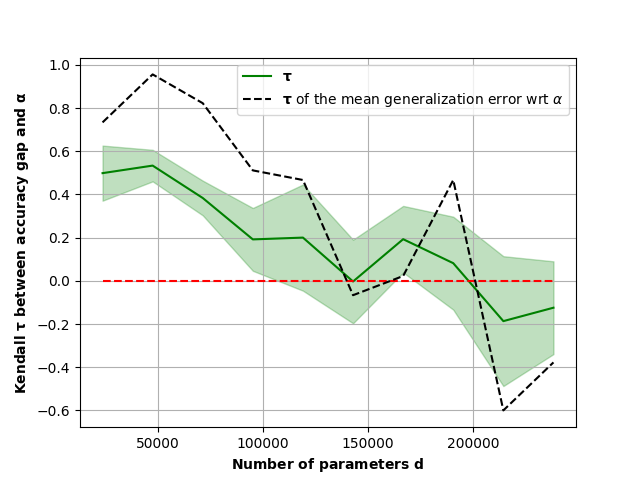}}
    \caption{Experiments with a $2$-layers FCN and the FashionMNIST dataset.}
\end{figure}

We conducted the same experiments as those presented on \cref{fig:regression_from_d,fig:correlation_main} on the Fashion MNIST dataset. The results are presented in \cref{fig:regression_from_d_fashion_mnist,fig:correlation_experiment_fashion_mnist}. We observe the exact same behavior than the experiments conducted on the MNIST dataset: the regression of the tail index shows a remarkable correlation with the ground truth $\alpha$, with the expected monotonicity, even though the value of $\alpha$ is underestimated. Moreover, the phase transition, between the two regimes predicted in \cref{sec:qualitative_analysis}, is observable in \cref{fig:correlation_experiment_fashion_mnist}. It seems to happen at a higher value of $d$ than in the MNIST experiment.

\begin{figure}[!ht]
    \centering
    \subfigure[$R=4$]{\includegraphics[width=0.4\linewidth]{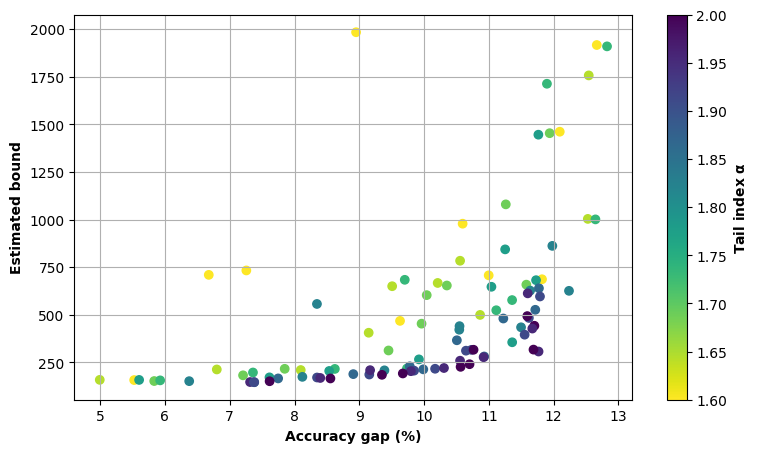}}
    \hfill
    \subfigure[$R=11$]{\includegraphics[width=0.4\linewidth]{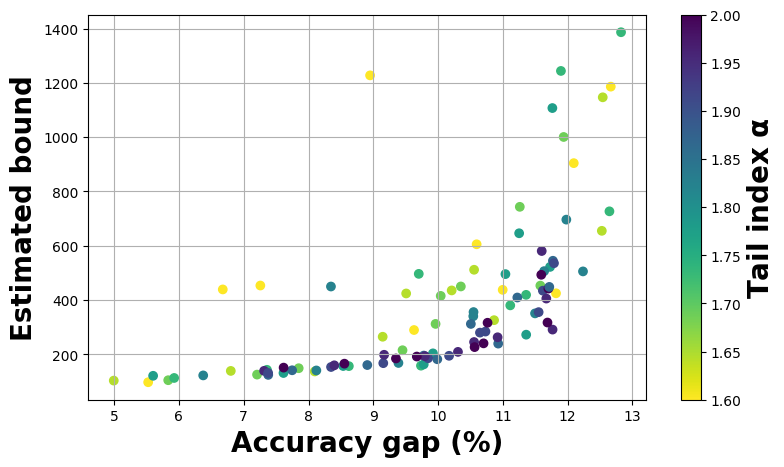}}
    \caption{FasionMNIST}
    \label{fig:fashion_mnist_all_R}
\end{figure}

Additionally, we present in \cref{fig:fashion_mnist_all_R} the value of the estimated bound, based on \cref{eq:estimation_formula}, compared to the accuracy gap, for different values of th parameter $R$. For the default choice $R=1$, we already observe a very good correlation. Based on \cref{fig:correlation_experiment_fashion_mnist}, using the same procedure than in our previous experiments, we can estimate the value of $R$ to be about $R\simeq 11$. As we observe in all our previous experiments, in \cref{sec:experiments,sec:linear_model_additional_experiments}, we observe that taking into account this value of $R$ improves the observed correlation betzeen the bound and the accuracy gap, hence supporting the theory.

\subsubsection{Additional experiments using the CIFAR$10$ dataset}
\label{sec:cnn-experiments}

All our main experiments were conducted on the MNIST and FashionMNIST datasets ($10\%$ sub-sample), and with small models (\ie, linear models and $2$-layers neural networks). The fact that we focus on these small scale experiments is mainly due to three factors:
\begin{enumerate}
    \item As mention in \cref{sec:experiments}, the computational cost of our experiments is high. They are much more scalable for small models and datasets.
    \item It is not clear whether larger DNNs satisfy our theoretical assumptions, further theoretical investigation would be needed in that direction.
    \item The Lévy stable noise $\levy$ can make the training dynamics highly unstable, as is shown on \cref{fig:losses_depth_1}. We believe that this is mainly due to the fact that the heavy-tailed noise in \cref{eq:Euler-Maruyama} does not take into account the model's structure, as in \citep{wan2023implicit} for instance. Further investigation would also be needed to better understand this behavior. 
\end{enumerate}

Despite these difficulties, we present small preliminary results toward the application of our theory to larger DNNs. More precisely we considered a small convolutional neural network (CNN) with $2$ convolutional layers, $3$ fully-connected (FC) layers and ReLU activation. We varied the width of these FC layers between $10$ and $200$ in order to vary the number of parameters $d$. As a dataset, we use a small sample of the CIFAR$10$ dataset, using only $2$ classes of a $10\%$ sub-sample of the dataset.

The results presented in \cref{fig:cnn-regression,fig:cnn-bound} are coherent with those obtained in \cref{sec:experiments}, hence hinting toward the idea that our theory could extend to more practical settings.

\begin{figure}
    \centering
    \includegraphics[width=0.5\linewidth]{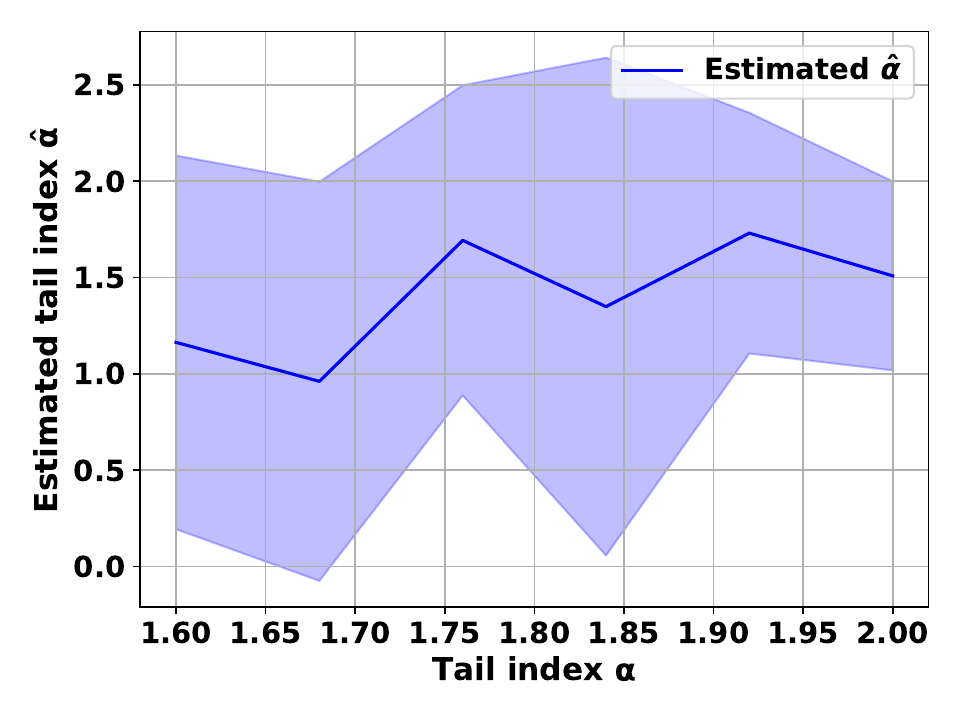}
    \caption{Regression of $\alpha$ from the generalization bound, for a CNN on CIFAR$10$.}
    \label{fig:cnn-regression}
\end{figure}

\begin{figure}[!ht]
    \centering
    \subfigure[$R=1$]{\includegraphics[width=0.49\linewidth]{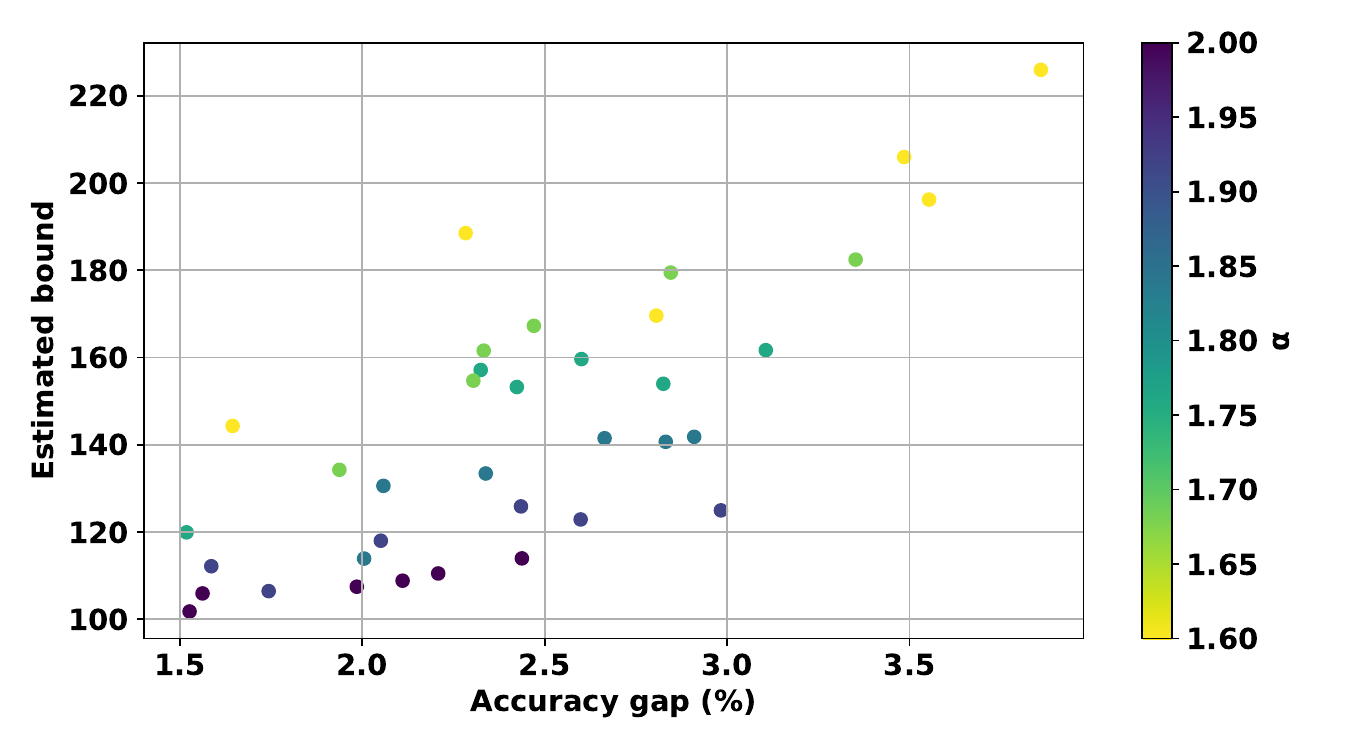}}
    \hfill
    \subfigure[$R=7$]{\includegraphics[width=0.49\linewidth]{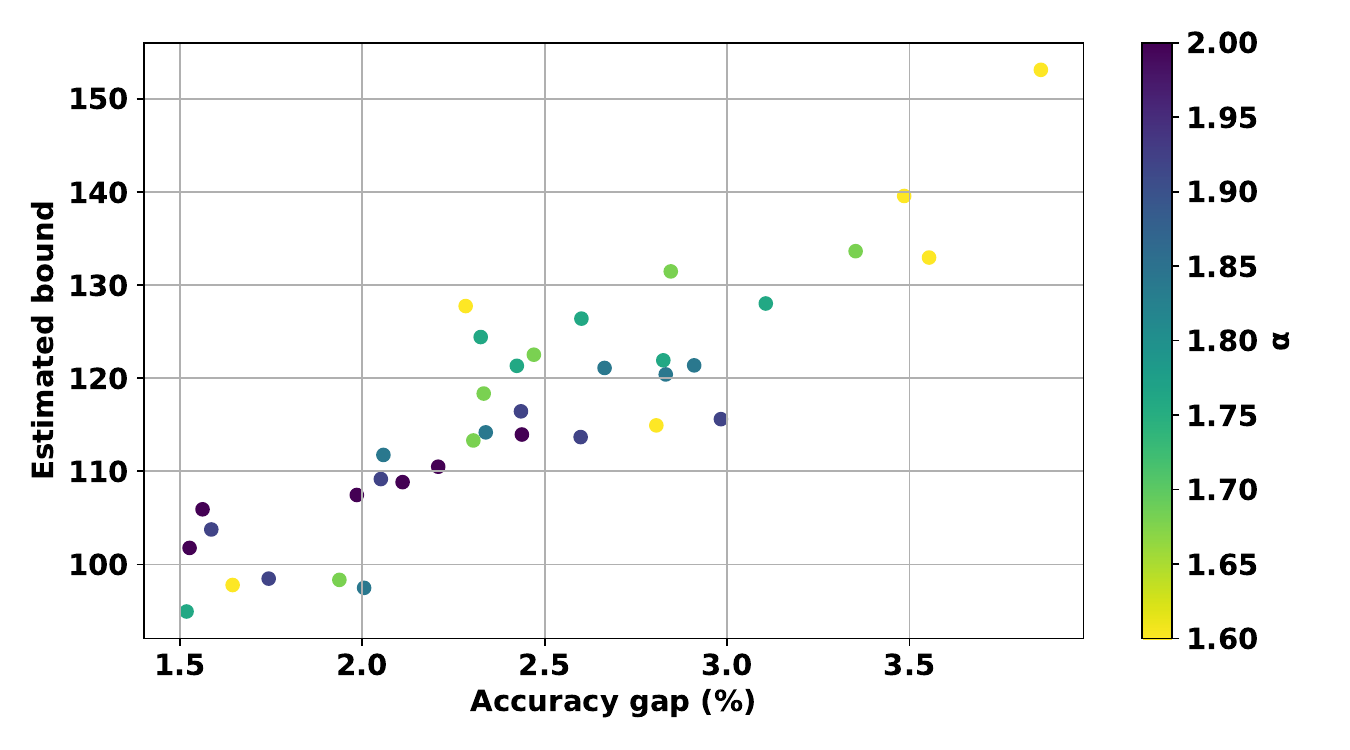}}
    \caption{Estimated bound, computed with \cref{eq:estimation_formula}, versus accuracy gap, for a CNN on CIFAR$10$, for $R=1$ (\textit{left}) and $R=7$ (\textit{right}).}
    \label{fig:cnn-bound}
\end{figure}

\subsubsection{Experiments using Pearson correlation coefficient}

In \cref{fig:correlation_main}, we presented the value of the Kendall's correlation coefficient between the generalization error and the tail index $\alpha$, for different values of $d$. We present the same experiment in \cref{fig:pearson_plot}, but using Pearson's correlation coefficient instead of Kendall's correlation coefficient. We observe that we can still observe the desired phase transition, hence showing that our experiment are robust to the choice of correlation coefficient. However, Kendall's coefficient may be more adapted to our setup, as it is not clear that a linear correlation can be observed in practice.

\begin{figure}[!ht]
    \centering
    \includegraphics[width=0.5\linewidth]{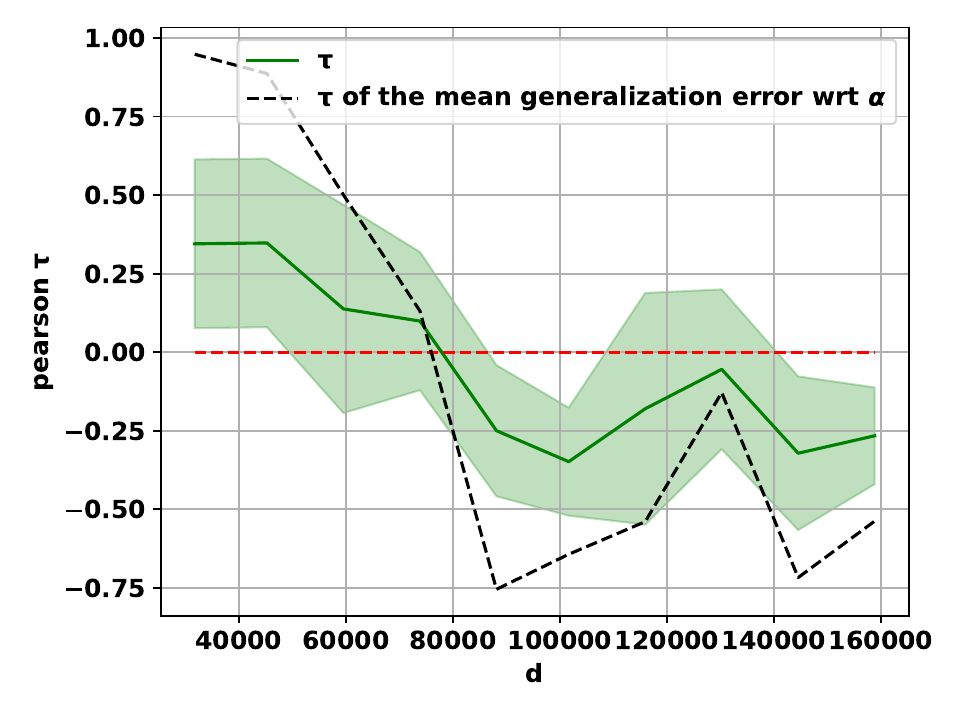}
    \caption{Same experiment as \cref{fig:correlation_main}, but using the Pearson correlation coefficient instead of the Kendall's correlation coefficient.}
    \label{fig:pearson_plot}
\end{figure}

\subsubsection{Additional experiments with the full MNIST dataset and mini-batches}
\label{sec:additiional_full_batch_appendix}

As already mentioned in \cref{sec:experiments}, in order to argue that our bounds may also be pertinent in more practical settings than the figure presented in the rest of the paper, we computed our bound using a FCN$5$, on the full MNIST dataset, using mini-batches. This is to be compared eith our other experiments, were, according to the SDE we study, we take the full batch at each iteration to compute $\nabla \ef$. The exact procedure and hyperparameters are detailed, as for other experiments, in \cref{sec:hyperparameters}. These results are shown on \cref{fig:full_batch_each_batch_size}, were we observe that, for several values of the batch size, we still observe a very good correlation between the estimated bound and the generalization error. Unfortunately, our theroy does not predict the behavior of the accuracy gap with respect to the batch-size, but, from this experiment, we understand that it is still capturing relatively well the behavior of the accuracy gap.

\begin{figure}[!ht]
    \label{fig:full_batch_each_batch_size}
    \centering
    \subfigure[Batch size $64$]{\includegraphics[width=0.3\linewidth]{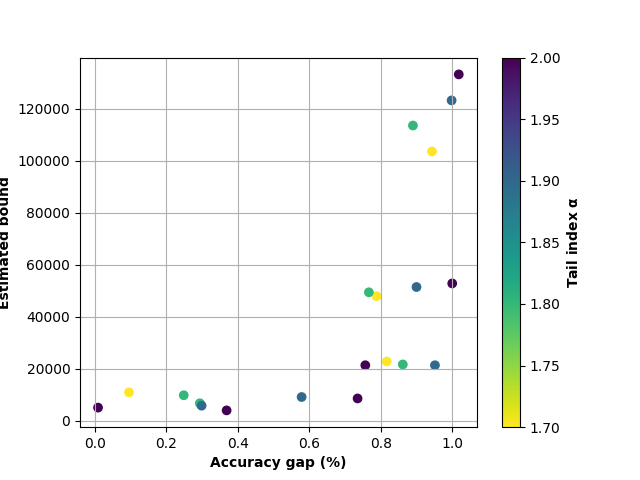}}
    \hfill
    \subfigure[Batch size $128$]{\includegraphics[width=0.3\linewidth]{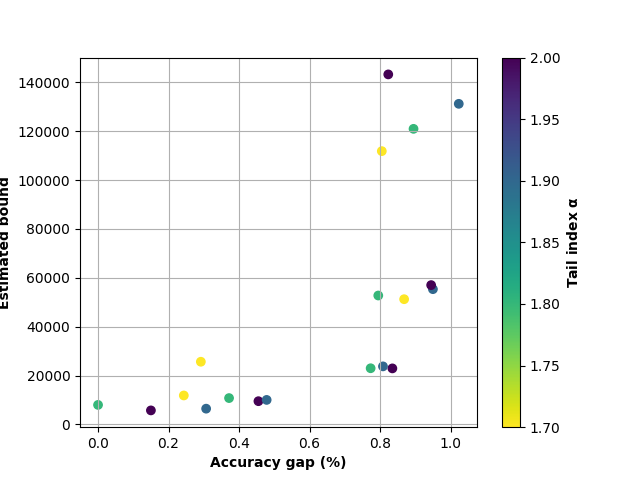}}
    \hfill
    \subfigure[Batch size $256$]{\includegraphics[width=0.3\linewidth]{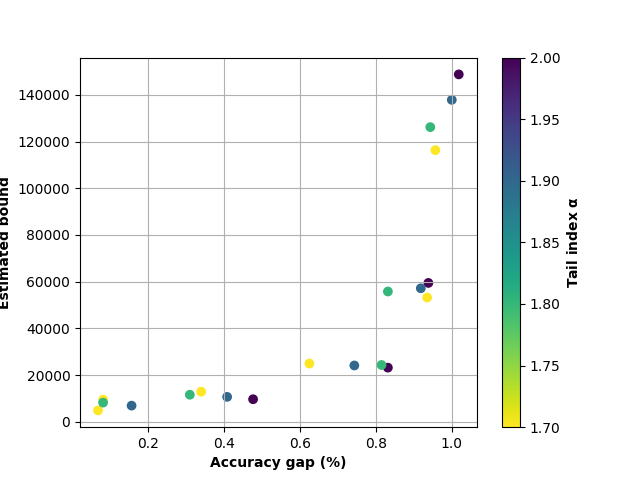}}
    \caption{Estimated bound, computed with \cref{eq:estimation_formula}, versus accuracy gap, for a $5$-layers FCN on the full MNIST dataset, for different value of batch size.}
\end{figure}

\end{document}